%% file: thesis.tex
\documentclass[11pt,a4paper,titlepage]{memoir}

\usepackage[OT1]{fontenc}

\usepackage[english]{babel}

\usepackage[utf8]{inputenc}

\usepackage[sc]{mathpazo}

\usepackage{amsmath,amssymb,amsfonts,mathrsfs}

\usepackage[amsmath,thmmarks]{ntheorem}

\usepackage{graphicx}

\usepackage{soul}

\usepackage{pdfpages}

\input{extrapackages}

\input{layoutsetup}

\input{theoremsetup}

\input{macrosetup}

\usepackage[linkcolor=black,colorlinks=true,citecolor=black,filecolor=black]{hyperref}

\title{Contextual Bandit Optimization with Pre-Trained Neural Networks}
\author{Mikhail Terekhov}
\thesistype{Master Thesis}
\advisors{Advisors: Parnian Kassraie, Prof.~Dr.~Andreas Krause}
\department{Department of Computer Science}
\date{}

\begin{document}

\frontmatter

\begin{titlingpage}
  \calccentering{\unitlength}
  \begin{adjustwidth*}{\unitlength-24pt}{-\unitlength-24pt}
    \maketitle
  \end{adjustwidth*}
\end{titlingpage}

\input{abstract}

\newpage

\input{acknowledgements}

\cleartorecto
\tableofcontents
\mainmatter

\input{chapters/intro}

\input{chapters/algorithm}
\input{chapters/linear}
\input{chapters/sgd}

\input{chapters/combined}
\input{chapters/experiments}

\input{chapters/conclusion}

\appendix

\input{appendix}

\backmatter

\bibliographystyle{plain}
\bibliography{thesis}

\includepdf[pages={-}]{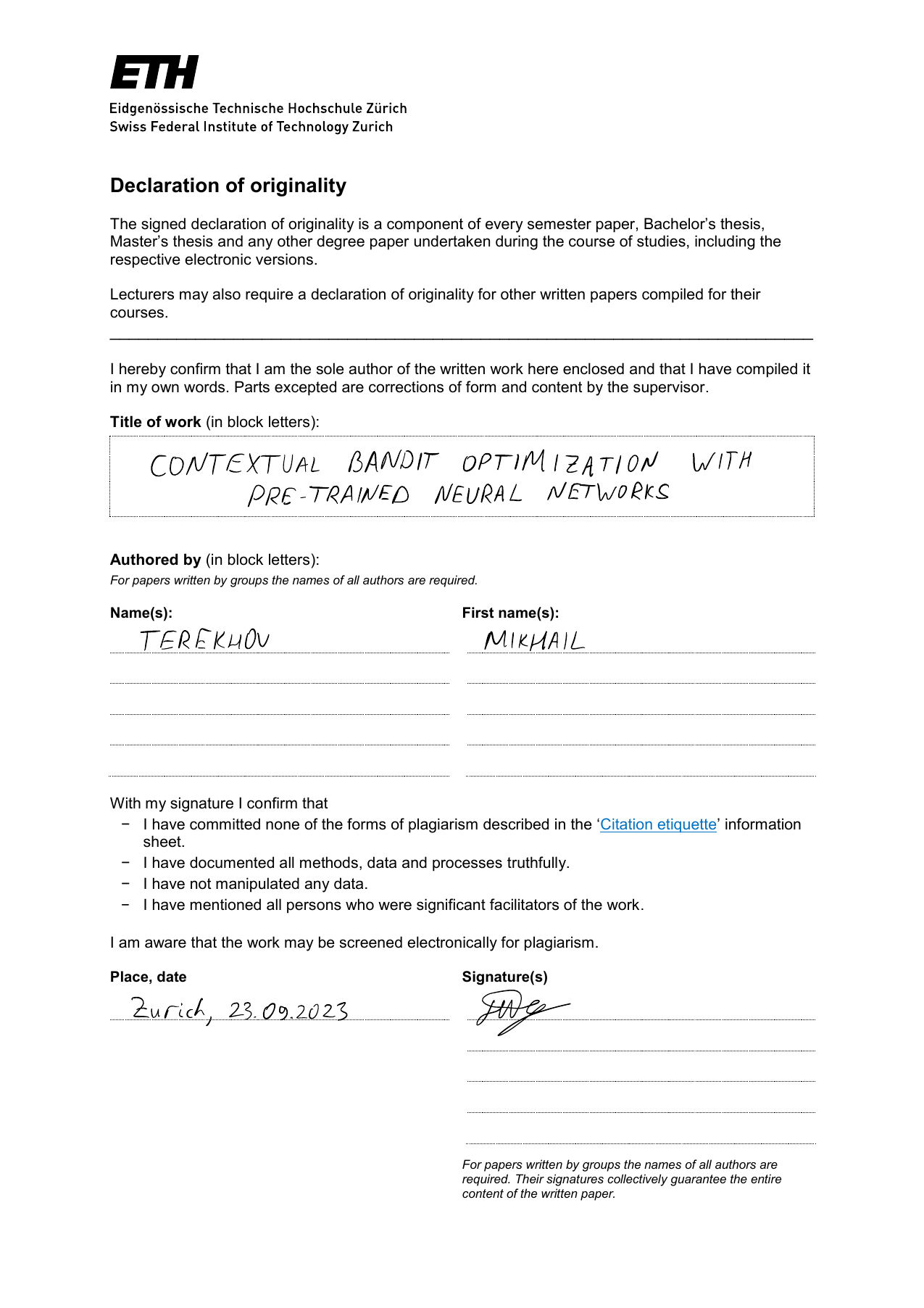}

\end{document}

%% file: extrapackages.tex
\usepackage{varioref}

\usepackage{datetime}

\usepackage{mathtools}

\usepackage[h]{esvect}

\usepackage{array}

\usepackage{listings}
\usepackage{microtype}

\usepackage{booktabs}

\usepackage{algorithm2e}


%% file: layoutsetup.tex
\usepackage{ETHlogo}

\nonzeroparskip
\defaultlists

\makeatletter

\if@twoside
  \copypagestyle{chapter}{Ruled}
\else
  \copypagestyle{chapter}{ruled}
\fi
\makeoddhead{chapter}{}{}{}
\makeevenhead{chapter}{}{}{}
\makeheadrule{chapter}{\textwidth}{0pt}
\copypagestyle{abstract}{empty}

\makechapterstyle{bianchimod}{%
  \chapterstyle{default}
  \renewcommand*{\chapnamefont}{\normalfont\Large\sffamily}
  
  \renewcommand*{\printchaptername}{%
    \chapnamefont\centering\@chapapp}

  }

\chapterstyle{bianchimod}

\setsecheadstyle{\Large\bfseries\sffamily}
\setsubsecheadstyle{\large\bfseries\sffamily}
\setsubsubsecheadstyle{\bfseries\sffamily}
\setparaheadstyle{\normalsize\bfseries\sffamily}
\setsubparaheadstyle{\normalsize\itshape\sffamily}
\setsubparaindent{0pt}

\captionnamefont{\sffamily\bfseries\footnotesize}
\captiontitlefont{\sffamily\footnotesize}
\setsecnumdepth{subsection}
\settocdepth{subsection}

\pretitle{\vspace{0pt plus 0.7fill}\begin{center}\HUGE\sffamily\bfseries}
\posttitle{\end{center}\par}
\preauthor{\par\begin{center}\let\and\\\Large\sffamily}
\postauthor{\end{center}}
\predate{\par\begin{center}\Large\sffamily}
\postdate{\end{center}}

\def\@advisors{}
\newcommand{\advisors}[1]{\def\@advisors{#1}}
\def\@department{}
\newcommand{\department}[1]{\def\@department{#1}}
\def\@thesistype{}
\newcommand{\thesistype}[1]{\def\@thesistype{#1}}

\renewcommand{\maketitlehookb}{\vspace{1in}%
  \par\begin{center}\Large\sffamily\@thesistype\end{center}}

\renewcommand{\maketitlehookd}{%
  \vfill\par
  \begin{flushright}
    \sffamily
    \@advisors\par
    \@department, ETH Z\"urich
  \end{flushright}
}

\checkandfixthelayout

\makeatother

\theoremstyle{plain}


%% file: theoremsetup.tex
\numberwithin{equation}{chapter}

\newtheorem{theorem}{Theorem}[chapter]

\newtheorem{corollary}[theorem]{Corollary}

\newtheorem{lemma}[theorem]{Lemma}

\newtheorem{assumption}{Assumption}



\theoremstyle{nonumberplain}
\theorembodyfont{\normalfont}
\theoremsymbol{\ensuremath{\square}}
\newtheorem{proof}{Proof}

%% file: macrosetup.tex
\newcommand{\N}{\mathbb{N}}
\newcommand{\R}{\mathbb{R}}

\newcommand{\PP}{\mathbb{P}}
\newcommand{\E}{\mathbb{E}}

\DeclarePairedDelimiter\norm{\lVert}{\rVert}

\renewcommand{\epsilon}{\ensuremath\varepsilon}

\renewcommand{\phi}{\ensuremath{\varphi}}

\newcommand{\dom}{\mathcal{B}}
\newcommand{\conv}{\mathcal{C}}

\DeclareMathOperator*{\argmin}{arg\,min}
\DeclareMathOperator*{\argmax}{arg\,max}

\newcommand{\pzero}{\phi_{\theta_0}}
\newcommand{\pstar}{\phi_{\theta^*}}
\newcommand{\wzero}{\mathbf{w}_0}
\newcommand{\w}{\mathbf{w}}
\newcommand{\gv}{\mathbf{v}}
\newcommand{\wstar}{\mathbf{w}^*}
\newcommand{\wstartp}{\mathbf{w}^{*\tp}}
\newcommand{\risk}{\mathfrak{R}}
\newcommand{\erisk}{\widehat{\mathfrak{R}}}
\newcommand{\crisk}{\widetilde{\mathfrak{R}}}

\newcommand{\deffo}{d_{1,\lambda}} 
\newcommand{\defft}{d_{2,\lambda}} 
\newcommand{\deffp}{d_{p,\lambda}} 
\newcommand{\ceff}{\tilde{d}_{1,\lambda}} 
\newcommand{\dht}{\hat{d}_{\lambda}} 

\newcommand{\deps}{\Delta\epsilon_\w} 
\newcommand{\sigz}{\Sigma_0}
\newcommand{\hsigz}{\widehat{\Sigma}_0}
\newcommand{\sigl}{\Sigma_\lambda}
\newcommand{\hsigl}{\widehat{\Sigma}_\lambda}
\newcommand{\sigd}{\Sigma_\delta}
\newcommand{\hwl}{{\widehat{\w}_\lambda}}

\newcommand{\tp}{\intercal}
\DeclareMathOperator{\tr}{tr}

%% file: abstract.tex
\begin{abstract}
Bandit optimization is a difficult problem, especially if the reward model is high-dimensional. When rewards are modeled by neural networks, sublinear regret has only been shown under strong assumptions, usually when the network is extremely wide. In this thesis, we investigate how pre-training can help us in the regime of smaller models. We consider a stochastic contextual bandit with the rewards modeled by a multi-layer neural network. The last layer is a linear predictor, and the layers before it are a black box neural architecture, which we call a representation network. We model pre-training as an initial guess of the weights of the representation network provided to the learner. To leverage the pre-trained weights, we introduce a novel algorithm we call Explore Twice then Commit (E2TC). During its two stages of exploration, the algorithm first estimates the last layer's weights using Ridge regression, and then runs Stochastic Gradient Decent jointly on all the weights. For a locally convex loss function, we provide conditions on the pre-trained weights under which the algorithm can learn efficiently. Under these conditions, we show sublinear regret of E2TC when the dimension of the last  layer and number of actions $K$ are much smaller than the horizon $T$. In the weak training regime, when only the last layer is learned, the problem reduces to a misspecified linear bandit. We introduce a measure of misspecification $\epsilon_0$ for this bandit and use it to provide bounds $O(\epsilon_0\sqrt{d}KT+(KT)^{4 /5})$ or $\widetilde{O}(\epsilon_0\sqrt{d}KT+d^{1 /3}(KT)^{2 /3})$ on the regret, depending on regularization strength. The first of these bounds has a dimension-independent sublinear term, made possible by the stochasticity of contexts. We also run experiments to evaluate the regret of E2TC and sample complexity of its exploration in practice. 
\end{abstract}

%% file: acknowledgements.tex
\section*{Acknowledgements}
First and foremost, I would like to thank Parnian for her continuous support and patience. I was new to the theory world, and without her guidance I would be completely lost there. I would also like to thank Prof. Krause for the opportunity to do the thesis at his lab. My time at ETH was life-changing, and I would not be here without the ESOP Scholarship. For that, I am grateful to the ETH Foundation and to those who supported my application. 

Thank you, Weijia, for being there, even when I'm writing throughout the night. I am happy to have absolutely amazing friends. Thanks, Dima, Antoshka, Andrew, Mike, and Justin! My parents were the ones who got me interested in science, so much that I barely manage to think about anything else since then. Thank you!

%% file: chapters/intro.tex
\chapter{Introduction}

In machine learning, empirical progress often outpaces careful theoretical analysis. Practitioners have collected a wide array of advice concerning model architectures, hyperparameter tuning, and the required amount of data, while theorists are struggling to justify the choices formally. 
\emph{Pre-training} is one of the well-established practices in the engineer's toolbox that is yet to see a theoretical treatment. In this thesis, we formalize the benefits of pre-training with a focus on applications to contextual multi-armed bandits.

Large neural network are typically pre-trained on a general task and then fine-tuned on another one, for which data is scarce. The classic example comes from computer vision, where we pre-train a classifier on ImageNet~\cite{deng2009imagenet} with millions of images, and refine the weights on a smaller dataset for a medical imaging task. To comply with the task specification, like the number of classes, we need to change the last layer's architecture and retrain it completely. The rest of the weights only have to be adjusted slightly.

The reasons for the wide adoption of pre-training are clear. Data for specialized tasks is expensive, so we would like to rely on the more general, cheaper data as much as possible. Pre-trained weights can also come from a model learned in an unsupervised fashion. The hope is that intermediate layers of the network learn some ``general features'' that encode meaningful information about the input sample. These features should be easier to use for a simple classifier than the original input. On top of that, the success of large language models on a variety of tasks out of the box~\cite{brown2020language} makes pre-training even more appealing compared to training a smaller model.

From a theoretical perspective, pre-training is less understood. How would we quantify the influence of the pre-trained weights on the performance after fine-tuning? The answer naturally depends on some notion of ``quality'' of these weights relative to the task at hand. In this thesis, we formalize this intuition. Depending on the size of the neural network, theorists employ different techniques to explain its generalization ability. In the regime of extremely wide networks, we now know that Stochastic Gradient Descent (SGD) in the parameter space can be viewed as following the so-called kernel gradient in the function space, with the inner product dictated by the Neural Tangent Kernel (NTK)~\cite{jacot2018neural}. When the networks are smaller, we can instead rely on assumptions on the loss function. Classical optimization theory can only provide guarantees of convergence of SGD to the minimum for very limited classes of losses, such as convex~\cite{shalev2014understanding} or strongly convex~\cite{rakhlin2011making} functions. Unfortunately, loss landscapes that arise during neural network training are highly non-convex~\cite{li2018visualizing}. A weaker assumption that we could make is that the loss is only \emph{locally} convex around its minimum. In general, this is not sufficient: how would we know that SGD is initialized close enough to the minimum that the trajectory falls into the convex basin of the loss? This is where pre-training comes into the theoretical story. A key idea in this thesis is that by imposing restrictions on the pre-trained weights, we can ensure that SGD at initialization does fall into the convex basin of the loss and stays there during training.

\emph{Contextual bandit optimization} is concerned with maximizing the cumulative reward of a learning algorithm over a dataset. Data samples, also referred to as {\em contexts}, are shown one by one, and the learner must choose an {\em action} from a given set and receive a {\em reward}. Algorithm's goal can be stated as minimizing regret, which is the expected difference of the sum of received rewards from its theoretical maximum. Crucial to solving this problem is to balance the trade-off of exploration versus exploitation. Exploration takes the form of choosing more diverse actions, also called ``arms''. It gives the algorithm valuable information about their potential rewards, but it also detracts from the cumulative reward, since the algorithm is not exploiting the arms it already knows to be good.

In the contextual bandit setting, analysis in the NTK regime can be used to show sublinear regret when training finite-width neural networks on the observed rewards. NN-UCB~\cite{kassraie2022neural} is the first algorithm to achieve this. Because of the limitations of NTK, however, it requires that the width of the network is polynomial in the horizon of the problem. The regret bound for NN-UCB scales as $O(T^{(2d-1) /2d})$, where $T$ is the horizon and $d$ is the dimension of the context. This thesis solves a similar problem as~\cite{kassraie2022neural}. We also work in a contextual bandit setting, with the rewards modeled by a neural network trained by SGD. However, we are lifting the restriction on the width of the network, replacing it by assumptions on the loss function.

We approach the bandit problem using an Explore-then-commit technique (ETC). In terms of the attained regret guarantees, ETC-type algorithms are known to be suboptimal compared to optimistic algorithms~\cite{lattimore2020bandit}. Moreover, they are not naturally anytime algorithms, i.e. the learner must know the stopping time before it selects the first arm. In spite of these deficiencies, we chose to focus on ETC for two key reasons. First, they are easier to analyze. Random variable dependency structures in optimistic algorithms are more involved than those in ETC, where in the exploration phase the actions and rewards are i.i.d. Before dealing with this technical challenge, it is meaningful to look into the peculiarities of pre-training in a simplified setup. Second, i.i.d.~action-reward pairs during exploration provide data for a classical learning problem, which might be of independent interest in non-adaptive applications. Bounds on the sample complexity of learning from pre-trained weights are the bulk of the results of this thesis. These bounds can be translated into regret bounds for ETC using the technique described in Section~\ref{sec:low_risk_to_low_regret}.\looseness-1

The pre-trained network provides us only with a feature map, therefore we would still need to tune the last year from scratch. The initialization of the last layer weights plays a crucial role as it can lead the SGD path in or out of the convex basin of the loss. To solve this problem, we introduce an algorithm we call \emph{Explore Twice then Commit (E2TC)}. The name reflects the two stages of exploration that the algorithm performs. In the first stage, the weights of the last layer are estimated from the random action-reward pairs using Ridge regression. In the second stage, pre-trained weights and the last layer are jointly refined with SGD. After exploration, the algorithm ``commits'' to the found weights and chooses actions greedily until the end of the episode. The core result of this thesis is that under suitable conditions on the loss, the algorithm can achieve a regret of $\widetilde{O}(T^{4 /5})$ in the stochastic context setup. This is worse than the typical $\widetilde{O}(\sqrt{T})$ regret of optimism-based algorithms, but it compares favorably to the $\widetilde{O}(T^{(2d-1)/2d})$ of NN-UCB. We highlight that out results hold under different conditions than that of NN-UCB. Our assumptions include realizability of the reward function, boundedness of the weights, and convexity of the mean squared error in a basin around the true weights. We also note that NN-UCB works with adversarial contexts, while we only consider the stochastic case.

The thesis is organised as follows. In the following section, we review the related work on neural contextual bandits and adjacent topics. In Chapter~\ref{ch:algo}, we introduce the  bandit setup and the main algorithm formally. We also discuss the techniques for translating guarantees for learning in the exploration phase into the regret of ETC algorithms. Chapter~\ref{ch:linear} contains a detailed analysis of the Ridge regression used in the first stage of exploration, based on the work~\cite{hsu2012random}. Without the second stage of exploration and training, our algorithm solves a misspecified linear bandit problem, for which we show a regret bound of $\widetilde{O}(\epsilon_0 \sqrt{d} KT+(KT)^{4 /5})$, where $K$ is the number of actions, $d$ is the dimension of the last layer and $\epsilon_0$ is a measure of misspecification of the pre-trained weights. This regret guarantee contains a linear term\footnote{For adversarial linear bandits with $\epsilon$-misspecification in $\infty$-norm, a lower bound of $\Omega(\epsilon\min\{T, K\})$ was shown in~\cite[Appendix~F]{lattimore2020learning}. Hence, for large numbers of arms, linear dependence on $T$ is unavoidable in general. }, but it does not rely on the convexity of the loss. In Chapter~\ref{ch:sgd}, we make a detour into high-probability guarantees for SGD. Using a time-uniform version of the Hoeffding-Azuma inequality~\cite{kassraie2023anytime}, we show that with high probability, the SGD trajectory will stay in a convex basin of a generally non-convex loss under a suitable choice of the learning rate. Next, in Chapter~\ref{ch:combined}, we piece together the analyses of Ridge regression and of SGD to provide guarantees for E2TC. We also study the optimal choices of hyperparameters for the algorithm. Finally, in Chapter~\ref{ch:experiments} we compare E2TC to baselines on two real-world datasets, and in Chapter~\ref{ch:conclusion} we conclude the thesis and provide an outlook.





\section{Related Work}
NN-UCB comes from a larger body of work on neural contextual bandits. Most of the results in this area are in the NTK regime, and only consider very wide neural networks. An established non-asymptotic analysis~\cite{arora2019exact} reduces training and inference with a wide neural network to kernel regression with Gaussian Processes (GP). Bandits with reward functions sampled from a GP are better understood~\cite{srinivas2009gaussian,chowdhury2017kernelized}, with contextual~\cite{krause2011contextual} and misspecified~\cite{bogunovic2021misspecified} cases having been studied. Among the neural contextual bandit algorithms is EE-Net~\cite{ban2021ee}, which models both the reward function and the exploration bounty with neural networks and has a $O(\sqrt{T\log T} )$ regret bound. In addition to a polynomial lower bound on the width of the network, the analysis of EE-Net requires a strong condition on the contexts: individual actions in each context should be sampled i.i.d. from some distribution. This assumption is not satisfied, for example, in a classification setup. It is stronger than asking that entire contexts are sampled i.i.d., while the actions can be dependent, as we do. \looseness-1

An empirical study~\cite{riquelme2018deep} compares several approaches to approximate Bayesian inference combined with Thompson Sampling in the contextual bandit setup. It reveales that ``while it deserves further investigation, it seems that decoupling representation learning and uncertainty estimation improves performance''. In particular, the study analyzes an algorithm that trains a deep representation network, mapping actions to meaningful representations. These representations are then used by a Bayesian linear regression to estimate the reward. This is equivalent to using a neural network to regress the reward, where uncertainty is modeled only for the linear last layer's weights. Neural-LinUCB~\cite{xu2020neural} is a algorithm for neural contextual bandits built on this insight. It performs LinUCB-style exploration in the last layer of the learned model and trains the representation network with SGD. The algorithm achieves an $\widetilde{O}(\sqrt{T})$ regret, again requiring the width of the network to grow polynomially with $T$.

The philosophy of \emph{deep representation and shallow exploration}, exhibited in the above two works, is in line with the approach of this thesis. In our case, the pre-trained network provides representation, which might or might not be fine-tuned, while the last layer is trained from scratch and thus induces the bulk of our uncertainty about the weights. The effect of performing one step of SGD on the representation was carefully analyzed in~\cite{ba2022high} in the case of two-layer networks. The authors theoretically show that with a single gradient descent step makes the Ridge regression on the new representation outperform a wide range of kernelized Ridge regression predictors. This conclusion is mirrored in this thesis in the analysis of E2TC, which now makes many update steps on the features, and in the experiments, which show that ``unfreezing'' the representation network weights brings in significant improvement in practice.

Outside the NTK regime, it is common to make assumptions on the shape of loss function. Recently,~\cite{liu2023global} introduced the GO-UCB algorithm, operating in a setup rather similar to ours. GO-UCB trains a neural network using the mean squared error, and its $\widetilde{O}(\sqrt{T})$ regret guarantee is dependent on the local strong convexity of the loss around the minimum and on a strong condition on the Hessian called ``self-concordance'' of the loss at the minimum. These conditions are still not enough to guarantee convergence. To provide a reasonable initialization for the parameters, GO-UCB relies on an Empirical Risk Minimization (ERM) oracle, which solves the intractable global optimization problem.

Arguably, even local strong convexity is unreasonable to expect when training deep neural networks. Strong convexity implies that the smallest eigenvalue $\lambda_{min} (H)$ of the Hessian of the loss is bounded from below by a constant $\mu$. This constant is then used to set the learning rate for training, and the regret guarantees typically have a factor of $1 /\mu$. Empirical studies of the Hessian in deep learning~\cite{sagun2016eigenvalues,ghorbani2019investigation,yao2020pyhessian}, however, show that the Hessian of the loss during training and at the minimum has a large nullspace. This nullspace appears, for example, when the network architecture is overparameterized, i.e. different parameters correspond to the same mapping of inputs to outputs. Any ReLU network is overparameterized because of the \emph{positive scale-invariance} of ReLU:
\begin{equation}
  \forall\lambda >0,\ x\in\R\qquad\text{ReLU}(\lambda x)=\lambda\text{ReLU}(x).
\end{equation}
Thus, scaling the parameters of a fully-connected layer by some $\lambda >0$ and the parameters of the next one by $1 /\lambda$ does not change the mapping given by the network. Even in the case when the problem is under-parameterized and Hessian does not have a nullspace, it is unclear how would the algorithm obtain a realistic lower bound $\mu$.

Unlike strong convexity, non-strong local convexity at the minimum will necessarily hold. This makes it more natural to expect that it will hold locally around the minimum, although this is not strictly guaranteed. For under-parameterized networks, one expects the local minima of the loss to be isoloated and surrounded by a region where the loss is convex. In the over-parameterized regime, the set of minimizers of the loss will generally be a manifold, and local convexity around each point of this manifold cannot hold unless the manifold is an affine subspace. In this case, a better framework is provided by the \emph{Polyak-Łojasiewicz} (PŁ) condition. A function $f:\R^d\to\R$ is said to be $\mu$-PŁ iff
\begin{equation}
  \forall x \in\R^d\qquad \norm{\nabla f(x)}^2 \ge  2\mu \left( f(x)-\min_{x'\in \R^d} f(x')\right).
\end{equation}
This condition is a strict relaxation of $\mu$-strong convexity, and it also implies exponential convergence of gradient descent~\cite{karimi2016linear}. In~\cite{liu2022loss}, the authors argue that the (slightly modified) PŁ condition is satisfied for losses of over-parameterized networks, connect this condition to the minimal eigenvalue of the NTK matrix, and show that SGD converges exponentially to the optimum in expectation. This thesis focuses on the under-parameterized regime and non-strong convexity, but we believe the results can be extended to PŁ losses and thus cover the over-parameterized regime as well.

The regime when the pre-trained weights are kept fixed, and only the last layer is trained was dubbed \emph{weak training} by~\cite{arora2019exact}. Weak training in the bandit setup corresponds to misspecified contextual linear bandits. For adversarial contexts, the authors of~\cite{lattimore2020learning} show that a modification of LinUCB achieves the regret bound $O(\epsilon T\sqrt{d\log T}+d\sqrt{T}\log T)$. Here, $\epsilon$ is the maximum possible misspecification in the reward. For stochastic contexts, we can relax the $\epsilon$ to a different measure of misspecification, based on the expected difference in the reward instead of the $\infty$-norm distance. It also turns out that we can substitute the dependence on $d$ in the sublinear term by a dependence on $K$, the number of actions. This thesis focuses on the case $K \ll d$, where this substitution is beneficial. At the same time, the dependence on $T$ of our bounds is worse ($T^{4 /5}$ instead of $\sqrt{T} $ ), because our algorithm is a flavor of ETC.\looseness-1

We evaluate the quality of the pre-trained weights based on how close they are to the true weights with a certain distance measure. One can instead look into the data distribution from which these weights were generated, and compare the distribution to the one provided by the bandit. This approach was taken by~\cite{zhang2019warm}, who introduce a formal measure of distance between these distributions and a no-regret algorithm that utilizes the pre-training data directly. This approach is conceptually attractive, but limited in its coverage. For us, the interesting cases of pre-trained weights come from tasks that do not correspond to the task the bandit is trying to solve, e.g. they were trained in an unsupervised fashion.

%% file: chapters/algorithm.tex
\chapter{Bandits and optimization}
\label{ch:algo}
\section{Problem Formulation}
\label{sec:formulation}
We consider the following contextual bandit setup. An agent observes a random sequence of contexts $C_1,\ldots,C_T$, each context providing features for $K$ actions:  $C_t=(X_{t,1},\ldots,X_{t,K})$. Individual features $X_{t,i}\in\mathcal{X}$ come from an arbitrary set $\mathcal{X}$. The actions at step $t$ are associated with unknown random rewards $(r_{t,1},\ldots,r_{t,K})$. The agent selects an action $A_t\in[K]$ and observes the reward $r_{t,A_t}$. An optimal action at step $t$ will be denoted as
\begin{equation}
  A_t^*\in\argmax_{a\in[K]}\E[r_{t,a}].
\end{equation}
We define the pseudo-regret of the agent on the time segment from $t_1$ to $t_2$ as
\begin{equation}
  R_{t_1:t_2}=\sum_{t=t_1}^{t_2} \left( r_{t,A_t^*} - r_{t,A_t} \right),
\end{equation}
We will also use notation $R_t:=r_{t,A_t^*}-r_{t,A_t}$ for immediate pseudo-regret at step $t$. Agent's regret on the segment $t_1:t_2$ is $\E[R_{t_1:t_2}]$, where the expectation is taken over the observed contexts, associated rewards, and randomness of the policy. The objective of the agent is to minimize $\E[R_{1:T}]$.

We will approximate the rewards using a neural network mapping input features to predicted rewards:
\begin{equation}
  x\mapsto \w^\tp \phi_{\theta}(x),
\end{equation}
where $x\in\mathcal{X}$ are the input features, $\phi_\theta:\mathcal{X}\to\R^d$ is the \emph{representation network}, parameterized by $\theta\in\R^{d_0}$. $\w\in\R^d$ are the weights of the linear layer on top of the representation network. The joint set of weights of the prediction network will be denoted by $\omega=(\w,\theta)$. Before seeing the contexts, the agent is given $\theta_0\in\R^{d_0}$, the weights of the pre-trained representation network. An initial guess of the weights for the last layer is \emph{not} provided. The dimension $d$ of the last layer weights plays an important role in our analysis and appears in the regret bounds. In contrast, the dimension $d_0$, which would usually be much larger than $d$, does not appear in the analysis --- only the $2$-norm $\norm{\theta}$ of the representation weights matters. We will call this bandit problem a \emph{stochastic contextual bandit with pre-training}.

Now we will discuss the assumptions that we make about the process generating the contexts and the rewards. We begin with the contexts:
\begin{assumption}[Context distribution]
  \label{as:context}
  $C_t\sim\mathcal{D}_C$ are i.i.d. samples from a distribution $\mathcal{D}_C$ over $\mathcal{X}^K$. Furthermore, individual actions $X_{t,a}$ have the same marginal distributions $\mathcal{D}_X$ over $\mathcal{X}$:
  \begin{equation}
    \forall t\in[T],\ a\in[K]\qquad X_{t,a}\sim \mathcal{D}_X.
  \end{equation}
\end{assumption}
Note that this assumption is weaker than requiring that individual actions within the context are independently distributed. For example, if we consider classification in the bandit setup, the contexts are the incoming data samples, and they are often i.i.d., while the actions are given by pairs of the data sample and the chosen class, and thus are not independent. The second part of the assumption is there purely for convenience. If the incoming contexts $C_t$ are i.i.d., but the marginal distributions of actions are not identical, we can make them identical by randomly permuting the actions in the context.

Next, we will assume that our reward model is realizable:
\begin{assumption}[Realizability]
  \label{as:realizability}
  There exist true weights $\wstar\in\R^d$ and $\theta^*\in\R^{d_0}$ and a distribution $\mathcal{H}$ such that the random reward $r_{t,a}$ corresponding to input features $X_{t,a}$ is given by
\begin{equation}
   r_{t,a}=\wstartp\pstar(X_{t,a})+\eta_{t},\quad \eta_t\sim \mathcal{H},
\end{equation}
where $\eta_t$ are i.i.d. zero-mean noise variables coming from $\mathcal{H}$.
\end{assumption}
Realizability gives us a general way of defining the ``quality'' of the pre-trained weights $\theta_0$. We can, in some appropriate way, define the difference between  $\theta_0$ and  $\theta^*$ and provide regret bounds based on this difference. Depending on whether we only learn the last layer or also fine-tune the representation weights  $\theta$, the exact difference definition will change.

Finally, for technical convenience, we will require that all objects in our regret model are bounded and differentiable:
\begin{assumption}[Regularity]
  \label{as:bound}
  We have $\norm{\wstar}\le  B_\w $ and $\norm{\theta^*}\le B_\theta$ for some $B_\w,B_\theta>0$. For each  $x\in\mathcal{X}$, $\phi_\theta(x)$ is differentiable w.r.t. $\theta$. Its Jacobian is given by $\partial \phi_\theta(x) /\partial \theta=J_\theta(x)\in\R^{d\times d_0}$. For all $x\in\mathcal{X}$ and $\theta \in\R^{d_0}$ such that $\norm{\theta}\le B_\theta$, it holds that $\norm{\phi_\theta(x)}\le B_\phi$ and $\norm{J_\theta(x)}_2\le L_\phi$ for some $B_\phi,L_\phi>0$. Finally, $|\eta_t|\le B_\eta$ a.s. for all $t$.
\end{assumption}
Under this assumption and realizability, the rewards will also be bounded, and we will use $B_r$ for this bound for brevity:
\begin{equation}
  |r_{t,a}|\le B_r:=B_\w B_\phi + B_\eta\qquad\text{a.s.}
\end{equation}

For an example bandit setup that approximately satisfies these assumptions, the reader is invited to look at Section~\ref{sec:mnist}, where we evaluate E2TC on MNIST classification.
\section{Algorithm}
We attack the bandit problem described above using an algorithm which we call \emph{Explore Twice then Commit (E2TC)}. As the name suggests, it proceeds in three stages. The details are presented in Algorithm~\ref{alg:main_algo}. 

\begin{algorithm}[t]
\DontPrintSemicolon
\SetNoFillComment
\SetNlSty{}{}{}
\SetKwInput{KwRequire}{Require}
\KwRequire{Network architecture $\phi$, pre-trained feature weights $\theta_0$, horizon length $T$}
\KwRequire{Hyperparameters $\lambda,\zeta_\w,\zeta_\theta \in\R_{>0};\ T_1,T_2\in\mathbb{N}$}
\tcc{Linear last layer estimation}
\For{$t\gets 1$ \KwTo $T_1$}{
  Observe a context $C_t=(X_{t,1},\ldots,X_{T,K})\sim\mathcal{C}$ \;
  Select a random action $A_t\sim U([K])$ and observe the reward $r_t$\;
}
$\widehat{\Sigma(\theta_0)}\gets \frac{1}{T_1}\sum_{t=1}^{T_1}\phi_{\theta_0}(X_t)\phi_{\theta_0}(X_t)^\tp$\;
$\wzero\gets (\widehat{\Sigma(\theta_0)}+\lambda I)^{-1} \frac{1}{T_1}\sum_{t=1}^{T_1}r_t\phi_{\theta_0}(X_t)$\;
\tcc{Preconditioned SGD on $\w$ and $\theta$}
$\w_1\gets\wzero,\ \theta_1\gets \theta_0$\;
\For{$t\gets 1$ \KwTo $T_2$}{
  Observe a context $C_{T_1+t}=(X_{T_1+t,1},\ldots,X_{T_1+t,K})\sim\mathcal{C}$ \;
  Select a random action $\widetilde{A}_t\sim U([K])$ and observe the reward $\tilde{r}_t$\; 
  $\gv^\w_t\gets 2(\w^\tp_t\phi_{\theta_t}(\widetilde{X}_t)-\tilde{r}_t)\phi_{\theta_t}(\widetilde{X}_t)$ \;
  $\gv^\theta_t\gets 2(\w^\tp_t\phi_{\theta_t}(\widetilde{X}_t)-\tilde{r}_t)J_{\theta_t}(\widetilde{X}_t)^\tp \w_t$\;
  $\w_{t+1}\gets \Pi_{\dom_\w}(\w_t-\zeta_\w (\widehat{\Sigma(\theta_0)}+\lambda I)^{-1}\gv^\w_t)$\;
  $\theta_{t+1}\gets \Pi_{\dom_\theta}(\theta_t-\zeta_\w \gv^\theta_t)$\;
}
$\overline{\w}\gets \frac{1}{T_2}\sum_{t=1}^{T_2} \w_t,\quad\overline{\theta}\gets \frac{1}{T_2}\sum_{t=1}^{T_2} \theta_t$\;
\tcc{``Commit'': greedy actions based on $\overline{\w}$ and $\overline{\theta}$}
\For{$t\gets T_1+T_2+1$ \KwTo $T$}{
  Observe a context $C_t=(X_{t,1},\ldots,X_{T,K})\sim\mathcal{C}$ \;
  Select $A_t\gets \argmax_{a\in[K]}\overline{\w}^\tp\phi_{\overline{\theta}}(X_{t,a})$ and observe the reward $r_t$\;
}
\caption{Explore Twice then Commit (E2TC) algorithm for pre-trained bandit problems.}
\label{alg:main_algo}
\end{algorithm}

In the first stage, $T_1$ random actions $A_t\sim U([K])$ are selected and the rewards from these actions collected. The input features $X_{t,A_t}$ selected during this stage will for brevity be denoted as  $X_t$, and the corresponding observed rewards --- by  $r_t$. Since the actions are chosen independently and uniformly, Assumption~\ref{as:context} guarantees that $X_t$ are i.i.d. according to  $\mathcal{D}_X$. We use Ridge regression with regularization strength $\lambda$ to compute $\wzero$, an initial estimate of the weights of the last layer. We use \emph{misspecified} representations $\pzero(X_t)$ in the regression. The problem that we are solving is an instance of \emph{random design Ridge regression}, and we can adapt the analysis~\cite{hsu2012random} to provide guarantees for $\wzero$. These guarantees are given in terms of the covariance of representations. For a given $\theta$, this covariance is defined as
\begin{equation}
  \Sigma(\theta)=\E_X\left[ \phi_\theta(X)^\tp\phi_\theta(X) \right]. \label{eq:sigma_def}
\end{equation}
Here and in the following, the expectation w.r.t. $X$ implies that $X\sim \mathcal{D}_X$. Ridge regression proceeds by providing an estimate of the covariance at the pre-trained weights. This estimate is given by
\begin{equation}
  \widehat{\Sigma(\theta_0)}=\frac{1}{T_1}\sum_{t=1}^{T_1} \pzero(X_t)^\tp\pzero(X_t). \label{eq:hsigma_def}
\end{equation}
Using it, Ridge regression estimates the weights of the last layer as
\begin{equation}
  \wzero=(\widehat{\Sigma(\theta_0)}+\lambda I)^{-1} \frac{1}{T_1}\sum_{t=1}^{T_1}r_t\phi_{\theta_0}(X_t).
\end{equation}

In the second stage, we perform $T_2$ more random actions and collect the rewards. We will use notation $\widetilde{X}_t$ for the features selected during this stage. $\widetilde{X}_t$ are also i.i.d. according to $\mathcal{D}_X$. The observed rewards are $\tilde{r}_t$. We run SGD on the joint weights $\omega=(\w,\theta)$ of the prediction network, initializing it at $(\w_0,\theta_0)$. SGD minimizes the \emph{true risk} $\risk: \dom_\w\times\dom_\theta\to \R$:
\begin{align}
    \risk(\w,\theta)&=\E_{X,r}\left[ \left(\w^\tp\phi_\theta(X)-r \right)^2 \right] \\ 
                    &=\E_{X}\left[ \left(\w^\tp\phi_\theta(X)-\wstartp\pstar(X) \right)^2 \right]+\E_\eta\left[ \eta^2 \right] , \label{eq:risk_decomposed}
\end{align}
where $\dom_\w=\{\w:\ \norm{\w}\le B_\w\}$ is the domain of the weights of the last layer, $\dom_\theta=\{\theta:\ \norm{\theta}\le B_\theta\}$ is the domain of the weights of the representation network, $r=\wstartp\pstar(X)+\eta$ is the random reward associated with the feature $X$ and $\eta\sim\mathcal{H}$ is the noise independent of $X$. At step $t$, gradient estimators of the risk at the current weights $(\w_t,\theta_t)$ are given by
\begin{gather}
  \gv^\w_t= 2(\w^\tp_t\phi_{\theta_t}(\widetilde{X}_t)-\tilde{r}_t)\phi_{\theta_t}(\widetilde{X}_t), \label{eq:gvw_def} \\
  \gv^\theta_t= 2(\w^\tp_t\phi_{\theta_t}(\widetilde{X}_t)-\tilde{r}_t)J_{\theta_t}(\widetilde{X}_t)^\tp \w_t. \label{eq:gvtheta_def}
\end{gather}
Under Assumptions~\ref{as:realizability} and~\ref{as:bound}, these estimators are bounded and unbiased, i.e.
\begin{gather}
  \E\left[ \gv^\w_t\mid \w_t,\theta_t \right] =\nabla_\w\risk(\w_t,\theta_t),\quad \E[ \gv^\theta_t\mid \w_t,\theta_t] =\nabla_\theta\risk(\w_t,\theta_t) \\
  \Vert\gv^\w_t\Vert \le D_\w:=(4B_\w B_\phi+2B_\eta)B_\phi\qquad\text{a.s.,} \label{eq:dw_def} \\
  \Vert\gv^\theta_t\Vert \le D_\theta:=(4B_\w B_\phi+2B_\eta)L_\phi B_\w\qquad\text{a.s.} \label{eq:dtheta_def}
\end{gather}
The algorithm updates the weights using the rules
\begin{equation}
  \w_{t+1}\gets \Pi_{\dom_\w}(\w_t-\zeta_\w (\widehat{\Sigma(\theta_0)}+\lambda I)^{-1}\gv^\w_t),\quad \theta_{t+1}\gets \Pi_{\dom_\theta}(\theta_t-\zeta_\theta \gv^\theta_t). \label{eq:steps}
\end{equation}
Separate learning rates $\zeta_\w$ and  $\zeta_\theta$ can be used to update $\w$ and $\theta$ in practice. Projection  $\Pi_\dom$ onto a set  $\dom$ is defined in the standard way:
\begin{equation}
  \Pi_\dom(x)=\argmin_{x'\in\dom}\norm{x-x'}.
\end{equation}
Since $\dom_\w$ and  $\dom_\theta$ are convex, projection onto them is unique. An important subtlety in~\eqref{eq:steps} is that the update step for $\w$ is multiplied on the left by $(\widehat{\Sigma(\theta_0)}+\lambda I)^{-1}$. Doing this using a general positive-definite matrix is called \emph{pre-conditioning} of SGD and it is sometimes used to improve convergence of SGD in practice. It also leads to different theoretical guarantees, an effect we discuss in more detail in Section~\ref{sec:preconditioned_sgd}. For now, it suffices to say that the guarantees for Ridge regression cannot be plugged into the guarantees for vanilla SGD, but with this pre-conditioning, it becomes possible. SGD outputs the estimated weights $\overline{\omega}=(\overline{\w}, \overline{\theta})$ as the average of the trajectory:
\begin{equation}
  \overline{\w}=\frac{1}{T_2}\sum_{t=1}^{T_2} \w_t,\quad\overline{\theta}=\frac{1}{T_2}\sum_{t=1}^{T_2} \theta_t.
\end{equation}
We will provide a high-probability upper bound on the \emph{suboptimality gap} of these estimates compared to $(\wstar,\theta^*)$:
\begin{equation}
  \risk(\overline{\w},\overline{\theta})-\risk(\wstar,\theta^*)=\E_{X}\left[ \left(\overline{\w}^\tp\phi_{\overline{\theta}}(X)-\wstartp\pstar(X) \right)^2 \right]. \label{eq:suboptimality_gap}
\end{equation}
As we argue in the next section, such guarantees translate into upper bounds on the regret of E2TC.
\section{Low risk to low regret}
\label{sec:low_risk_to_low_regret}
Suboptimality gap in the form~\eqref{eq:suboptimality_gap} is clearly a measure of quality of the approximation of $(\wstar,\theta^*)$ by  $(\overline{\w}, \overline{\theta})$. An intuition behind the main result in this section is guided by this observation. When the approximation is good, greedily selecting an action according to the approximate weights cannot be too much worse than selecting the optimal action. Since the risk function deals with the mean squared error, however, we have to take a square root to translate it into instant regret, which is measured as absolute difference.
\begin{theorem} \label{thm:risktoregret}
  Let E2TC be run on a stochastic contextual bandit with pre-training. Let Assumptions~\ref{as:context},~\ref{as:realizability}, and~\ref{as:bound} hold. Assume also that the weights $(\overline{\w},\overline{\theta})$ after the two exploration phases satisfy the following bound with $\PP\ge 1-\delta$:
\begin{equation}
  \risk(\overline{\w},\overline{\theta})-\risk(\wstar,\theta^*) <\epsilon(T_1,T_2,\delta).\label{eq:abstract_highp_risk}
\end{equation}
Then the regret of E2TC will be bounded by
\begin{equation}
  \E\left[ R_{1:T} \right] \le 2B_\w B_\phi T\delta + 2B_\w B_\phi(T_1+T_2)+K(T-T_1-T_2)\sqrt{\epsilon(T_1,T_2,\delta)}. \label{eq:abstract_regret}
\end{equation}
\end{theorem}
For a given high-probability bound of the form~\eqref{eq:abstract_highp_risk}, this theorem provides a recipe for selecting $T_1$ and $T_2$ by minimizing~\eqref{eq:abstract_regret}. As an example, SGD for convex functions (assuming $T_1=0)$ will give us a guarantee of the form 
\begin{equation}
  \epsilon=O\left(\sqrt{\frac{\log(1 /\delta)}{T_2}}\right).
\end{equation}
Then, for $K\ll T$, we can select $T_2=(KT)^{4 /5}\log^{1/ 5} T$ and  $\delta=T_2 / T$ to get a bound
\begin{equation}
  \E[R_{1:T}]=O\left( (KT)^{4 /5}\log^{1 /5}T \right).
\end{equation}
 For strongly-convex functions~\cite{rakhlin2011making}, we can instead get a bound of the form 
\begin{equation}
  \epsilon=O\left( \frac{\log((\log T_2) / \delta)}{T_2} \right),
\end{equation}
leading to a regret guarantee $\E[R_{1:T}]=\widetilde{O}((KT)^{2 /3})$ for an appropriate $T_2$.

\begin{proof}[Theorem~\ref{thm:risktoregret}]
  We will use notation $\overline{\omega}=(\overline{\w},\overline{\theta}),\ \omega^*=(\w^*,\theta^*)$, and $\omega_t=(\w_t,\theta_t)$ throughout the proof. Let the event $S$ (from ``success'') be
\begin{equation}
  S=\left\{  \risk(\overline{\omega})-\risk(\omega^*)< \epsilon(T_1,T_2,\delta)\right\}.
\end{equation}
Its complement will be called $F$. By assumption, $\PP[F]\le \delta$.
Next, given an arbitrary action $A_t$, we can bound its immediate regret by
\begin{equation}
  \E[R_t]=\E[\wstartp\pstar(X_{t,A^*_t})- \wstartp\pstar(X_{t,A_t})] \le 2B_\w B_\phi. \label{eq:bad_regret}
\end{equation}
In steps starting from $T_1+T_2+1$, the action is chosen as 
\begin{equation}
 A_t=\argmax_{a\in[K]}\overline{\w}\phi_{\overline{\theta}}(X_{t,a}),
\end{equation}
so we can bound the immediate regret using the quality of the estimate $\overline{\omega}$ :
\begin{align}
  \E\left[R_t\right]&=\E\left[\wstartp\pstar(X_{t,A^*_t})-\wstartp\pstar(X_{t,A_t})\right] \\
  &= \E\left[\wstartp\pstar(X_{t,A^*_t})-\wstartp\pstar(X_{t,A_t})|\ A_t\ne A_t^*\right]\PP\left[ A_t \ne A_t^* \right]  \\
    &= \E\big[\wstartp\pstar(X_{t,A^*_t}) -\overline{\w}^\tp\phi_{\overline{\theta}}(X_{t,A^*_t})+\overline{\w}^\tp\phi_{\overline{\theta}}(X_{t,A_t})- \wstartp\pstar(X_{t,A_t}) \nonumber \\
    &\qquad+\underbrace{\overline{\w}^\tp\phi_{\overline{\theta}}(X_{t,A^*_t})-\overline{\w}^\tp\phi_{\overline{\theta}}(X_{t,A_t})}_{\le 0,\text{ since $A_t$ maximizes $\overline{\w}^\tp\phi_{\overline{\theta}}$}}|\ A_t\ne A^*_t\big]\PP\left[ A_t \ne A_t^* \right] \\
    &\le\E\left[\sum_{a=1}^K\left| \wstartp\pstar(X_{t,a})-\overline{\w}^\tp\phi_{\overline{\theta}}(X_{t,a}) \right| \right]. \label{eq:imm_bound}
\end{align}
If S holds, then the cumulative regret in the ``commit'' phase of the algorithm can be bounded with the help of~\eqref{eq:imm_bound}. $\E\left[R_{T_1+T_2+1:T}\mid S\right]=$
\begin{align}
  \qquad&\phantom{\le}\E\left[\left.\sum_{t=T_1+T_2+1}^T\left(\wstartp\pstar(X_{t,A^*_t})- \wstartp\pstar(X_{t,A_t})\right)\ \right| S\right] \\
        &\le \E\left[\left. \sum_{t=T_1+T_2+1}^T\sum_{a=1}^K\left| \wstartp\pstar(X_{t,a})-\overline{\w}^\tp\phi_{\overline{\theta}}(X_{t,a}) \right|\ \right| S\right] \\
        & \le \sqrt{\E\left[\left.\left(\sum_{t=T_1+T_2+1}^T\sum_{a=1}^K\left|\wstartp\pstar(X_{t,a})-\overline{\w}^\tp\phi_{\overline{\theta}}(X_{t,a})\right| \right)^2\ \right| S \right]} \label{eq:variance_ineq} \\
       & \le \sqrt{ K(T-T_1-T_2)\E\left[\left.\sum_{t=T_1+T_2+1}^T\sum_{a=1}^K\left(\wstartp\pstar(X_{t,a})-\overline{\w}^\tp\phi_{\overline{\theta}}(X_{t,a}) \right)^2\ \right| S \right] }. \label{eq:sumsq}
\end{align}
Analyzing the randomness in~\eqref{eq:sumsq} carefully, we see that it is coming from two independent sources: $\overline{\omega}$ is random, and it depends on $C_1, \ldots,C_{T_1+T_2}$ and the random actions selected during exploration. The distribution of $\overline{\omega}$ is conditioned on the event $S$. $C_{T_1+T_2+1}, \ldots,C_{T}$, on the other hand, are independent from $\overline{\omega}$ and from $S$. Thus, we can rewrite~\eqref{eq:sumsq} as
\begin{align}
  &\E[R_{T_1+T_2+1:T}\mid S]\le \sqrt{ K(T-T_1-T_2)}\times \label{eq:exp_split}\\
  &\times\sqrt{\E_{\overline{\omega}\sim p(\overline{\omega}|S)}\E_{C_{T_1+T_2+1}, \ldots,C_T}\left[\sum_{t=T_1+T_2+1}^T\sum_{a=1}^K\left(\wstartp\pstar(X_{t,a})-\overline{\w}^\tp\phi_{\overline{\theta}}(X_{t,a}) \right)^2\right] }. \nonumber
\end{align}
Here we notice that the expectation w.r.t. $C_t$ of each term in~\eqref{eq:exp_split} only depends on the marginal distribution of $X_{t,a}$, which is $\mathcal{D}_X$ by Assumption~\ref{as:context}. Together with~\eqref{eq:suboptimality_gap} and the definition of $S$, this gives us
\begin{align}
       \E[R_{T_1+T_2+1:T}\mid S]& \le K(T-T_1-T_2)\sqrt{\E_{\overline{\omega}\sim p(\overline{\omega}|S)}\left[\risk(\overline{\omega})-\risk(\omega^*)\right] } \\
       & \le K(T-T_1-T_2)\sqrt{\epsilon(T_1,T_2,\delta)} . \label{eq:rsucc_ineq}
\end{align}
Using~\eqref{eq:bad_regret},
\begin{align}
  \!\!\E[R_{1:T}] &=\E[R_{1:T}\mid F]\PP[F]+\left(\E[R_{1:T_1+T_2}\mid S]+\E[R_{T_1+T_2+1:T}\mid S]\right)\PP[S] \\
              & \le 2B_\w B_\phi T\delta + 2B_\w B_\phi (T_1+T_2)+K(T-T_1-T_2)\sqrt{\epsilon(T_1,T_2,\delta)}.
\end{align}
\end{proof}

In the rest of the thesis, we will focus on upper bounds on the suboptimality gap, translating them into regret guarantees using this theorem. Providing such bounds is equivalent to solving the sample complexity problem, formulated as ``how many data samples do we need to bring the suboptimality gap below a certain threshold''. When we are not in the bandit setup, but instead simply do regression with a pre-trained representation network from i.i.d. data, a bound of the form~\eqref{eq:abstract_highp_risk} is also useful. Given $T_0$ data samples and a desired probability of success $\delta$, we should minimize  $\epsilon(T_1,T_0-T_1,\delta)$ w.r.t. $T_1$ to learn how many samples to spend on the initialization of $\w$, and how many --- on SGD.

\section{Empirical risk minimization}
Guarantees for SGD can be given in terms of a bound on the suboptimality gap~\eqref{eq:suboptimality_gap}. We directly optimize the true risk $\risk$, but we only use estimators of its gradients. In the thesis we focus on this approach. However, a mirrored approach is also popular in the literature, under the name of \emph{empirical risk minimization}~\cite{shalev2014understanding}. In this section, we will briefly discuss how it can be used to provide regret bounds for stochastic contextual bandits with pre-training similar to the ones that we have. 

Given $T_2$ data samples $(\widetilde{X}_t,\tilde{r}_t)_{t=1}^{T_2}$, we define empirical risk as
\begin{equation}
  \erisk(\w,\theta)=\frac{1}{T_2}\sum_{t=1}^{T_2} (\w^\tp\phi_\theta(\widetilde{X}_t)-\tilde{r}_r)^2.
\end{equation}
We can compute the true gradients of $\erisk$ and optimize it with regular gradient descent (GD), but now the risk itself is an approximation. Given sufficient computational resources, the minimum of the empirical risk can be found with arbitrary precision. Then, we use a \emph{uniform law}, stating that with $\PP\ge 1-\delta$, 
\begin{equation}
  \forall \omega\qquad|\risk(\omega)-\erisk(\omega) |\le \hat{\epsilon}(\delta)
\end{equation}
for some $\hat{\epsilon}(\delta)$. Let $\widehat{\omega}\in\argmin_{\omega}\erisk(\omega)$ be the weights minimizing the empirical risk. Now it is easy to show that, with $\PP\ge 1-\delta$,
\begin{equation}
  \risk(\widehat{\omega})-\risk(\omega^*)\le 2\hat{\epsilon}(\delta). \label{eq:subopt_gen_bound}
\end{equation}
Naturally, the catch here is ``sufficient computational resources''. We focus on SGD instead of this approach because full-batch gradient descent (required to compute the full gradient of $\erisk$ ) is very rare in practice. We note in passing that SGD on $\erisk$ with multiple passes through the data can also be used and there are results providing suboptimality gap bounds in this case without the need for a uniform bound~\cite{lin2017optimal}. This direction can be interesting for future research.

Uniform laws for risk functions associated with general classes of losses can be derived through \emph{Rademacher complexity}. For a class of functions $\mathcal{F}\subset\left\{ f: \mathcal{X}\to\R \right\} $, it is defined as
\begin{equation}
  \mathcal{R}_n(\mathcal{F})=\E_{X,\epsilon}\left[ \sup_{f\in\mathcal{F}}\left| \frac{1}{n}\sum_{k=1}^{n} \epsilon_i f(X_i) \right|  \right],
\end{equation}
where $X_1,\ldots,X_n$ are i.i.d. samples from $\mathcal{D}_X$ and $\epsilon_1,\ldots,\epsilon_n$ are i.i.d. samples from $U[\left\{ -1,1 \right\} ]$. The class of function relevant to us is
\begin{equation}
  \mathcal{F}=\left\{x\mapsto (\w^\tp\phi_\theta(x)-r)^2 \right\},
\end{equation}
where $x\in\mathcal{X}$, $\norm{\w}\le B_\w$, $\norm{\theta}\le B_\theta$, and $|r|\le B_r$. Any $f\in\mathcal{F}$ satisfies $\norm{f}_\infty \le (B_\w B_\phi+B_r)^2$. Theorem~4.10 in~\cite{wainwright2019high} now gives us that
\begin{equation}
  \PP\left[\forall\omega\quad |\risk(\omega)-\erisk(\omega)|\le 2\mathcal{R}_{T_2}(\mathcal{F})+(B_\w B_\phi + B_r)^2\sqrt{\frac{2\log(1 /\delta)}{T_2}} \right] \ge 1-\delta. \label{eq:use_rad_nn}
\end{equation}

Rademacher complexity of neural networks has been extensively studied, with one of the recent results being~\cite{golowich2018size}. Authors of this paper improve the dependence of Rademacher complexity on the depth of the network, but for us the most important part is that it depends on $n$ as  
\begin{equation}
  \mathcal{R}_n(\mathcal{F})=O(n^{-1 /2}) \label{eq:rad_nn}
\end{equation}
as long as the depth of the network is kept fixed and Frobenius norms of parameters are bounded and some assumptions on activation functions are satisfied. In this case, we require that the input features come from a space $\mathcal{X}\subset\R^{d_{in}}$ with bounded norms. Note that the paper actually bounds the complexity of the class
\begin{equation}
  \widetilde{\mathcal{F}}=\left\{ x\mapsto \w^\tp\phi_\theta(x) \right\}, 
\end{equation}
but the contraction lemma (Lemma~26.9 in~\cite{shalev2014understanding}) and Lipschitzness of $x\mapsto (x-a)^2$ defined on a bounded subset of $\R$ give us that $\mathcal{R}_n(\mathcal{F})=O(\mathcal{R}_n(\widetilde{\mathcal{F}}))$. Substituting~\eqref{eq:rad_nn} into~\eqref{eq:use_rad_nn} and the result into~\eqref{eq:subopt_gen_bound} we get that, with $\PP\ge 1-\delta$,
\begin{equation}
   \risk(\widehat{\omega})-\risk(\omega^*)=O\left(\sqrt{\frac{\log(1/\delta)}{T_2}}\right). \label{eq:rademacher_sub}
\end{equation}
This is the same asymptotics as what we will show for SGD in Chapter~\ref{ch:sgd} when only one pass is made through the data, and the true risk $\risk$ is (locally) convex.

It turns out that the asymptotic bound~\eqref{eq:rademacher_sub} can actually be improved. Theorem~C.1 in~\cite{liu2023global}, itself adapted from lecture notes~\cite{nowak2007complexity}, shows that the suboptimality gap for mean squared losses can be bounded as
\begin{equation}
  \risk(\widehat{\omega})-\risk(\omega^*)= O\left( \frac{\log(T_2 /\delta)}{T_2} \right),
\end{equation}
where $\widehat{\omega}$ is the exact minimizer of the empirical risk $\erisk$. This improvement over~\eqref{eq:rademacher_sub} could potentially be adapted to the more realistic case of an inexact minimizer of $\erisk$ and lead to a reduction in regret bounds from  $O(T^{4 /5})$ that we currently have to  $O(T^{2 /3})$, but efficient optimization methods of $\erisk$ would still be required.

%% file: chapters/linear.tex
\chapter{Linear estimation of the last layer}
\label{ch:linear}
In this chapter we provide guarantees on the initial value $\w_0$ of the last layer weights that we compute in Algorithm~\ref{alg:main_algo} using Ridge regression. In this and the following chapters, we will use notation
\begin{equation}
  \sigz=\Sigma(\theta_0),\quad\sigl=\Sigma(\theta_0)+\lambda I,\quad\hsigz=\widehat{\Sigma(\theta_0)},\quad\hsigl=\widehat{\Sigma(\theta_0)}+\lambda I. \label{eq:succ_sigdef}
\end{equation}
For convenience, we repeat the equations that define these matrices and $\wzero$:
\begin{gather}
  \sigz=\E_X\left[ \pzero(X)\pzero(X)^\tp \right],\quad \hsigz=\frac{1}{T_1}\sum_{t=1}^{T_1} \pzero(X_t)^\tp\pzero(X_t), \\
  \wzero=\hsigl^{-1} \frac{1}{T_1}\sum_{t=1}^{T_1}r_t\phi_{\theta_0}(X_t). \label{eq:wzero_newdef}
\end{gather}
Here, $X_t\sim\mathcal{D}_X$ are i.i.d. samples  that correspond to the random features of the actions that we choose during the first phase of exploration in E2TC. 

We would like to provide a high-probability bound on ``closeness'' between $\mathbf{w}_0$ and $\mathbf{w}^*$, for some appropriate definition of closeness. Since we are using the misspecified weights $\theta_0$, we can't expect this closeness to go to zero with $T_1$, but only down to some constant that corresponds to the cost of misspecification. Two notions of closeness will be relevant to us: $\norm{\wzero-\wstar}_{\hsigl}$ and $\norm{\wzero-\wstar}_{\sigz}$. The first expression is easier to bound. In fact, in the bandit literature, this quantity is bounded \emph{for any} sequence of $\{X_t\}_t$, w.h.p. over the noise  $\eta_t$. In the following section, we provide a bound on this quantity. The second quantity has been analyzed in the literature on \emph{random design Ridge regression}. We consider it in Section~\ref{sec:ridge_lls}. Note that we cannot bound $\norm{\w-\wstar}_2$. This is a fundamental limitation: if $\mathcal{D}_X$ is only supported on a subspace of $\R^d$, we cannot learn anything about the projections of $\wstar$ onto directions orthogonal to this subspace. To work around this limitation, as we will see, we can perform pre-conditioned SGD, guarantees for which are given in terms of matrix-weighted distance from $\wstar$ to  $\wzero$.

\section{Fixed design Ridge regression}
\label{sec:bandit_lls}


Due to realizability, the rewards are generated as $r_t=\wstartp\pstar(X_t)+\eta_t$. Our Ridge estimate~\eqref{eq:wzero_newdef} relies on $\pzero(X_t)$. Misspecification comes from the discrepancy between $\phi_{\theta_0}$ and  $\phi_{\theta^*}$. Let
\begin{equation}
  \label{eq:deltaphi_def}
 \Delta\phi(x)=\phi_{\theta^*}(x)-\phi_{\theta_0}(x),\quad \Sigma_\delta=\E_X\left[ \Delta\phi(X)\Delta\phi(X)^\tp \right].
\end{equation}
With this notation,
\begin{gather}
  r_t=\left(\phi_{\theta_0}(X_t)+\Delta\phi(X_t)\right)^\tp\mathbf{w}^{*}+\eta_t, \\
  \mathbf{w}_0=\mathbf{w}^*-\lambda \hsigl^{-1}\mathbf{w}^*+\hsigl^{-1} \frac{1}{T_1}\sum_{t=1}^{T_1} \left( \phi_{\theta_0}(X_t)\Delta\phi(X_t)^\tp\mathbf{w}^*+\phi_{\theta_0}(X_t)\eta_t \right). \label{eq:rewritew0}
\end{gather}
Let us now focus on the suboptimality gap of $\omega_0=(\wzero,\theta_0)$:
\begin{align}
  &\sqrt{\left(\risk(\omega_0)-\risk(\omega^*)\right)} \nonumber \\
  &\quad=\sqrt{\E_{X}\left[\left(\mathbf{w}_0^\tp\phi_{\theta_0}(X)-\mathbf{w}^{*\tp}\phi_{\theta^*}(X)\right)^2 \right]} \label{eq:boundloss} \\
                                              &\quad\le \sqrt{\E_X\left[ \left((\mathbf{w}_0-\mathbf{w}^*)^\tp\phi_{\theta_0}(X)\right)^2 \right]}+\sqrt{\E_X\left[\left( \mathbf{w}^{*\tp}\Delta\phi(X)\right)^2 \right]}   \\
                                              &\quad= \norm{\mathbf{w}_0-\mathbf{w}^*}_{\sigz}+\norm{\mathbf{w}^*}_{\sigd  }. \label{eq:twoterms}
\end{align}
In the first equality we used~\eqref{eq:risk_decomposed}. The first term in the above is the main character of Section~\ref{sec:ridge_lls}. The second term has a somewhat intuitive interpretation: it measures the average discrepancy between the true feature map $\phi_{\theta^*}$ and the initial one $\phi_{\theta_0}$ in the direction of the true weights $\mathbf{w}^*$. As it turns out, this term is an appropriate measure of misspecification of $\theta_0$ compared to  $\theta^*$, and we can derive the bounds on  $\norm{\wzero-\wstar}_{\hsigl}$ and $\norm{\wzero-\wstar}_{\sigz}$ using it. We will use notation
\begin{equation}
  \epsilon_0^2=\norm{\wstar}_{\sigd}^2=\E_X\left[ \left( \wstartp(\pzero(X)-\pstar(X)) \right) ^2 \right]. \label{eq:epsilonzero_def}
\end{equation}
Note that requiring a small $\epsilon_0$ is weaker than, for example, requiring small $\E_X\left[ \norm{\phi_{\theta_0}(X)-\phi_{\theta^*}(X)}^2 \right] $ or $\norm{\theta_0-\theta^*}^2$, although the latter will be required for SGD analysis later.

Now, we substitute~\eqref{eq:rewritew0} into the distance we are trying to bound:
\begin{align}
  &\norm{\wzero-\wstar}_{\hsigl} \nonumber \\
  &\quad=\left\Vert -\lambda\wstar+\frac{1}{T_1}\sum_{t=1}^{T_1}\left( \pzero(X_t)\Delta\phi(X_t)^\tp\wstar+\pzero(X_t)\eta_t\right)\right\Vert_{\hsigl^{-1}} \\
  &\quad\le \lambda\norm{\wstar}_{\hsigl^{-1}}+\left\Vert{\frac{1}{T_1}\sum_{t=1}^{T_1} \pzero(X_t)\Delta\phi(X_t)^\tp\wstar}\right\Vert_{\hsigl^{-1}}+\left\Vert{\frac{1}{T_1}\sum_{t=1}^{T_1} \pzero(X_t)\eta_t}\right\Vert_{\hsigl^{-1}}. \label{eq:hsiglnorm_decomp}
\end{align}
The first term here can be bounded by observing that $\hsigl^{-1}\le \lambda^{-1}I$, so
\begin{equation}
  \lambda\norm{\wstar}_{\hsigl^{-1}}\le \sqrt{\lambda} \norm{\wstar}. \label{eq:bound_lambdanorm}
\end{equation}
The second term in~\eqref{eq:hsiglnorm_decomp} is the cost of misspecification. The following lemma bounds it in terms of $\epsilon_0$.
\begin{lemma}[The cost of misspecification]\label{lem:misspec}
  For all $\delta>0$,
  \begin{equation}
    \PP\left[\left\Vert \frac{1}{T_1}\sum_{t=1}^{T_1} \left(\phi_{\theta_0}(X_t)\Delta\phi(X_t)^\tp\mathbf{w}^{*}\right)\right\Vert_{\hsigl^{-1}} < \sqrt{d\left( \epsilon_0^2+\epsilon_\delta \right)}\right] \ge 1-\delta, \label{eq:boundmisspec}
  \end{equation}
  where
  \begin{equation}
    \epsilon_{\delta  }=4B_\w ^2B_\phi^2\sqrt{\frac{1}{2T_1}\log \frac{1}{\delta}}.
  \end{equation}
\end{lemma}
\begin{proof}
We define the empirical estimator of $\Sigma_\delta$ as
\begin{equation}
  \widehat{\Sigma}_\delta= \frac{1}{T_1}\sum_{t=1}^{T_1} \Delta\phi(X_t)\Delta\phi(X_t)^\tp.
\end{equation}
With this notation,
\begin{align}
  &\left\Vert{\frac{1}{T_1}\sum_{t=1}^{T_1} \left(\phi_{\theta_0}(X_t)\Delta\phi(X_t)^\tp\mathbf{w}^{*}\right)}\right\Vert_{\hsigl^{-1}} \nonumber \\
  &\qquad\qquad\qquad\le \frac{1}{T_1}\sum_{t=1}^{T_1} \left| \Delta\phi(X_t)^\tp\mathbf{w}^* \right| \norm{\phi_{\theta_0}(X_t)}_{\hsigl^{-1}} \\
  &\qquad\qquad\qquad\le \sqrt{\frac{1}{T^2_1}\sum_{t=1}^{T_1}\left(\Delta\phi(X_t)^\tp\mathbf{w}^*  \right)^2}\sqrt{\sum_{t=1}^{T_1} \norm{\phi_{\theta_0}(X_t)}_{\hsigl^{-1}}^2 } \\
  &\qquad\qquad\qquad= \frac{\norm{\mathbf{w}^*}_{\widehat{\Sigma}_\delta}}{\sqrt{T_1} }\sqrt{\tr\left[\hsigl^{-1}\sum_{t=1}^{T_1}\pzero(X_t)\pzero(X_t)^\tp   \right]}  \\
  &\qquad\qquad\qquad= \frac{\norm{\mathbf{w}^*}_{\widehat{\Sigma}_\delta}}{\sqrt{T_1} }\sqrt{T_1 \tr\left[I-\lambda \hsigl^{-1} \right]}\le \sqrt{d}\norm{\mathbf{w}^*}_{\widehat{\Sigma}_\delta}. \label{eq:wdeltamhat}
\end{align}
In~\eqref{eq:wdeltamhat}, we used that $\hsigz=\hsigl-\lambda I$ and the definition of $\hsigz$. Since $\mathbf{w}^*$ is not random, Hoeffding's inequality is sufficient to replace $\norm{\wstar}_{\widehat{\Sigma}_\delta}$ by $\norm{\mathbf{w}^*}_{\Sigma_\delta  }$. Let 
\begin{equation}
  Y_t=(\Delta\phi(X_t)^\tp\wstar)^2.
\end{equation}
These are i.i.d. random variables with mean $\E[Y_t]=\norm{\wstar}_{\sigd}^2$. Their empirical mean is 
\begin{equation}
  \frac{1}{T_1}\sum_{t=1}^{T_1}Y_t=\norm{\wstar}_{\widehat{\Sigma}_\delta}^2.
\end{equation}
We also have that  $0\le Y_t\le 4B_\w^2 B_\phi^2$ a.s. Hence, we can apply Hoeffding's inequality to $Y_t$. Introducing an event $E_\delta$, bound $\epsilon_\delta$, and a probability of failure $\delta$, we have
\begin{gather}
  E_{\delta  }=\left\{ \norm{\mathbf{w}^*}^2_{\widehat{\Sigma}_\delta}-\norm{\mathbf{w}^*}^2_{\Sigma_\delta  }<\epsilon_{\delta  } \right\},\quad\PP\left[ E_{\delta  } \right] \ge  1-\delta, \\
  \epsilon_{\delta  }=4B_\w ^2B_\phi^2\sqrt{\frac{1}{2T_1}\log \frac{1}{\delta}}.
\end{gather}
When $E_{\delta}$ holds, we can continue~\eqref{eq:wdeltamhat} as
\begin{equation}
  \sqrt{d}\norm{\mathbf{w}^*}_{\widehat{\Sigma}_\delta} < \sqrt{d\left(\norm{\mathbf{w}^*}^2_{\Sigma_\delta  }+\epsilon_\delta\right)}=\sqrt{d(\epsilon_0^2+\epsilon_\delta)} 
\end{equation}
and arrive at~\eqref{eq:boundmisspec}.  \[\]
\end{proof}

The next lemma bounds the third term in~\eqref{eq:hsiglnorm_decomp}. Since the proof closely follows the exposition at the beginning of Chapter~20 in~\cite{lattimore2020bandit}, we defer it to Appendix~\ref{ap:proof_wellspec_new}.
\begin{lemma}[Noise in well-specified linear bandits] \label{lem:wellspec_new}
  \begin{equation}
  \PP\left[\left\Vert{\frac{1}{T_1} \sum_{t=1}^{T_1} \phi_{\theta_0}(X_t)\eta_t}\right\Vert_{\hsigl^{-1}}<\epsilon_\eta \right] \ge 1-\delta, \label{eq:boundeta}
  \end{equation}
  where
\begin{equation}
  \epsilon_\eta=2\sqrt{\frac{d\log 6 + \log(1 /\delta)}{T_1}}.
\end{equation}
\end{lemma}

We are now ready to continue the chain of inequalities~\eqref{eq:hsiglnorm_decomp}. Using~\eqref{eq:bound_lambdanorm}, Lemma~\ref{lem:misspec}, and Lemma~\ref{lem:wellspec_new}, we have, with $\PP\ge 1-2\delta$,
\begin{equation}
  \norm{\wzero-\wstar}_{\hsigl}\le \sqrt{d(\epsilon_0^2+\epsilon_\delta)}+\epsilon_\eta+\sqrt{\lambda} \norm{\wstar}. \label{eq:highp_bound_hsigl}
\end{equation}
This is the result we needed. Asymptotically, it reads as
\begin{equation}
  \norm{\wzero-\wstar}_{\hsigl}^2=\widetilde{O}\left(\epsilon_0^2d+\frac{d}{\sqrt{T_1} }+\lambda\right),
\end{equation}
where $\widetilde{O}$ hides the dependence on $\log(1 /\delta)$. At this point, we could also convert this bound to one on $\norm{\wzero-\wstar}_{\sigl}$ by using a bound on $\norm{\sigz-\hsigz}_2$, provided by Lemma~\ref{lem:matmomentum}. With a suitable $\lambda$, this would give a guarantee  $\norm{\wzero-\wstar}_{\sigl}^2=O(\epsilon_0^2d+d /\sqrt{T_1} )$. This turns out to be suboptimal. In the next section, we use more advanced techniques to show that this bound can be improved to $\widetilde{O}(\epsilon_0^2d+ d /T_1)$ or to $\widetilde{O}(\epsilon_0^2d+1 /\sqrt{T_1} )$ with high probability, depending on the regularization.

\section{Random design Ridge regression}
\label{sec:ridge_lls}
The approach we outlined in the last paragraph of the previous section relies on a bound on $\norm{\sigz-\hsigz}_2$ in covariance estimation. As it turns out, we can instead rely on the \emph{relative} error in covariance estimation $\norm{\sigl^{1 /2}\hsigl^{-1}\sigl^{1 /2}}_2$. This idea is exploited in the work~\cite{hsu2012random}, on which we base the analysis in this section. 

When the weights of the feature map $\phi$ are fixed at $\theta_0$, the estimate $\mathbf{w}_0$ given by~\eqref{eq:wzero_newdef} is solving the Ridge regression problem
\begin{equation}
  \mathbf{w}_0=\argmin_{\mathbf{w}}\frac{1}{T_1}\sum_{t=1}^{T_1} \left(\mathbf{w}^\tp\pzero(X_t)-r_t\right)^2+\lambda\norm{\mathbf{w}}^2.
\end{equation}
The way our algorithm is designed, $\pzero(X_t)$ are i.i.d. random variables for $t$ from  $1$ to $T_1$. Therefore, the problem can be framed as a \emph{random design} ridge regression. This problem is studied in the literature, with the goal of bounding the difference of the estimate $\wzero$ to the minimizer $\widetilde{\w}$ of the true square error. In our notation, this minimizer is given by
\begin{align}
  \widetilde{\w}&\in\argmin_\w \E_{X,\eta}\left[\left( \w^\tp\pzero(X)-\wstartp\pstar(X)-\eta \right) ^2\right] \label{eq:wtildef} \\
  &= \argmin_\w \E_{X}\left[\left( \w^\tp\pzero(X)-\wstartp\pstar(X)\right) ^2\right].
\end{align}
Authors of~\cite{hsu2012random} give a high-probability upper bound on the \emph{excess risk} of $\wzero$, which can be conveniently expressed as
\begin{multline}
  \E_{X}\left[\left( \wzero^\tp\pzero(X)-\wstartp\pstar(X)\right) ^2\right]-\E_{X}\left[\left( \widetilde{\w}^\tp\pzero(X)-\wstartp\pstar(X)\right) ^2\right] \\
  =\norm{\wzero-\widetilde{\w}}_{\sigz}^2. \label{eq:exrisk}
\end{multline}
Note that the ambiguity of $\argmin$ in~\eqref{eq:wtildef} is such that all minimizers produce the same excess risk landscape: for any $\w\in\R^d$ and two minimizers $\widetilde{\w}_1,\widetilde{\w}_2$ of~\eqref{eq:wtildef}, it holds that
\begin{equation}
  \norm{\w-\widetilde{\w}_1}_{\sigz}^2=\norm{\w-\widetilde{\w}_2}_{\sigz}^2,
\end{equation}
because $\widetilde{\w}_1-\widetilde{\w}_2$ lies in the nullspace of $\sigz$. The bounds on the excess risk in~\cite{hsu2012random} are given in terms of the \emph{effective dimension} of the problem,
\begin{equation}
  d_{p,{\lambda}}=\sum_{j=1}^{d} \left(\frac{\lambda_j}{\lambda_j+{\lambda}}\right)^p\quad\text{for $p\in\{1,2\}$}, \label{eq:deffot_def}
  \end{equation}
  where $\{\lambda_j\}_j$ are the eigenvalues of  $\sigz$. It always holds that $d_{p,\lambda}\le d$, but we try to preserve $d_{p,\lambda}$ whenever possible, because sometimes this allows to substitute explicit $d$ by a bound on the norm of $\wstar$ or $\phi_\theta(X)$. The following inequalities will be helpful for this:
\begin{equation}
  d_{1,\lambda}\le \frac{\tr \sigz}{\lambda},\quad d_{2,\lambda}\le \frac{\tr \sigz}{2\lambda},\quad \tr \sigz=\E_X\left[ \norm{\pzero(X)}^2 \right]  \le B_\phi^2. \label{eq:boundeffdim}
\end{equation}
Apart from the notions of effective dimension above, we will use another two. The first is introduced in~\cite{hsu2012random}. The second we introduce ourselves.
\begin{equation}
  \ceff=\max\left\{ \deffo,1 \right\},\qquad  \dht=\sum_{j=1}^{d} \left(\frac{\lambda}{\lambda_j+\lambda} \right)^2 \le d. \label{eq:extra_effdim_def}
\end{equation}
We will denote the unit eigenvector of $\sigz$ corresponding to  $\lambda_i$ as  $\mathbf{v}_i$. $\sigz$ admits a spectral decomposition
\begin{equation}
  \sigz=\sum_{i=1}^{d} \lambda_i \mathbf{v}_i\mathbf{v}_i^\tp.
\end{equation}
Using the definition~\eqref{eq:epsilonzero_def} of $\epsilon_0$, bounds on the excess risk~\eqref{eq:exrisk} translate into bounds on our excess loss of the estimated weights. We first note that by definition of $\widetilde{\w}$,
  \begin{equation}
    \E_X\left[\left( \widetilde{\w}^\tp\pzero(X)-\wstartp\pstar(X)\right) ^2  \right] \le \E_X\left[\left( \wstartp\pzero(X)-\wstartp\pstar(X)\right) ^2  \right] = \epsilon_0^2.\label{eq:boundapprox}
\end{equation}
Next, continuing~\eqref{eq:boundloss},
\begin{align}
  &\sqrt{\risk(\omega_0)-\risk(\omega^*)} \nonumber\\
  &\quad=\sqrt{\E_{X}\left[\left(\mathbf{w}_0^\tp\phi_{\theta_0}(X)-\mathbf{w}^{*\tp}\phi_{\theta^*}(X)\right)^2 \right]} \\
  &\quad\le \sqrt{\E_{X}\left[\left(\mathbf{w}_0^\tp\phi_{\theta_0}(X)-\widetilde{\w}^{\tp}\phi_{\theta_0}(X)\right)^2 \right]}+\sqrt{\E_{X}\left[\left(\widetilde{\w}^{\tp}\phi_{\theta_0}(X)-\mathbf{w}^{*\tp}\phi_{\theta^*}(X)\right)^2 \right]} \\
  &\quad\le\norm{\wzero-\widetilde{\w}}_{\sigz}+\epsilon_0. \label{eq:ridgeboundrisk}
\end{align}
The term $\norm{\wzero-\widetilde{\w}}_{\sigz}$ is analyzed in~\cite{hsu2012random} by decomposing it into three parts, each corresponding to an error coming from a specific source. The first source is the \emph{regularization} error that stems from the introduction of $\lambda$. Let
  \begin{align}
    \w^\lambda &= \argmin_\w \left\{\E_X\left[\left(\w^\tp\pzero(X)-\wstartp\pstar(X)\right)^2\right]+{\lambda}\norm{\w}^2\right\} \label{eq:wlam_def} \\
               &= \left(\sigz+{\lambda} I\right)^{-1}\E_X\left[ \pzero(X)\pstar(X)^\tp \right]\wstar \label{eq:matrix_expr_wlambda} \\
               &= \sum_{i=1}^{d} \frac{\mathbf{v}_i^\tp\E_X\left[ \pzero(X)\pstar(X)^\tp \right]\wstar}{\lambda + \lambda_i} \mathbf{v}_i
  \end{align}
  be the minimizer of the true Ridge loss. If we introduce
\begin{equation}
  \alpha_i=\mathbf{v}_i^\tp\E_X\left[ \pzero(X)\pstar(X)^\tp \right]\wstar, \label{eq:alphadef}
\end{equation}
$\w^\lambda$ can be expressed as
\begin{equation}
  \w^\lambda=\sum_{i=1}^{d} \frac{\alpha_i}{\lambda+\lambda_i}\mathbf{v}_i. \label{eq:wlamexpr}
\end{equation}
We also introduce the conditional expectation of $\wzero$,
\begin{equation}
  \w^c=\E_X\left[ \wzero|X_1,\ldots,X_{T_1} \right].
\end{equation}
The error is now split into three terms:
  \begin{equation}
    \norm{\wzero- \widetilde{\w}}_{\sigz}^2\le 3\left(\underbrace{\Vert{\widetilde{\w}-\mathbf{w}^\lambda\Vert}^2_{\sigz}}_{=:\epsilon_{rg}}+\underbrace{\Vert{\w^\lambda-\w^c\Vert}^2_{\sigz}}_{=:\epsilon_{bs}}+\underbrace{\Vert{\w^c-\w_0\Vert}^2_{\sigz}}_{=:\epsilon_{vr}}\right), \label{eq:introerrors}
  \end{equation}
  where  $\epsilon_{rg}$ is the regularization error (coming from the introduction of ${\lambda}$), $\epsilon_{bs}$ is the bias error, introduced by the randomness of the design, and $\epsilon_{vr}$ is the variance error, coming from noise $\eta_t$. Alternatively to the above decomposition, we could bound the excess risk using~\eqref{eq:twoterms} as
\begin{align}
  \sqrt{\risk(\omega_0)-\risk(\omega^*)}&\le \norm{\wzero-\wstar}_{\sigz}+\epsilon_0 \\
                                             &\le \sqrt{3\Vert\wstar-\w^\lambda\Vert^2_{\sigz}+3\epsilon_{bs}+3\epsilon_{vr}} +\epsilon_0. \label{eq:altriskbound}
\end{align}
Both $\Vert\wstar-\w^\lambda\Vert^2_{\sigz}$ and $\Vert\widetilde{\w}-\w^\lambda\Vert^2_{\sigz}$ can be bounded  in terms of $\lambda$ and  $\epsilon_0$, as the following two lemmas show.
\begin{lemma}[Regularization error in presence of misspecification] \label{lem:ridgeregerr}
  \begin{equation}
    \epsilon_{rg}=\Vert\w^\lambda-\widetilde{\w}\Vert_{\sigz}^2 \le \frac{\lambda}{2}\norm{\wstar}^2+2\epsilon_0^2 \dht.
  \end{equation}
\end{lemma}
When misspecification is not present, i.e. $\theta_0=\theta^*$, we have $\wstar=\widetilde{\w}$, and we can similarly bound $\epsilon_{rg}\le \lambda\norm{\w^*}^2/4$, hence the name of the lemma.
\begin{proof}
  Let us for now assume that $\sigz$ is nonsingular. Then, $\lambda_i>0$ for all $i$, and the inverse of $\sigz$ is expressed as
\begin{equation}
  \sigz^{-1}=\sum_{i=1}^{d} \lambda_i^{-1}\mathbf{v}_i\mathbf{v}_i^\tp.
\end{equation}
The minimizer of~\eqref{eq:wtildef} is now unique and given by
\begin{align}
  \widetilde{\w} &=\sigz^{-1}\E_X\left[ \pzero(X)\pstar(X)^\tp \right]\wstar = \sum_{i=1}^{d} \frac{\alpha_i}{\lambda_i}\mathbf{v}_i. \label{eq:wtilexpr}
\end{align}
$\alpha_i$ can be bounded as follows. Using $(a+b)^2\le 2a^2+2b^2$,
\begin{align}
  \alpha_i^2 &\le 2 \left( \mathbf{v}_i^\tp\E_X\left[ \pzero(X)\pzero(X)^\tp \right]\wstar \right) ^2+2\left( \mathbf{v}_i^\tp\E_X\left[ \pzero(X)\Delta\phi(X)^\tp \right]\wstar\right)^2. \\
 \intertext{Remembering that $\sigz=\E_X\left[ \pzero(X)\pzero(X)^\tp \right]$ and that $\mathbf{v}_i$ is its eigenvector,}
  &=  2 \lambda_i^2 \left( \mathbf{v}_i^\tp \wstar \right) ^2+2\E_X\left[  \mathbf{v}_i^\tp\pzero(X)\Delta\phi(X)^\tp\wstar  \right] ^2 \\
  &\le 2 \lambda_i^2 \left( \mathbf{v}_i^\tp \wstar \right)^2+2\E_X\left[  \left(\mathbf{v}_i^\tp\pzero(X)\right)^2\right]\E_X\left[\left(\Delta\phi(X)^\tp\wstar\right)^2\right].
\end{align}
Using the definitions of $\mathbf{v}_i$ and $\epsilon_0$ once again,
\begin{align}
  \alpha_i^2&\le 2 \lambda_i^2 \left( \mathbf{v}_i^\tp \wstar \right)^2+2\epsilon_0^2\lambda_i. \label{eq:boundalpha}
\end{align}
With the help of decompositions~\eqref{eq:wlamexpr} and~\eqref{eq:wtilexpr}, we can express $\epsilon_{rg}$ as
\begin{align}
\epsilon_{rg}&=\sum_{i=1}^{d} \lambda_i\left( \frac{\alpha_i}{\lambda_i}-\frac{\alpha_i}{\lambda+\lambda_i} \right)^2   \\
             &= \sum_{i=1}^{d} \frac{\lambda^2 \alpha_i^2}{\lambda_i\left( \lambda+\lambda_i \right)^2 }\\
             &\le 2 \sum_{i=1}^{d} \frac{\lambda^2\lambda_i}{\left( \lambda+\lambda_i \right)^2 }\left( \mathbf{v}_i^\tp \wstar \right)^2+2\epsilon_0^2\sum_{i=1}^{d} \frac{\lambda^2}{\left(\lambda+\lambda_i\right)^2} \\
             &\le \frac{\lambda}{2}\norm{\wstar}^2+2\epsilon_0^2 \dht.
\end{align}
In the last inequality we used that for $a,b>0$, $ab/(a+b)^2\le 1 /4$. As long as $\lambda>0$,  $\dht$ has a finite limit when  $\lambda_i\to 0$. Taking this limit, we can see that the bound holds for singular  $\sigz$ as well. \[\]
\end{proof}

A similar bound could be given for the weighted distance $\Vert \w^\lambda-\wstar\Vert_{\sigz}^2$:
\begin{lemma}[Regularization error w.r.t. $\wstar$]
  \label{lem:ridgeregerr_wstar}
  \begin{equation}
    \Vert \w^\lambda-\wstar\Vert_{\sigz}^2 \le \frac{\lambda}{2}\norm{\wstar}^2+2\epsilon_0^2 \defft.
  \end{equation}
\end{lemma}
Compared to Lemma~\ref{lem:ridgeregerr}, the effective dimension $\dht$ is replaced by  $\defft$. As we shall see later, both bounds can be useful in different contexts.
\begin{proof}
  Using~\eqref{eq:wlamexpr} once again,
  \begin{align}
    \Vert \w^\lambda-\wstar\Vert_{\sigz}^2 &= \sum_{i=1}^{d} \lambda_i\left( \frac{\alpha_i}{\lambda + \lambda_i}-\mathbf{v}_i^\tp \wstar \right) ^2 \\
    &=  \sum_{i=1}^{d} \frac{\lambda_i}{(\lambda+\lambda_i)^2}\left( \alpha_i-(\lambda+\lambda_i)\mathbf{v}_i^\tp\wstar\right)^2 \\
    &\le 2\sum_{i=1}^{d} \frac{\lambda_i\lambda^2}{(\lambda+\lambda_i)^2}\left( \mathbf{v}_i^\tp\wstar \right) ^2 +2\sum_{i=1}^{d} \frac{\lambda_i}{(\lambda+\lambda_i)^2}(\alpha_i-\lambda_i \mathbf{v}_i^\tp\wstar)^2 \\
    &\le \frac{\lambda}{2}\norm{\wstar}^2+2\sum_{i=1}^{d} \frac{\lambda_i}{(\lambda+\lambda_i)^2}(\alpha_i-\lambda_i \mathbf{v}_i^\tp\wstar)^2. \label{eq:wstarbound_decomp}
  \end{align} 
  In the first inequality, we used $(a+b)^2\le 2a^2+2b^2$. In the second --- that $\lambda\lambda_i/(\lambda+\lambda_i)^2 \le 1/4$. We will now focus on the remaining sum. Using the definition~\eqref{eq:alphadef} of $\alpha_i$,
  \begin{align}
    (\alpha_i-\lambda_i \mathbf{v}_i^\tp\wstar)^2 &= \left( \mathbf{v}_i^\tp\big(\E_X\left[ \pzero(X)\pstar(X)^\tp \right] -\lambda_i I\big)\wstar \right)^2  \\
    &= \left( \mathbf{v}_i^\tp\big(\E_X\left[ \pzero(X)\pzero(X)^\tp \right] -\lambda_i I +\E_X\left[ \pzero(X)\Delta\phi(X)^\tp \right]\big)\wstar \right)^2 \\
    &= \left( \lambda_i\mathbf{v}_i^\tp\wstar-\lambda_i\mathbf{v}_i^\tp\wstar+\mathbf{v}_i^\tp\E_X\left[ \pzero(X)\Delta\phi(X)^\tp \right]\wstar \right)^2 \label{eq:usedeflam} \\
    &= \left(\mathbf{v}_i^\tp\E_X\left[ \pzero(X)\Delta\phi(X)^\tp \right]\wstar \right)^2 \\
    &\le \E_X\left[ \left( \mathbf{v}_i^\tp\pzero(X) \right) ^2 \right] \E_X\left[ \left( \wstartp\Delta\phi(X) \right)^2  \right] \\
    &= \epsilon_0^2\lambda_i\label{eq:usedeflam_again}.
  \end{align}
  In~\eqref{eq:usedeflam} and~\eqref{eq:usedeflam_again}, we used the definitions of $\lambda_i$ and  $\mathbf{v}_i$, and  in~\eqref{eq:usedeflam_again} we also used the definition of $\epsilon_0$. Substituting this into~\eqref{eq:wstarbound_decomp},
  \begin{equation}
    \Vert \w^\lambda-\wstar\Vert_{\sigz}^2\le \frac{\lambda}{2}\norm{\wstar}^2+2\epsilon_0^2\sum_{i=1}^{d} \left(\frac{\lambda_i}{\lambda+\lambda_i}\right)^2 = \frac{\lambda}{2}\norm{\wstar}^2+2\epsilon_0^2\defft.
  \end{equation}
\end{proof}

The bound on $\norm{\wzero-\widetilde{\w}}_{\sigz}^2$ is given in~\cite[Theorem 16]{hsu2012random}. We adapted this theorem to our notation as Theorem~\ref{thm:ridge_base} in the appendix. Our version introduces notation 
\begin{equation}
  \text{approx}_\lambda(X)=\left(\wstartp\pstar(X)-\w^{\lambda\tp}\pzero(X)\right).
\end{equation}
The theorem relies on several conditions, which we will now verify in our setup. The first condition is that
\begin{equation}
  \frac{\norm{\left(\sigl\right)^{-1 /2} \phi_{\theta_0}(X)}}{\sqrt{d_{1,\lambda}} }\le \rho_\lambda
\end{equation}
holds almost surely for some $\rho_\lambda$. Since $\norm{\pzero(X)}\le B_\phi$, it is satisfied with
\begin{equation}
  \rho_\lambda = B_\phi / \sqrt{d_{1,\lambda}\lambda}. \label{eq:good_rho}
\end{equation}
Another condition concerns the boundedness of misspecification. The following lemma states this condition and provides a constant $b_\lambda$ with which it is satisfied. \begin{lemma}[Bounded misspecification] \label{lem:ridgeboundmisspec}
  \begin{equation}
    \delta_{\text{ms}}=\frac{\norm{\sigl^{-1 /2}\pzero(X)\left(\wstartp\pstar(X)-\w^{\lambda\tp}\pzero(X)\right) }}{\sqrt{d_{1,\lambda}} }\le b_\lambda
  \end{equation}
  almost surely, where
  \begin{equation}
    b_\lambda= \frac{1}{\sqrt{d_{1,\lambda}\lambda }}\left(1+2\sqrt{2}  \right)B_\w  B_\phi^2+\frac{\epsilon_0}{\lambda}(2+\sqrt{2})B_\phi^2.
  \end{equation}
\end{lemma}
In the final bound on the excess risk, $b_\lambda$ enters as a $O(b_\lambda^2d_{1,\lambda}/n^2)$ term. This lemma allows us to write this as
\begin{equation}
  O\left(\frac{1}{\lambda n^2}+\frac{\epsilon_0^2 d_{1,\lambda}}{\lambda^2 n^2} \right). \label{eq:asympblambda}
\end{equation}
The leading term coincides with the case of bounded $\norm{\widetilde{\w}}$, considered in~\cite[Remark 4]{hsu2012random}. 
\begin{proof}
  Let $\widehat{\w}\in\R^d$ to be specified later. We have that
  \begin{align}
    \delta_{ms} &\le \rho_\lambda\left| \wstartp\pstar(X)-\w^{\lambda\tp}\pzero(X) \right|  \\
                &\le \rho_\lambda\left( \left| \wstartp\pstar(X)-\widehat{\w}^{\tp}\pzero(X) \vphantom{\w^{\lambda\tp}}\right|+\left| (\widehat{\w}-\w^\lambda)^\tp\pzero(X)\right| \right)\\
                &\le \rho_\lambda\left( B_\w B_\phi+B_\phi \norm{\widehat{\w}}+\Vert{\widehat{\w}-\w^\lambda}\Vert_{\sigl}\norm{\pzero(X)}_{\left(\sigl\right)^{-1}} \right) \\
                &\le \rho_\lambda\left( B_\w B_\phi+B_\phi \norm{\widehat{\w}}+\rho_\lambda \sqrt{d_{1,\lambda}} \Vert{\widehat{\w}-\w^\lambda}\Vert_{\sigl}\right). \label{eq:preliminary_deltams}
  \end{align}
  A similar derivation is presented in~\cite[Remark 4]{hsu2012random} with $\widehat{\w}=\widetilde{\w}$. The authors then proceed to bound
  \begin{equation}
    \Vert{\widetilde{\w}-\w^\lambda}\Vert_{\sigl}\le \sqrt{\lambda}\norm{\widetilde{\w}}
  \end{equation}
  and thus get an expression for $b_\lambda$ in terms of  $\norm{\widetilde{\w}}$. Whenever $\norm{\widetilde{\w}}$ is bounded, this is a valid technique. However, an examination of~\eqref{eq:wtilexpr} reveals that the components of $\widetilde{\w}$ in the eigenbasis of $\sigz$ depend on $\lambda_i^{-1}$, so without additional assumptions on the spectrum of  $\sigz$, we cannot bound $\norm{\widetilde{\w}}$. Yet the largest components come from directions with small $\lambda_i$. The components of $\widetilde{\w}$ along these directions also contribute little to the value of $\Vert{\widetilde{\w}-\w^\lambda}\Vert_{\sigl}$. We can exploit this informal observation in the following manner.

Let us divide the set of indices $[ d ] $ into two subsets:
\begin{equation}
  I_0=\left\{ i\in[d]:\ \lambda_i \le \lambda  \right\}\quad\text{and}\quad I_1=\left\{ i\in[d]:\ \lambda_i > \lambda  \right\}.
\end{equation}
One of them can be empty without affecting the proof. In this case, a sum over an empty set is assumed to be $0$. We set $\widehat{\w}$ to
\begin{equation}
  \widehat{\w}=\sum_{i\in I_1} \frac{\alpha_i}{\lambda_i}\mathbf{v}_i.
\end{equation}
That is, we take $\widetilde{\w}$, but set some of its components to zero. The size of $I_1$ can be bounded by the effective dimension:
\begin{equation}
    |I_1|=\sum_{i\in I_1}\frac{\lambda +\lambda_i}{\lambda +\lambda_i} < 2\sum_{i\in I_1}\frac{\lambda_i}{\lambda +\lambda_i} \le 2 d_{1,\lambda}.
\end{equation}
Now, the norm of $\widehat{\w}$  is bounded by
\begin{align}
  \norm{\widehat{\w}}^2&= \sum_{i\in I_1} \frac{\alpha_i^2}{\lambda_i^2} \le  2\sum_{i\in I_1}\left(\left( \wstartp\mathbf{v}_i \right)^2 +\frac{\epsilon_0^2}{\lambda_i}\right) < 2\norm{\wstar}^2+4\frac{\epsilon_0^2d_{1,\lambda}}{\lambda}. \label{eq:barwnorm}
\end{align}
Above, we used~\eqref{eq:boundalpha} and the definition of $I_1$. It remains to deal with $\Vert{\widetilde{\w}-\w^\lambda}\Vert_{\sigl}$. Taking into account the definition of $\widehat{\w}$ and~\eqref{eq:wlamexpr},
\begin{align}
  \Vert{\widehat{\w}-\w^\lambda}\Vert_{\sigl}^2 &=\sum_{i\in I_0} \frac{\alpha_i^2}{\lambda+\lambda_i}+\sum_{i\in I_1}\left( \lambda+\lambda_i \right) \left( \frac{\alpha_i}{\lambda_i}-\frac{\alpha_i}{\lambda+\lambda_i} \right) ^2 \\
                                                                 &= \sum_{i\in I_0} \frac{\alpha_i^2}{\lambda+\lambda_i}+\sum_{i \in I_1}\frac{\alpha_i^2 \lambda^2}{\lambda_i^2\left( \lambda+\lambda_i \right) }.
\end{align}
Using the bound~\eqref{eq:boundalpha} on $\alpha_i^2$,
\begin{align}
  \Vert{\widehat{\w}-\w^\lambda}\Vert_{\sigl}^2 &\le 2\sum_{i\in I_0}\frac{\lambda_i^2}{\lambda+\lambda_i} \left(\wstartp \mathbf{v}_i\right)^2+2\epsilon_0^2\sum_{i \in I_0}\frac{\lambda_i}{\lambda+\lambda_i} \\
                                                                 &\qquad+2\sum_{i \in I_1}\frac{\lambda^2}{\lambda+\lambda_i}\left(\wstartp \mathbf{v}_i\right)^2+2\epsilon_0^2\sum_{i\in I_1}\frac{\lambda^2}{\lambda_i(\lambda+\lambda_i)}.
  \intertext{By definitions of $I_0$ and  $I_1$, we can bound $\lambda_i^2 \le \lambda^2$ in the first sum and $\lambda^2<\lambda_i^2$ in the fourth. Merging the first sum with the third sum and the second with the fourth,}
                                                                 &< 2 \lambda\norm{\wstar}^2+ 2\epsilon_0^2 d_{1,\lambda}. \label{eq:barwdiffnorm}
\end{align}
Finally, we can substitute~\eqref{eq:barwnorm} and~\eqref{eq:barwdiffnorm} into~\eqref{eq:preliminary_deltams}. When taking the square root, we use $\sqrt{a+b}\le \sqrt{a} +\sqrt{b}$.
\begin{align}
  \delta_{ms} &\le \rho_\lambda\left(B_\w B_\phi +B_\phi\left( \sqrt{2}B_\w +2\epsilon_0\sqrt{\frac{d_{1,\lambda}}{\lambda}}   \right)+\rho_\lambda \sqrt{d_{1,\lambda}} \left(\sqrt{2\lambda}B_\w +\epsilon_0\sqrt{2d_{1,\lambda}} \right)\right)  \\
              &\le  \rho_\lambda\left( (1+\sqrt{2})B_\w  B_\phi+\rho_\lambda\sqrt{2\lambda d_{1,\lambda}}B_\w \right)+\epsilon_0\rho_\lambda\sqrt{d_{1,\lambda}}\left(2\frac{B_\phi}{\sqrt{\lambda} } + \rho_\lambda \sqrt{2d_{1,\lambda}} \right) \\
              &=\frac{1}{\sqrt{d_{1,\lambda}\lambda }}\left(1+2\sqrt{2}  \right)B_\w  B_\phi^2+\frac{\epsilon_0}{\lambda}(2+\sqrt{2})B_\phi^2.
\end{align}
In the final equality, we substituted $\rho_\lambda=B_\phi /\sqrt{\lambda d_{1,\lambda}} $. \[\]

\end{proof}

We are now ready to use~Theorem~\ref{thm:ridge_base} and provide regret bounds for the explore-then-commit algorithm with misspecified features. As mentioned before, this theorem provides all bounds in terms of the effective dimensions $d_{p,\lambda}$. We can bound the effective dimensions using~\eqref{eq:boundeffdim} or simply as $d_{p,\lambda}\le d$. It turns out that these two techniques lead to different regret bounds (the first with $\lambda=T_1^{-1 /2}$, and the second with $\lambda=O(\log T_1 /T_1)$). The first bound has a dimension-independent sublinear term! The bounds are given in the two theorems that follow. The proofs for both theorems are deferred to Appendix~\ref{ap:boringproofs}, since they are not very insightful. Both proofs mostly consist of verifying the conditions for Theorem~\ref{thm:ridge_base} and substituting the specific hyperparameters into the bounds this theorem provides.

\begin{theorem}[E2TC with $T_2=0$, data-poor regime] \label{thm:nodimregret} $\ $ \\
  Let E2TC be run with ${T_2=0}$ (only the first stage of exploration). Let Assumptions~\ref{as:context},~\ref{as:realizability},~\ref{as:bound} hold. There exist $N_0\in\mathbb{N}$ and a universal constant $C>0$ such that, whenever $T_1>N_0$, with $\lambda=T_1^{-1 /2}$, the following statements hold simultaneously:
  \begin{enumerate}
    \item With $\PP\ge 1-C / \sqrt{T_1}$,
      \begin{equation}
        \norm{\wzero-\wstar}_{\sigz}^2=O\left(\epsilon_0^2d+ \frac{1}{\sqrt{T_1} } \right). \label{eq:wclose_sigz_datapoor}
      \end{equation}
    \item Regret of E2TC the following asymptotics:
    \begin{equation}
      R_{1:T}=O\left(TT_1^{-1 /2}+T_1 +\epsilon_0 \sqrt{\min\{\dht,\defft\}} KT+T_1^{-1 /4}KT\right). \label{eq:dimfreeregret_general}
    \end{equation}
  \end{enumerate}
\end{theorem}
When $T\ge K^4$, we can select $T_1=(KT)^{4 /5}$ to get
\begin{equation}
  R_{1:T}=O\left(\epsilon_0 \sqrt{\min\{\dht,\defft\}} KT+(KT)^{4 /5}\right). \label{eq:dimfreeregret}
\end{equation}
Otherwise, with $T_1=KT^{4 /5}$,
\begin{equation}
    R_{1:T}=O\left(\epsilon_0 \sqrt{\min\{\dht,\defft\}} KT+KT^{4 /5}\right).
\end{equation}
Since our analysis targets applications with small $K$ (e.g. a form of few-class classification), the first regret bound is the main one.

The constant $N_0$ is rather large: it depends on  $B_\phi$ approximately as  $B_\phi^4$ due to the restrictions of the bound from ~\cite{hsu2012random}. To resolve this issue, we can scale $\lambda$ up by a factor of $B_\phi^2$. This will make $N_0$ independent of $B_\phi$, but the factor will also show up in the regret bound. While it doesn't change the asymptotics, in practice this approach makes the guarantees we get much weaker. 

We left two notions of the effective dimension ($\dht$ and  $\defft$) in the regret bound on purpose.  $\defft$ allows us to discard the dimension of the nullspace of $\sigz$. In networks with ReLU activations, this can amount to a large fraction of $d$. On the other hand, using $\dht$ can lead to sublinear regret in strongly convex losses, as we show at the end of this section.

The next theorem shows that if we can afford to introduce a factor of $d$ in the sublinear term of the regret, then the $T$ dependency of this term can be cut to $T^{2 /3}$.

\begin{theorem}[E2TC with $T_2=0$, data-rich regime] \label{thm:dimregret}
  Let E2TC be run with $T_2=0$ (only the first stage of exploration). Let Assumptions~\ref{as:context},~\ref{as:realizability},~\ref{as:bound} hold. Let $a>0$. There exist $N_0(a)\in\N$ such that, whenever $T_1>\max\{N_0, d^{1 /a}\}$, with
\begin{equation}
  \lambda:=\frac{7 B_\phi^2 \log T_1}{T_1}, \label{eq:lambdatwo}
\end{equation}
the following statements hold simultaneously:
\begin{enumerate}
  \item With $\PP>1-4 /\sqrt{T_1} $,
    \begin{equation}
      \norm{\w_0-\wstar}_{\sigz}^2=\widetilde{O}\left( \epsilon_0^2d+\frac{d}{T_1} \right).
    \end{equation}
  \item Regret of E2TC has the following asymptotics:
  \begin{multline}
    \E\left[ R_{1:T} \right]= O\left(\epsilon_0 \sqrt{d} KT+\frac{T}{\sqrt{T_1}} +T_1
    \vphantom{+KT\left( \sqrt{\frac{d}{T_1}} +\frac{d^{1 /4}\log ^{1 /4}T_1}{\sqrt{T_1} }+\sqrt{\frac{\log T_1}{T_1}}\right)}
    \right.\\
    \left.+KT\left( \sqrt{\frac{d}{T_1}} +\frac{d^{1 /4}\log ^{1 /4}T_1}{\sqrt{T_1} }+\sqrt{\frac{\log T_1}{T_1}}  \right) \right).
  \end{multline}
\end{enumerate}
\end{theorem}
Depending on the relative size of $T$, $d$, and $K$, we should use different choices for $T_1$. In the best case scenario, we can choose $T_1=d^{1 /3}(KT)^{2 /3}$ and get
\begin{equation}
  \E\left[R_{1:T}  \right] =O(\epsilon_0 \sqrt{d} KT)+\widetilde{O}(d^{1/3}(KT)^{2 /3}).
\end{equation}
For $T< d^{5 /2}$ and $K=O(1)$, our previous theorem provides better regret bounds, hence ``data-poor'' and ``data-rich'' regimes in the names of the theorems. $N_0$ depends on $a$ only linearly, so the lower bound on $T_1$ here is not restrictive.

\paragraph{Strong convexity makes weak learning achieve sublinear regret} Let us focus on Theorem~\ref{thm:nodimregret}. It is easy to see that the Hessian of $\risk$ w.r.t.  $\w$ at $\theta_0$ is constant and given by
\begin{equation}
  \frac{\partial^2}{\partial\w^2} \risk(\w,\theta_0)=2\sigz.
\end{equation}
Let us assume that the loss is strongly convex at $(\wstar,\theta_0)$ w.r.t. $\w$, and that the strong convexity parameter is bounded from below by $\mu$. In this case, the lowest eigenvalue of the Hessian is bounded by the same value:
\begin{equation}
  \lambda_{min}(2\sigz)\ge \mu.
\end{equation}
We can use this to bound $\dht$, which, we remind the reader, is defined in~\eqref{eq:extra_effdim_def}:
\begin{equation}
  \dht=\sum_{i=1}^{d} \frac{\lambda}{\lambda+\lambda_i}\le \frac{2d\lambda}{\mu}.
\end{equation}
Taking into account that Theorem~\ref{thm:nodimregret} sets $\lambda=O(T_1^{-1 /2})$, the regret bound~\eqref{eq:dimfreeregret_general} can be rewritten as
\begin{equation}
  R_{1:T}=O\left(\epsilon_0\sqrt{d /\mu}T_1^{-1 /4} KT+T_1^{-1 /4}KT\right).
\end{equation}
Selecting $T_1=(KT)^{4 /5}$ again, we arrive at
\begin{equation}
  R_{1:T}=O\left((\epsilon_0\sqrt{d /\mu}+1)(KT)^{4 /5}\right).
\end{equation}
The regret is sublinear in $T$, even though we only learn the last layer of the network! We believe that this counterintuitive result points that \emph{strong convexity should not be considered as a framework for neural network learning}. Even if the risk is strongly convex, it is unrealistic to expect that a lower bound on $\mu$ will be known.


%% file: chapters/sgd.tex
\chapter{High-probability guarantees for stochastic gradient descent}
In this chapter we study high-probability bounds on the suboptimality gap that can be achieved with Stochastic Gradient Descent in the case of convex risk. The expected value of the gap can be bounded using standard techniques, described e.g. in~\cite[Chapter 14]{shalev2014understanding}. In Section~\ref{sec:global_sgd}, we show how this can be generalized to a high-probability bound using the Hoeffding-Azuma inequality. Although we believe this result has been demonstrated before, we could not find a proper reference for it, so we decided to derive it ourselves. In Section~\ref{sec:locally_convex_loss}, we examine the more nuanced case of a risk function that is only locally convex. In this situation, we provide a restriction on the learing rate under which we can guarantee that SGD does not escape  the convex basin, given that it is initialized in it. Finally, in Section~\ref{sec:preconditioned_sgd}, we introduce the specific flavor of SGD that E2TC uses, namely preconditioned SGD. We discuss the intuition behind it and how it leads to different guarantees compared to vanilla SGD.

\label{ch:sgd}
\section{Globally convex loss}
\label{sec:global_sgd}
Let $\dom \subset\R^P$ be a closed convex set. Projection $\Pi_\dom$ onto it is defined as
\begin{equation}
  \forall \omega\in\R^P,\quad \Pi_\dom(\omega)=\argmin_{\widetilde{\omega} \in\dom}\norm{\widetilde{\omega}-\omega}^2.
\end{equation}
Let $\risk: \dom\to \R$ be defined over this set as $\risk(\omega)=\E_{\mathbf{z}\sim\mathcal{D}}\left[ \ell(\omega,\mathbf{z}) \right] $. This is a generalization of the mean squared risk~\eqref{eq:risk_decomposed} that we use in E2TC. Here, $\mathcal{D}$ is the data distribution over some set $\mathcal{Z}$, and $\ell$ is the loss function that depends on the model parameters $\omega\in \dom$ and a random data sample $\mathbf{z}$. We consider the stochastic gradient descent (SGD) to minimize $\risk$ as described in~\cite[Chapter 14]{shalev2014understanding}. From an initialization $\omega_1\in\dom$ it proceeds as $\omega_{t+1}:=\Pi_\mathcal{B}(\omega_t-\zeta \mathbf{v}_t)$, where $\mathbf{v}_t=\nabla_\omega \ell(\omega_t,\mathbf{z}_t)$ is an unbiased estimator of $\nabla \risk(\omega_{t})$, $\zeta>0$ is the learning rate, and $\{\mathbf{z}_t\}$ are iid samples from $\mathcal{D}$. After $T$ updates, the algorithm outputs 
\begin{equation}
 \overline{\omega}:=\frac{1}{T}\sum_{t=1}^{T}\omega_t.
\end{equation}
Let $\omega^*\in\argmin_{\omega \in\dom} \risk(\omega)$. \cite[Theorem 14.8]{shalev2014understanding} bounds $\E\left[ \risk(\overline{\omega})\right]-\risk(\omega^*)$ under certain conditions. Here, we will adapt the proof under the same conditions, but instead obtain a high-probability bound on $\risk(\overline{\omega})-\risk(\omega^*)$.

\begin{theorem}
  \label{thm:highp_sgd}
  Assume that $\risk$ is convex, $\norm{\mathbf{v}_t}\le D $ a.s., and the domain $\dom$ is a $B_\omega$-ball:
\begin{equation}
  \dom=\left\{ \omega:\ \norm{\omega}\le B_\omega \right\}.
\end{equation}
Then, for any $\delta>0$, with probability not less than $1-\delta$,
\begin{equation}
  \risk(\overline{\omega})-\risk(\omega^*)< \frac{\norm{\omega^*-\omega_1}^2}{2\zeta T}+\frac{\zeta}{2T}\sum_{t=1}^{T} \norm{\gv_t}^2+
  4 D  B_\omega \sqrt{\frac{2}{T}\log \frac{1}{\delta}}. \label{eq:subopt_bound_sgd_default}
\end{equation}
\end{theorem}
Compared to the bound on $\E\left[ \risk(\overline{\omega})\right]-\risk(\omega^*)$ from~\cite{shalev2014understanding}, this one has an extra term $4DB_\omega\sqrt{2(\log(1 /\delta)) /T} $. The rest of the bound is the same. If we set $\zeta=2B_\omega /( D  \sqrt{T} )$ and bound $\norm{\gv}\le D,\ \norm{\omega^*-\omega_1}\le 2B_\omega$, this theorem implies that with $\PP>1-\delta$,
\begin{equation}
  \risk(\overline{\omega})-\risk(\omega^*)<  \frac{2 D  B_\omega}{\sqrt{T} }\left(1+2\sqrt{2\log \frac{1}{\delta}}\right). \label{eq:suboptimality_bound_convex_f}
\end{equation}

We will use a lemma from the original proof:
\begin{lemma}[Lemma 14.1 in~\cite{shalev2014understanding}]
  \label{lem:steps}
  Let $\mathbf{v}_1,\cdots,\mathbf{v}_{T}\in\R^P$ be a sequence of vectors and $\mathcal{B}\subset\R^P$ be a convex set. We define a sequence $\omega_t$ recursively. Let $\omega_1 \in\R^P$ be arbitrary and 
\begin{equation}
  \omega_{t+1}=\Pi_\dom(\omega_{t}-\zeta\mathbf{v}_t) \label{eq:default_sgd_step}
\end{equation} 
for $t>1$. Then, for any $\omega^*\in\dom$,
\begin{equation}
  \sum_{t=1}^{T}(\omega_t-\omega^*)^\tp\gv_t\le \frac{\norm{\omega^*-\omega_1}^2}{2\zeta}+\frac{\zeta}{2}\sum_{t=1}^{T} \norm{\gv_t}^2.
\end{equation}
\end{lemma}
The original lemma omits the projection $\Pi_\dom$, but subsequent discussion in~\cite[Section 14.4.1]{shalev2014understanding} shows that the lemma holds with projection as well. The lemma was modified to accommodate nonzero $\omega_1$ without affecting the proof. If we further assume that $\norm{\mathbf{v}_t}\le  D $ for all $t$, then, with the learning rate $\zeta=2B_\omega / ( D \sqrt{T})$, we get 
  \begin{equation}
    \frac{1}{T}\sum_{t=1}^{T}(\omega_t-\omega^*)^\tp\gv_t\le\frac{2B_\omega D }{\sqrt{T}}.
  \end{equation}

\begin{proof}[Proof of Theorem~\ref{thm:highp_sgd}.]
  Let $\{\mathcal{F}_t\}_t$ be a filtration defined as
  \begin{equation}
    \mathcal{F}_t=\sigma(\mathbf{z}_1,\ldots,\mathbf{z}_t),
  \end{equation}
  where $\mathbf{z}_t\in\mathcal{Z}$ are the iid random data samples used to compute stochastic gradients $\mathbf{v}_t=\nabla_\omega l(\omega_t,\mathbf{z}_t)$. With this definition, the sequence $\{\omega_t\}_t$ is  $\mathcal{F}$-predictable and the sequence $\{\mathbf{v}_t\}_t$ is $\mathcal{F}$-adapted. We now define the $\mathcal{F}$-adapted sequence $\{a_t\}_t$ as
\begin{equation}
  a_t=\left(\risk(\omega_t)-\risk(\omega^*)\right)-(\omega_t-\omega^*)^\tp\gv_t.
\end{equation}
Since $\risk$ is convex,
\begin{equation}
  \E\left[ a_t | \mathcal{F}_{t-1}\right] =
  \left(\risk(\omega_t)-\risk(\omega^*)\right)-(\omega_t-\omega^*)^\tp\nabla f(\omega_t) \le  0.
\end{equation}
This means that $a_t$ form a supermartingale difference sequence. Furthermore, bounded domain of $\risk$ and $\norm{\gv_t}\le  D $ (from which follows Lipschitzness of $\risk$ with parameter $ D $) allow us to write
\begin{equation}
  |a_t|\le 4 D  B_\omega\quad\text{a.s.}
\end{equation}
Therefore, Hoeffding-Azuma inequality (Lemma~\ref{lem:azuma}) is applicable:
\begin{equation}
  \PP\left[ \sum_{t=1}^{T}\left(\risk(\omega_t)-\risk(\omega^*)\right)< \sum_{t=1}^{T}(\omega_t-\omega^*)^\tp\gv_t+4 D  B_\omega\sqrt{2T\log \frac{1}{\delta}}  \right] \ge  1-\delta.
  \label{eq:appl_azuma}
\end{equation}
Assume that the event in~\eqref{eq:appl_azuma} holds. Then, we can employ Lemma~\ref{lem:steps} to finish the proof:
\begin{align}
  \risk(\overline{\omega})-\risk(\omega^*) & \le \frac{1}{T}\sum_{t=1}^{T}(\risk(\omega_t)-\risk(\omega^*))\\
                                   & < \frac{1}{T}\sum_{t=1}^{T}(\omega_t-\omega^*)^\tp\mathbf{v}_t+4 D  B_\omega \sqrt{\frac{2}{T}\log \frac{1}{\delta}} \\
                                   & \le \frac{\norm{\omega^*-\omega_1}^2}{2\zeta T}+\frac{\zeta}{2T}\sum_{t=1}^{T} \norm{\gv_t}^2+4 D  B_\omega \sqrt{\frac{2}{T}\log \frac{1}{\delta}}
\end{align}
The first step here relies on the convexity of $\risk$. The third --- on Lemma~\ref{lem:steps}.  \[\]
\end{proof}

\section{Locally convex loss}
\label{sec:locally_convex_loss}
The analysis above requires $\risk$ to be convex on its entire domain. Often, this is an unrealistic assumption. We can relax it by requiring $\risk$ to only be convex \emph{locally}, in a region around $\omega^*$. We will call this region the \emph{convex basin} around $\omega^*$. However, if we want SGD to converge to  $\omega^*$, we still have to require that it is initialized within the convex basin. Now we have an additional worry: we need to avoid the trajectory escaping the basin at any point. Without knowing where in the basin did we land, we cannot simply enforce the trajectory to stay within it. Instead, we can provide a high-probability statement that the trajectory will stay in the basin without any enforcement.

\begin{theorem}[Containing SGD within a convex basin] \label{thm:sgd_containment}
  Let $\risk:\dom\to\R$ be defined over a ball 
\begin{equation}
  \dom=\left\{ \omega:\ \norm{\omega}\le B_\omega \right\} 
\end{equation}
  as before. Starting at a point $\omega_1\in\dom$, we run SGD for $T$ steps with constant stepsize $\zeta$. Assume that $\norm{\mathbf{v}_t}\le  D $ a.s. Let $\omega^*\in\argmin_{\omega\in\dom}\risk(\omega)$ and assume $\norm{\omega_1-\omega^*}< \epsilon$. Further, assume that $\risk$ is convex within the region 
  \begin{equation}
    \mathcal{C}=\left\{ \omega: \norm{\omega^*-\omega}<\epsilon_c \right\} \subset\dom
  \end{equation}
  for some  $\epsilon_c>\epsilon$. Finally, assume that for some $\delta>0$,
\begin{multline}
  \!\!\!\!\!\zeta^2 D ^2 T+5(\zeta^2 D ^2+4\zeta D  B_\omega)\sqrt{T\left(\log\log (eT(2\zeta^2 D ^2+8\zeta D  B_\omega)^2)\right)_++T\log(2 /\delta) } \\
  < \epsilon_c^2-\epsilon^2.\label{eq:containment_req}
\end{multline}
  Then, with probability $\PP\ge 1-2\delta$, the entire trajectory of SGD will be contained within $\mathcal{C}$ ($\forall t\ \omega_t\in\mathcal{C}$) and the suboptimality gap will be bounded as
  \begin{equation}
    \risk(\overline{\omega})-\risk(\omega^*)< \frac{\epsilon^2}{2\zeta T}+\frac{\zeta  D ^2}{2}+4 D  B_\omega \sqrt{\frac{2}{T}\log \frac{1}{\delta}}. \label{eq:contained_sgd_suboptimality_bound}
  \end{equation}
\end{theorem}
Note that if $\epsilon_c>2\epsilon$, a simple solution would be to project SGD onto the known region
\begin{equation}
  \widetilde{\mathcal{C}}=\left\{ \omega: \norm{\omega_0-\omega}\le \epsilon \right\},
\end{equation}
since in this case $\omega^*\in\widetilde{\mathcal{C}}$ and $\risk$ is convex on $\widetilde{\mathcal{C}}$. However, this stops working when $\epsilon_c<2\epsilon$. What is more, this approach requires accurate knowledge of $\epsilon$, while the lemma above simply asks for enough data as $\epsilon_c^2-\epsilon^2$ becomes smaller. The precise relationship between the epsilons, $\zeta$,  $T$, and $\delta$ is captured by~\eqref{eq:containment_req}. Asymptotically, as $\delta\to 0,\ \zeta\to 0,\ T\to \infty$, it asks for
\begin{equation}
  O\left(\zeta^2 T +\zeta\sqrt{T\left(\log \log (\zeta^2 T)\right)_+ + T\log(1 /\delta)} \right) < \epsilon_c^2-\epsilon^2.
\end{equation}
One possible choice of $\zeta$ that satisfies this requirement for large $T$ is
\begin{equation}
  \zeta=O\left( \frac{\epsilon_c^2-\epsilon^2}{\sqrt{T \log (1 /\delta)}}\right).\label{eq:choice_zeta}
\end{equation}
Thus, we adjust the asymptotics of $\zeta$ by a factor of $(\epsilon_c^2-\epsilon^2) /\sqrt{\log (1 /\delta)}$ compared to the case where $\risk$ is convex on the entire domain. Despite the different $\zeta$, for a fixed $\epsilon_c^2-\epsilon^2$ asymptotically we get the same high-probability bound as~\eqref{eq:suboptimality_bound_convex_f}. Epsilons enter the high-probability guarantee as
\begin{equation}
  \risk(\overline{\omega})-\risk(\omega^*)<O\left(\sqrt{\frac{\log(1 /\delta)}{T}}\Bigg(1+\frac{\epsilon^2}{\epsilon_c^2-\epsilon^2}+\frac{\epsilon_c^2-\epsilon^2}{\log(1 /\delta)}\Bigg)\right). \label{eq:new_contained_sgd_suboptimality_bound}
\end{equation}
However, selecting  $\zeta$ according to~\eqref{eq:choice_zeta} might require a large $T$ to satisfy~\eqref{eq:containment_req}. It also requires a known lower bound on $\epsilon_c^2-\epsilon^2$. If we are willing to sacrifice the asymptotics to make Theorem~\ref{thm:sgd_containment} work in the small $T$ and unknown $\epsilon_c^2-\epsilon^2$ regime, we can make $\zeta$ decay faster in $T$.

Internally, the theorem depends on a high-probability uniform bound for a maringale with bounded differences. We rely on the time-uniform Hoeffding-Azuma inequality, in the form it was presented in~\cite{kassraie2023anytime}. This form was itself adapted from the recent work~\cite{howard2020time}, which presents a new method to derive so-called \emph{curved} uniform bounds for martingales.

\begin{proof}[Theorem~\ref{thm:sgd_containment}]

  Assume that during the first $t-1$ steps, the projection operator $\Pi_\dom$ was never involved, i.e. that
\begin{equation}
  \forall k\le t-1\quad \norm{\omega_k-\zeta \mathbf{v}_k}<B_\omega.
\end{equation}
Let $\widetilde{\omega}_{t+1}=\omega_{t}-\zeta \mathbf{v}_t$. Note that we have $\omega_{t+1}=\Pi_\dom \widetilde{\omega}_{t+1}$. Then,
\begin{align}
  \norm{\widetilde{\omega}_{t+1}-\omega^*}^2 &= \norm{\omega_t-\omega^*}^2+\zeta^2\norm{\mathbf{v}_t}^2+2\zeta (\omega^*-\omega_t)^\tp \mathbf{v}_t.
\end{align}
Since by assumption for $k<t$ we have  $\omega_{k+1}=\omega_k-\zeta \mathbf{v}_k$, we can apply the above equation recursively and get
\begin{equation}
  \norm{\widetilde{\omega}_{t+1}-\omega^*}^2=\norm{\omega_1-\omega^*}^2+\zeta^2 \sum_{k=1}^{t} \norm{\mathbf{v}_k}^2+2\zeta \sum_{k=1}^{t} (\omega^*-\omega_k)^\tp \mathbf{v}_k. \label{eq:rewrite_normdiff}
\end{equation}
Now the way forward becomes more clear. We know that $\norm{\omega_1-\omega^*}^2<\epsilon^2$, and we need to show that the two sums at all times stay below  $\epsilon_c^2-\epsilon^2$. Let us introduce
\begin{equation}
  a_k=\zeta^2\norm{\mathbf{v}_k}^2+2\zeta(\omega^*-\omega_k)^\tp\mathbf{v}_k-\E\left[\zeta^2\norm{\mathbf{v}_k}^2+2\zeta(\omega^*-\omega_k)^\tp\mathbf{v}_k\mid\mathcal{F}_{k-1}  \right] .
\end{equation}
It is an $\mathcal{F}$-adapted martingale difference sequence. We also know that
\begin{equation}
  |a_k|\le 2\zeta^2  D ^2+8\zeta  D  B_\omega=:B_a\quad\text{a.s.} \label{eq:ak_bound}
\end{equation}
Let $S_t=\sum_{k=1}^{t}a_k$. By Lemma~\ref{lem:uniform_azuma},
\begin{equation}
  \PP\left[ \forall t\quad S_{t}<\frac{5B_a}{2}\sqrt{t\left((\log\log etB_a^2)_++\log(2 /\delta)\right) }  \right] \ge 1-\delta.
\end{equation}
We will denote the event in the above as $E$. We will also introduce a shorthand
\begin{equation}
  B_\delta(t)=\frac{5B_a}{2}\sqrt{t\left((\log\log etB_a^2)_++\log(2 /\delta)\right) }.
\end{equation}
Whenever $E$ holds, we can write for all $t$ 
\begin{align}
    &\zeta^2 \sum_{k=1}^{t} \norm{\mathbf{v}_k}^2+2\zeta \sum_{k=1}^{t} (\omega^*-\omega_k)^\tp \mathbf{v}_k \nonumber\\
    &\qquad\qquad\qquad< \sum_{k=1}^{t}\E\left[\zeta^2\norm{\mathbf{v}_k}^2+2\zeta(\omega^*-\omega_k)^\tp\mathbf{v}_k\mid\mathcal{F}_{k-1}  \right] +B_\delta(t)  \\
  &\qquad\qquad\qquad\le \zeta^2 D ^2t +2\zeta\sum_{k=1}^{t}(\omega^*-\omega_k)^\tp \nabla \risk(\omega_k)+B_\delta(t). \label{eq:bound_two_sums}
\end{align}
Let $F_t$ be the event defined as
\begin{equation}
  F_t=\left\{ \forall k\le t\quad \omega_k\in\mathcal{C} \right\}\ \cap\ \left\{ \forall k< t\quad \omega_k-\zeta\mathbf{v}_k\in\mathcal{C} \right\} .
\end{equation}
Note that $F_t$ implies that the projection $\Pi_\dom$ was not involved during the first  $t-1$ steps, since $\mathcal{C}\subset\dom$. Therefore, whenever $F_t$ holds, we can use~\eqref{eq:rewrite_normdiff}. The events $\{E\cap F_t\}_t$ form a nested sequence. Hence,
\begin{align}
  \PP\left[ E\cap F_T \right] &=\PP\left[ E\cap F_1 \right] \prod_{t=1}^{T-1}\PP\left[ E\cap F_{t+1}|E\cap F_t \right] \\
                              &=  \PP\left[ E\right] \prod_{t=1}^{T-1}\PP\left[ F_{t+1}|E\cap F_t \right]. \label{eq:p_decomposition}
\end{align}
In the above, we used that by the conditions of the theorem, $F_1$ always holds. We will now focus on the conditional probability $\PP\left[ F_{t+1}|E\cap F_t \right]$. Assuming $E$ and $F_t$ hold, we use~\eqref{eq:rewrite_normdiff} and~\eqref{eq:bound_two_sums} as
\begin{align}
  \norm{\widetilde{\omega}_{t+1}-\omega^*}^2 &=\norm{\omega_1-\omega^*}^2+\zeta^2 \sum_{k=1}^{t} \norm{\mathbf{v}_k}^2+2\zeta \sum_{k=1}^{t} (\omega^*-\omega_k)^\tp \mathbf{v}_k \\
                                &< \norm{\omega_1-\omega^*}^2+\zeta^2  D ^2t +2\zeta \sum_{k=1}^{t}(\omega^*-\omega_k)^\tp \nabla \risk(\omega_k)+B_\delta(t).
\end{align}
Since $F_t$ holds,  $\risk$ is convex at $\omega_k$ for all  $k\le t$, so $(\omega^*-\omega_k)^\tp \nabla \risk(\omega_k)\le \risk(\omega^*)-\risk(\omega_k)$. Taking into account that $\omega^* \in\argmin_{\omega\in\dom} \risk(\omega)$,
\begin{align}
  \norm{\widetilde{\omega}_{t+1}-\omega^*}^2&<\norm{\omega_1-\omega^*}^2+\zeta^2  D ^2t+B_\delta(t) \\
                                          &< \epsilon^2+(\epsilon_c^2-\epsilon^2)=\epsilon_c^2. \label{eq:desired_next_norm_bound}
\end{align}
In the last inequality, we used~\eqref{eq:containment_req} and the definition~\eqref{eq:ak_bound} of $B_a$. It follows from~\eqref{eq:desired_next_norm_bound} that $\omega_{t+1}=\widetilde{\omega}_{t+1}$, and, together with $F_t$, this gives that  $F_{t+1}$ also holds. Overall, we have shown that $\PP\left[ F_{t+1}|E\cap F_t \right]=1$. Substituting into~\eqref{eq:p_decomposition}, we get
\begin{equation}
  \PP\left[ F_T \right] \ge  \PP\left[ E\cap F_T \right] =\PP\left[ E \right] \ge 1-\delta. \label{eq:btprob_bound}
\end{equation}
Once we know that $F_T$ holds, SGD stays in the convex basin of $\risk$, and Theorem~\ref{thm:highp_sgd} becomes applicable (our SGD takes the same steps as SGD on a restriction $\left. \risk\right|_\mathcal{C}$ ). Choosing the same probability of failure $\delta$,
\begin{equation}
  \PP\left[ \left.\risk(\overline{\omega})-\risk(\omega^*)< \frac{\epsilon^2}{2\zeta T}+\frac{\zeta  D ^2}{2}+4 D  B_\omega \sqrt{\frac{2}{T}\log \frac{1}{\delta}}\quad \right|B_T\right] \ge  1-\delta.
\end{equation}
Together with~\eqref{eq:btprob_bound}, this gives the statement of the theorem.
\end{proof}

\section{Preconditioning in E2TC}
\label{sec:preconditioned_sgd}
Now we can discuss the SGD variant that is performed by E2TC. During the second stage of exploration, Algorithm~\ref{alg:main_algo} jointly refines the parameters $\omega=(\w,\theta)$. It minimizes the risk function $\risk$, which we provide here for a quick reference:
\begin{equation}
  \risk(\w,\theta)=\E_{X,r}\left[ \left(\w^\tp\phi_\theta(X)-r \right)^2 \right]. \label{eq:rewritef}
\end{equation}
Here, $r=\wstartp\pstar(X)+\eta$ is the random reward from action $X$, and $\eta\sim\mathcal{H}$ is zero-mean noise. The function is defined over a product of closed convex sets $\dom_\w\subset\R^d$ and $\dom_\theta\subset\R^{d_0}$. Given a random action-reward pair $(\widetilde{X}_t, \tilde{r}_t)$, we get the following unbiased gradient estimates w.r.t. parameters $\w_t$ and $\theta_t$:
\begin{align}
  \gv^\w_t&=\left.\nabla_\w\right|_{\w=\w_t}\left(\w^\tp_t\phi_{\theta_t}(\widetilde{X}_t)-\tilde{r}_t  \right)^2=2(\w^\tp_t\phi_{\theta_t}(\widetilde{X}_t)-\tilde{r}_t)\phi_{\theta_t}(\widetilde{X}_t), \label{eq:vv_def1}\\
    \gv^\theta_t&=\left.\nabla_\theta\right|_{\theta=\theta_t}\left(\w^\tp_t\phi_{\theta_t}(\widetilde{X}_t)-\tilde{r}_t  \right)^2=2\left(\w^\tp_t\phi_{\theta_t}(\widetilde{X}_t)-\tilde{r}_t  \right)J_\theta(\widetilde{X}_t)^\tp \w_t. \label{eq:vv_def2}
\end{align}
It is worth reminding here that $J_\theta(X) \in\R^{d\times d_0}$ is the Jacobian of $\phi_\theta(X)$ w.r.t. $\theta\in \R^{d_0}$. By Assumption~\ref{as:bound}, its norm is bounded by $\norm{J_\theta(\mathbf{x})}_2\le L_\phi$. Overall, $\gv^\w_t$ and $\gv^\theta_t$ can be bounded almost surely by $D_\w$ and $D_\theta$ respectively, defined in~\eqref{eq:dw_def} and~\eqref{eq:dtheta_def}.

SGD in E2TC is initialized at $\omega_0=(\wzero,\theta_0)$, where  $\wzero$ is produced by Ridge regression. Results from Chapter~\ref{ch:linear} give us high-probability bounds on the suboptimality gap $\risk(\omega_0)-\risk(\omega^*)$. Unfortunately, this cannot be directly plugged into the bounds in Theorem~\ref{thm:highp_sgd} or~\ref{thm:sgd_containment}. If the risk was Polyak-Łojasiewicz (strongly convex functions are also included in that class, but they are not as relevant for us, as we discuss at the end of Section~\ref{sec:ridge_lls}), then a low suboptimality gap of $\omega_0$ would actually imply that it is close in $2$-norm to some minimizer $\omega^*\in\argmin_\omega\risk(\omega)$~\cite{garrigos2023square}. Convex functions do not necessarily satisfy this property. 

Although Ridge regression cannot provide a reasonable bound on $\norm{\wzero-\wstar}$, in Section~\ref{sec:bandit_lls} we have seen that it does give us a bound~\eqref{eq:highp_bound_hsigl} on the matrix-weighted norm $\norm{\wzero-\wstar}_{\hsigl}$, which behaves asymptotically as
\begin{equation}
  \norm{\wzero-\wstar}_{\hsigl}^2=\widetilde{O}\left(\epsilon_0^2d+\frac{d}{\sqrt{T_1} }+\lambda\right). \label{eq:rewrite_asymp} 
\end{equation}
This behaviour is as expected: there is a constant term $O(\epsilon_0^2 d )$ that comes from misspecification, a term decaying with $T_1$, and the cost of regularization.

To get a bound on SGD suboptimality in terms of this norm, we need to modify the algorithm. Instead of $\risk$, we can minimize a similar loss $\crisk:\dom_{\widetilde{\w}}\times\dom_\theta\to\R$,
\begin{equation}
  \crisk(\widetilde{\w},\theta)=\risk\left(\hsigl^{-1 /2}\widetilde{\w}, \theta\right). \label{eq:def_ftilde}
\end{equation}
We introduced a change of variables 
\begin{equation}
  \widetilde{\w}=\hsigl^{1 /2}\w,\quad \dom_{\widetilde{\w}}=\left\{\widetilde{\w}:\ \hsigl^{-1 /2}\widetilde{\w}\in\dom_\w\right\}.  \label{eq:changew}
\end{equation}
Under this change, the 2-norm transforms into the desired matrix norm. Note that we used the regularized sample covariance $\hsigl$, as opposed to true covariance, because this is the matrix that is available when running the algorithm. Thus, results from Section~\ref{sec:ridge_lls} are not directly applicable, and we have to rely on the slightly worse bounds from Section~\ref{sec:bandit_lls}. The rule also requires that $\hsigl$ is invertible, which is why we cannot use $\hsigz$, even though this would give a better asymptotic bound than~\eqref{eq:rewrite_asymp}.

Running SGD on $\crisk$ is not equivalent to running it on $\risk$. The trajectory on $\crisk$ will be generated by the following update rules:
\begin{equation}
  \widetilde{\w}_{t+1}=\Pi_{\dom_{\widetilde{\w}}}(\widetilde{\w}_t-\zeta_\w\widetilde{\gv}^\w_t),\quad\theta_{t+1}=\Pi_{\dom_\theta}(\theta_t-\zeta_\theta\gv^\theta_t),
\end{equation}
where $\zeta_\w,\ \zeta_\theta>0$ are the learning rates corresponding to the last layer and to other parameters, respectively; $\widetilde{\gv}^\w_t$ and $\gv^\theta_t$ are the stochastic gradient estimates. The update rule of $\widetilde{\w}_t$ can be represented in the original space of $(\w,\theta)$ as
\begin{equation}
  \w_{t+1}=\Pi_{\dom_{\w}}(\w_{t}-\zeta_\w\hsigl^{-1 /2}\widetilde{\gv}^\w_t). \label{eq:vvtil_step}
\end{equation}
Since $\widetilde{\gv}^\w_t$ is an unbiased gradient estimator of $\crisk$, it must satisfy
\begin{equation}
  \E\left[ \widetilde{\gv}^\w_t |\widetilde{\w}_{t},\theta_{t}\right] =\nabla_{\widetilde{\w}} \crisk(\widetilde{\w}_t,\theta_t)=(\widehat{\Sigma(\theta_0)}+\lambda I)^{-1 /2}\nabla_\w \risk(\w_t,\theta_t). \label{eq:vvtil_def}
\end{equation}
Now, let $\gv^\w_t$ be a gradient estimate of $\risk$, i.e. a random variable such that
\begin{equation}
  \E\left[ \gv^\w_t |\w_t,\theta_t\right] =\nabla_\w \risk(\w_t,\theta_t).
\end{equation}
Then, 
\begin{equation}
  \widetilde{\gv}^\w_t=(\widehat{\Sigma(\theta_0)}+\lambda I)^{-1 /2}\gv^\w_t \label{eq:changev}
\end{equation}
satisfies~\eqref{eq:vvtil_def}, and therefore provides an unbiased gradient estimator for $\crisk$. Substituting it into~\eqref{eq:vvtil_step}, together with the unchanged rule for $\theta$, we get an update rule of the form
\begin{equation}
  \w_{t+1}=\Pi_{\dom_\w}(\w_{t}-\zeta_\w\hsigl^{-1}{\gv}^\w_t),\quad \theta_{t+1}=\Pi_{\dom_{\theta}}(\theta_t-\zeta_\theta\gv^\theta_t). \label{eq:vv_step}
\end{equation}
This is the update rule that we use in E2TC. In general, performing updates of the form $\w_{t+1}=\w_t-\zeta_\w A^{-1}\gv^\w_t$ instead of $\w_{t+1}=\w_t-\zeta_\w \gv^\w_t$ is called \emph{preconditioning} of SGD, and it is sometimes used to improve convergence of the algorithm in practice. As we shall see in the following chapter, the preconditioned variant with the update rule~\eqref{eq:vv_step} is also helpful theoretically. In our case, it provides bounds suitable for analysis after the first phase of E2TC. 

The coordinate transformation given by~\eqref{eq:def_ftilde} and~\eqref{eq:changew} hints at how the modified guarantees materialize. Let us briefly consider the general setting of a globally convex loss function from Section~\ref{sec:global_sgd} again. If we use updates of the form $\omega_{t+1}=\Pi_\dom(\omega_t-\zeta A^{-1}\gv_t)$, we can apply Theorem~\ref{thm:highp_sgd} to $\crisk(\widetilde{\omega})=\risk(A^{-1 /2}\widetilde{\omega})$, where $\widetilde{\omega}=A^{1 /2}\omega$. We would get a high-probability bound that translates back into the space of $\omega$ as
\begin{equation}
  \risk(\overline{\omega})-\risk(\omega^*)< \frac{\norm{\omega^*-\omega_1}_{A}^2}{2\zeta T}+\frac{\zeta}{2T}\sum_{t=1}^{T} \norm{\gv_t}_{A^{-1}}^2+
  4 D  B_\omega \sqrt{\frac{2}{T}\log \frac{1}{\delta}}. \label{eq:subopt_bound_sgd_default}
\end{equation}
Note how the matrix-weighted norms appeared in this bound compared to~\eqref{eq:subopt_bound_sgd_default}, which only uses $2$-norms.

%% file: chapters/combined.tex
\chapter{Analysis of the main algorithm}
\label{ch:combined}
Previous chapters provide the building blocks for analyzing the performance of the entire E2TC. Here, we bring these blocks together. We show sublinear regret assuming that the risk is locally convex around the minimum and that the pre-trained weights are, in a certain sense, close enough to the optimal ones. We require a special shape of the convex basin for our results to be applicable. Our main result, Theorem~\ref{thm:main}, provides regret and suboptimality gap bounds without hard-coded hyperparameter choices. This result can therefore be used in various situations as a template for hyperparameter tuning. Under the constraints of the theorem, one can optimize the desired bound w.r.t. the hyperparameters of E2TC. In Section~\ref{sec:hyperparams_no_assumptions}, we perform this optimization and show how theory dictates optimal hyperparameter choices for minimal sample complexity and regret. We arrive at a regret bound of $\widetilde{O}((KT)^{4 /5})$ in the regime $d=O(T^{2 /5})$, and to a family of dimension-dependent bounds for $d$ that scales faster than $T^{2 /5}$.

\section{Main result}
Convex loss functions that are not strongly convex occupy the central place in this thesis. Analogously to Theorem~\ref{thm:sgd_containment}, we can show that if SGD is initialized at a point $\omega_0=(\wzero,\theta_0)$ ``sufficiently close'' to a true minimizer $\omega^*=(\wstar,\theta^*)$, and the loss is convex in a basin around  $\omega^*$, then with high probability the trajectory of SGD will not escape the basin and the suboptimality gap of the final estimate will be small. Since we can only show that $\wzero$ is close to  $\wstar$ in terms of the matrix norm distance $\norm{\wzero-\wstar}_{\sigz}$ or $\norm{\wzero-\wstar}_{\hsigl}$, the shape of the convex basin has to be adjusted accordingly. The initial weights $\theta_0$ of the representation network also have to be close to  $\theta^*$. Since $\theta_0$ is provided to the algorithm from outside, we will simply require that the $2$-norm distance $\epsilon_\theta=\norm{\theta_0-\theta^*}$ is small. Note that $\epsilon_0$ given by~\eqref{eq:epsilonzero_def} can be bounded by  $\epsilon_\theta$ under Assumption~\ref{as:bound}:
\begin{equation}
\epsilon_0\le B_\w L_\phi \epsilon_\theta. \label{eq:epsilonzero_bound_epsilontheta}
\end{equation}
However, we believe it is instructive to keep both $\epsilon_\theta$ and  $\epsilon_0$ when formulating the result, since the bound~\eqref{eq:epsilonzero_bound_epsilontheta} can be rather loose.

The main theorem intentionally provides a general statement without committing to a hyperparameter selection. In Section~\ref{sec:hyperparams_no_assumptions}, we provide a less general and more clear statement which follows from this theorem when hyperparameters are chosen in a certain way. Before presenting the theorem, we discuss its assumptions in detail. These assumptions are made in addition to the three basic assumptions we presented in Section~\ref{sec:formulation}.

The convexity requirement on the risk function reads as follows:
\begin{assumption}[Local convexity of $\risk$] \label{as:localconv}
  Let $\risk:\dom_\w\times\dom_\theta \to\R$ be defined by~\eqref{eq:rewritef} on the domains
\begin{equation}
  \dom_\w=\left\{ \w:\ \norm{\w}\le B_\w \right\},\quad\dom_\theta=\left\{ \theta:\ \norm{\theta}\le B_\theta \right\}.
\end{equation}
  There exists $0<\epsilon_c<1$ such that $\risk(\w,\theta)$ is convex in the region $\conv$ around $(\wstar,\theta^*)\in\argmin_{\w\in\dom_\w,\theta\in\dom_\theta}\risk(\w,\theta)$. The region is defined as
  \begin{equation}
  \conv=\left\{ (\w,\theta):\ \norm{\w-\wstar}_{\sigz}^2+\norm{\theta-\theta^*}_2^2<\epsilon_c^2\right\} \subset \dom_\w\times\dom_\theta\label{eq:conv_def}.
  \end{equation}
\end{assumption}
The convex region $\conv$ is a ball around  $\omega^*$ with respect to a metric generated by a special norm which we will denote  $\norm{\cdot}_c$. For $\omega=(\w,\theta)$, it is defined as
\begin{equation}
  \norm{\omega}_c^2=\norm{\w}_{\sigz}^2+\norm{\theta}^2.
\end{equation}
The requirement for $\risk$ to be convex in the ball $\conv$ can be seen as somewhat contrived, especially because of the dependency of the norm $\norm{\cdot}_c$ on $\theta_0$ through $\Sigma_0$. However, there is no clear way to get rid of this dependency, because the linear estimate  $\wzero$ is not in general constrained along the nullspace of  $\Sigma_0$. If the projection of $\wzero$ onto this nullspace can be arbitrary, we need the projection of $\conv$ to be the entire nullspace as well. Otherwise, there is no guarantee that $\w_0$ lands into the basin. A helpful way to think about this requirement is as follows. When $(\w,\theta)\in\conv$, we have that $\norm{\w-\wstar}_{\Sigma_0}<\epsilon_c$ and $\norm{\theta-\theta^*}<\epsilon_c$. Therefore, we know from~\eqref{eq:twoterms} that $\risk(\w,\theta_0)-\risk(\wstar,\theta^*)<\epsilon_c+\epsilon_0$. When $\risk$ is $L_\theta$-Lipschitz w.r.t. $\theta$, we also know that $|\risk(\w,\theta)-\risk(\w,\theta_0)|<2L_\theta\epsilon_c$. Note that $\theta_0\in \conv_\theta$, since $\epsilon_\theta<\epsilon_{c,\theta}$. Together, these facts give us that
\begin{align}
  \risk(\w,\theta)-\risk(\wstar,\theta^*) &=\risk(\w,\theta)-\risk(\w,\theta_0)+\risk(\w,\theta_0)-\risk(\wstar,\theta^*) \\
                                  &<(2L_\theta+1)\epsilon_c+\epsilon_0.
\end{align}
Thus, the convexity condition can be replaced by a (stronger) requirement that $\risk$ is convex in the region $\conv_l$ of low loss, which is somewhat more intuitive:
 \begin{equation}
  \conv\subset\conv_l=\left\{ (\w,\theta):\ \risk(\w,\theta)-\risk(\wstar,\theta^*)< (2L_\theta+1)\epsilon_c+\epsilon_0\right\}.
\end{equation}

Assumption~\ref{as:localconv} imposes $\epsilon_c<1$ without loss of generality. If it does not hold, we can rescale the parameter space appropriately and adjust the constants $B_\w$ and $B_\theta$ that bound the size of the domain. At the same time,  $\epsilon_c<1$ significantly simplifies the exposition of the result.

SGD only has a chance of initializing within the convex basin $\conv$ if the provided pre-trained weights $\theta_0$ are good enough. The following assumption formalizes this requirement:
\begin{assumption}[Quality of $\theta_0$] \label{as:quality}
  The provided $\theta_0$ satisfies
   \begin{equation}
     2\epsilon_0^2d+\epsilon_\theta^2<\epsilon_c^2, \label{eq:epsilons_small}
  \end{equation}
  where
  \begin{equation}
    \epsilon_0^2=\E_X\left[ \left(\wstartp(\pzero(X)-\pstar(X))\right)^2\right],\quad \epsilon_\theta^2=\norm{\theta_0-\theta^*}^2.
  \end{equation}
\end{assumption}
The intuition for~\eqref{eq:epsilons_small} is as follows. In the simplest case when $\epsilon_\theta>\epsilon_c$, SGD cannot be initialized in $\conv$ simply because for any  $\w$,  $(\w,\theta_0)\not\in \conv$. Alternatively, if  $\epsilon_0^2d$ is too large, irreducible misspecification term in the bound~\eqref{eq:wclose_sigz_datapoor} on $\norm{\wzero-\wstar}^2_{\sigz}$ from Theorem~\ref{thm:nodimregret} grows, and any amount of data used for Ridge regression will not help us.

When we discussed SGD on locally convex losses in Section~\ref{sec:locally_convex_loss}, we saw that the learning rate has to satisfy a constraint~\eqref{eq:new_contained_sgd_suboptimality_bound} for SGD to stay in the convex basin. Our main theorem will introduce similar restrictions on its hyperparameters $\lambda$ and $T_1$, which will ensure that SGD is successfully initialized in $\conv$ with high probability. Another restriction on $\zeta$ and  $T_2$ will ensure that SGD stays in $\conv$ during training. The proof of the theorem, which combines techniques from the previous chapters, is deferred to Appendix~\ref{ap:proof_main}.

\begin{theorem}[Suboptimality gap for Algorithm~\ref{alg:main_algo}, locally convex loss]
  \label{thm:main} $\ $ \\
  Let Algorithm~\ref{alg:main_algo} be run with hyperparameters $\zeta_\w=\zeta_\theta=\zeta$, $T_1,\ T_2,$, and $\lambda$. Let Assumptions~\ref{as:context},~\ref{as:realizability},~\ref{as:bound},~\ref{as:localconv}, and~\ref{as:quality} hold. Let $\epsilon_\w \in\R_+$ be arbitrary such that 
\begin{equation}
 2\epsilon_0^2d+\epsilon_\theta^2<\epsilon_\w^2+\epsilon_\theta^2<\epsilon_c^2,
\end{equation}
Assume that $\lambda$ is chosen sufficiently small, so that
\begin{equation}
  2\epsilon_0^2d+\frac{B_\w^2}{2}\lambda < \epsilon^2_\w, \label{eq:require_small_epsilons} \\
\end{equation}
and let $\Delta\epsilon_\w=\epsilon^2_\w-2\epsilon_0^2d-\frac{B_w^2}{2}\lambda$. Assume that $\zeta$ is also sufficiently small, so that
\begin{align}
&4B_\w^2\lambda +\frac{16B_\w^2+4D_\w^2\zeta^2 / \lambda}{T_1}\left(B_\phi^2\log \frac{2d}{\delta}+\sqrt{B_\phi^4\log^2 \frac{2d}{\delta}+2B_\phi^4T_1\log\frac{2d}{\delta}}\right) \nonumber\\ 
\begin{split}
&+\zeta^2 T_2\left(\frac{D_\w^2}{\lambda}+ D_\theta\right) \\
&+ 20 B_\w D_\w\zeta\sqrt{T_2\left((\log \log (64eB_\w D_\w T_2\zeta^2))_++\log(2 /\delta)\right)}\label{eq:require_small_zeta}
\end{split} \\
&+ 20 B_\theta D_\theta\zeta\sqrt{T_2\left((\log \log (64eB_\theta D_\theta T_2\zeta^2))_++\log(2 /\delta)\right)} < \epsilon_c^2-\epsilon_\w^2-\epsilon_\theta^2.\nonumber 
\end{align}
 Let also $\delta\in\R_+$ be such that $\delta<\exp(-2.6)$. There exists $N_1(\delta,\lambda,\epsilon_0,d,\Delta\epsilon_\w)$ 
with asymptotics
\begin{equation}
  N_1=\Theta\left( \left( 1+\epsilon_0^2d \right)\frac{\log (d /\delta)}{\lambda\deps^2}  \right) \label{eq:asymp_n1}
\end{equation}
 such that, for $T_1\ge N_1$, the following statements hold simultaneously with probability not less than $1-10\delta$:
\begin{enumerate}
  \item $\norm{\wzero-\wstar}_{\Sigma_0}<\epsilon_\w$.
  \item For all $t=1,\ldots,T_2$, $(\w_t,\theta_t)\in \conv$.
  \item The final suboptimality gap is bounded by
    \begin{multline}
       \risk(\overline{\w}, \overline{\theta})-\risk(\wstar,\theta^*)< \frac{1}{2\zeta T_2}\left(\left(\sqrt{d(\epsilon_0^2+\epsilon_{\delta })}+\epsilon_\eta+\sqrt{\lambda} \norm{\wstar}\right)^2+\epsilon_\theta^2\right) \\
       +\frac{\zeta D_\w^2}{2\lambda} +\frac{\zeta D_\theta^2}{2}+Q_\delta, \label{eq:main_local_conv_bound}
    \end{multline} 
    where 
  \begin{gather}
    \epsilon_\delta=4B_w^2B_\phi^2\sqrt{\frac{\log(1 /\delta)}{2T_1}},\quad  \epsilon_\eta=2\sqrt{\frac{d\log 6+\log(1 /\delta)}{T_1}},\\
\epsilon_s=4(D_wB_w+D_\theta B_\theta)\sqrt{\frac{2\log(1 /\delta)}{T_2}}, \\ 
  Q_\delta=(4B_\w B_\phi+2B_\w D_\w + 2B_\theta D_\theta)\sqrt{\frac{2\log(1 /\delta)}{T_2}}. 
  \end{gather}
\end{enumerate}
\end{theorem}
The theorem could be extended to all $\delta<1$ (as opposed to  $\delta<\exp(-2.6)$), but at the cost of loosing the clarity of presentation. The requirement $\delta<\exp(-2.6)$ comes from Theorem~\ref{thm:ridge_base}. One of the concentration inequalities used in~\cite{hsu2012random} to establish this result has a simpler form for small $\delta$. 


This result tells the following story about optimizing a locally convex $\risk$ starting with the pre-trained feature weights $\theta_0$. If the pre-trained feature weights are sufficiently close to $\theta^*$, we can, with enough data, bring the estimate weights of the last layer $\wzero$ close to $\wstar$ during the first stage of E2TC. Once we did this, $(\w_0,\theta_0)$ is in the convex basin of the loss. Therefore, SGD performed during the second stage of E2TC enjoys the convergence guarantees similar to those provided in Section~\ref{sec:locally_convex_loss}.

The theorem introduces an extra hyperparameter $\epsilon_\w$. It stands for a high-probability upper bound on $\norm{\wzero-\wstar}_{\Sigma(\theta_0)}$. This hyperparameter induces the tension between the first and the second stages of exploration. To make $\epsilon_\w$ smaller, we have to increase $T_1$, thus restricting our budget for $T_2$. A certain amount of exploration in the first stage is inevitable, though, because we need $\epsilon_\w^2<\epsilon_c^2-\epsilon_\theta^2$ to get into the convex basin. What is more, when $\epsilon_\w$ is large, requirement~\eqref{eq:require_small_zeta} will ask for a small $\zeta$. Thus, limited exploration in the first stage will force us to a smaller learning rate in $\w$ in the second stage to avoid escaping the convex basin.

The lower bound $T_1>N_1$ is necessary to ensure that $\norm{\w_0-\wstar}^2_{\sigz}<\epsilon_\w^2$. As ${T_1\to\infty}$, the distance $\norm{\wzero-\wstar}_{\Sigma(\theta_0)}^2$ does not go to zero. Misspecification and regularization (with parameter $\lambda$) both contribute to the limit of this distance. Expression $2\epsilon_0^2d+\frac{B_w^2}{2}\lambda$ upper bounds this limit, as shown by Lemma~\ref{lem:ridgeregerr_wstar}. This is why we need the inequality~\eqref{eq:require_small_epsilons}. Moreover, if $\Delta\epsilon_\w=\epsilon^2_\w-2\epsilon_0^2d-\frac{B_w^2}{2}\lambda$ is small, we will need a large $T_1$ to get ${\norm{\wzero-\wstar}^2_{\Sigma(\theta_0)}<\epsilon_\w^2}$. This justifies the dependency of $N_1$ on $\Delta\epsilon_\w$. Asymptotics~\eqref{eq:asymp_n1} of $N_1$ as $d\to \infty $ and $\delta,\lambda \to 0$ can be slightly improved, at the expense of becoming too large to be handled efficiently. The tighter asymptotics of $N_1$ are given by~\eqref{eq:full_asymp_n1} in the proof of the theorem. 

Note that if we only perform linear estimation of $\w_0$ and omit SGD, Theorem~\ref{thm:nodimregret} asks for $\lambda=1 /\sqrt{T_1}$, or, equivalently, $T_1=1 /\lambda^2$. The bound $T_1>N_1$ in Theorem~\ref{thm:main}, on the other hand, only has a $1/\lambda$ term, giving us more freedom in selecting $T_1$. $N_1$ also depends inversely on $\Delta \epsilon_\w$. In case this quantity is so small that $N_1$ becomes larger than $T$ (or when assumptions~\eqref{eq:require_small_epsilons} do not hold), we can still rely on the techniques from Chapter~\ref{ch:linear} to estimate  $\w_0$, but then we have to incur the linear regret coming from misspecified weights $\theta_0$.

We also impose that the learning rates for $\zeta_\w$ and  $\zeta_\theta$ here should be equal. This requirement can be lifted. Consider a function $\overline{\risk}$ in the transformed space, defined as
 \begin{equation}
  \hat{\risk}(\widetilde{\w},\theta):=f\left(\frac{\zeta_\w}{\zeta_\theta}\widetilde{\w}, \theta \right).
\end{equation}
Here, we introduced a coordinate transformation $\widetilde{\w}=(\zeta_\w /\zeta_\theta)\w$. As one can easily verify, running SGD on $\risk$ with learning rates $\zeta_\w$ for $\w$ and  $\zeta_\theta$ for $\theta$ is equivalent to running SGD on  $\overline{\risk}$ with a single learning rate $\zeta_\theta$. If $\risk$ was locally convex around its minimum, $\overline{\risk}$ will keep this property, albeit with a different convex region. Note how this coordinate transformation is similar to the transformation~\eqref{eq:def_ftilde} that we used to describe preconditioned SGD. This transformation gives us a straightforward way to generalize Theorem~\ref{thm:main} to the case of different learning rates, by applying the original theorem to $\overline{\risk}$ and then transforming the results back to $\risk$. However, the learning rates would not be cleanly separated in the resulting statement, and the requirement~\eqref{eq:require_small_zeta} would include both learning rates.


\section{Hyperparameter tuning}
\label{sec:hyperparams_no_assumptions}
In this section we select hyperparameters for E2TC. Theorem~\ref{thm:main} poses a constrained optimization problem of minimizing the bound~\eqref{eq:main_local_conv_bound} with respect to the hyperparameters under the constraints~\eqref{eq:require_small_epsilons} and~\eqref{eq:require_small_zeta}. In addition to setting the learning rate  $\zeta$ and the regularization parameter  $\lambda$, we are interested in the optimal distribution of a fixed exploration time $T_{\mathrm{exp}}=T_1+T_2$ between the two stages of exploration. This question is important both to obtain optimal rates of the regret in the bandit setup and to get sample complexity bounds with a fixed amount of data. The following result shows the asymptotic behavior of the optimal hyperparameters satisfying the constraints of Theorem~\ref{thm:main}. 

\begin{theorem}[Hyperparameter tuning for E2TC]
  \label{thm:hyperparam_nocov}
Let $T_{\mathrm{exp}}=T_1+T_2$ be the total number of steps that E2TC spends in the two exploration phases. Suppose assumptions~\ref{as:context},~\ref{as:realizability},~\ref{as:bound},~\ref{as:localconv}, and~\ref{as:quality} hold. Define $  \Delta\epsilon\coloneqq\epsilon_c^2-2\epsilon_0^2d-\epsilon_\theta^2$. Let $\alpha,\beta\in\R_+$. Choose the hyper-parameters such that $T_1< T_{\mathrm{exp}}$ and
  \begin{gather}
    T_1=O\left( \left(1+\epsilon_0^2d  \right) \frac{d^\alpha T_{\mathrm{exp}}^\beta\log(d /\delta)\log(1 /\delta)} {\Delta\epsilon^3} \right), \label{eq:opt_t1_nocov} \\
    \ T_2=T_{\mathrm{exp}}-T_1 \\
    \lambda=O\left(\frac{\Delta\epsilon}{\sqrt{\log (1/ \delta)} }\right),\ \zeta=O\left(\frac{\Delta\epsilon}{\sqrt{T_2\log (1/ \delta)} }\right), \label{eq:zetalabmda_sel} \\
    \epsilon_\w^2=\frac{1}{2}\left( \epsilon_c^2-\epsilon_\theta^2+2\epsilon_0^2d \right).
  \end{gather}
  Then the constraints~\eqref{eq:require_small_epsilons} and~\eqref{eq:require_small_zeta} are met and E2TC satisfies
  \begin{align}
        \risk(\overline{\w},\overline{\theta})-\risk(\wstar,\theta^*)=
  O\Bigg(& \sqrt{\frac{\log(1 /\delta)}{T_2}} \left(1+\frac{\epsilon_0^2d+\epsilon_\theta^2}{\Delta\epsilon}+\frac{\Delta\epsilon}{\log(1 / \delta)}\right) \notag\\
  & +\frac{\Delta\epsilon^2 d^{1-\alpha /2}}{(1+\epsilon_0^2d)\log(d /\delta)\sqrt{T_2} T_{\mathrm{exp}}^{\beta /2}}  \Bigg)\label{eq:suboptimality_nocov}
  \end{align}
with probability not less than $1-10\delta$.
\end{theorem}

For the purposes of this theorem, the $O$-notation hides $B_\w,B_\theta,B_\phi,B_\eta,D_\w,D_\theta$, and constant factors. The asymptotic expressions for $T_1,\ \lambda,\ \zeta$ should be understood in a sense that there exist functions $\lambda(\Delta\epsilon,\delta)$, $\zeta(\Delta\epsilon,T_2,\delta)$, and $T_1(\Delta\epsilon,\epsilon_0,d,\delta)$ with the provided asymptotics such that, for large enough $T_{\mathrm{exp}}$ and $d$ and for small enough $\delta$, when the hyperparameters are selected according to these functions, the suboptimality bound will satisfy~\eqref{eq:suboptimality_nocov}. We would like to contrast this interpretation to the \emph{incorrect} idea that any selection of hyperparameters with the provided asymptotics would lead to the suboptimality bound~\eqref{eq:suboptimality_nocov}.

$T_1$ can only be selected according to~\eqref{eq:opt_t1_nocov} if the provided asymptotic expression for it is below $T_{\mathrm{exp}}$. This is why we left two hyperparameters in the new theorem --- the powers $\alpha,\beta$ in the expression for  $T_1$ can be adjusted based on the situation. For example, if $d\ll T_{\mathrm{exp}}$, we are in the data-rich regime. In this case, selecting $\alpha=2$ avoids the dependency of the bound~\eqref{eq:suboptimality_nocov} on $d$ altogether. In many cases, we can expect this to work, because $d$ is the dimension of the last layer $\w$, not of the entire parameter space of $(\w, \theta)$.

We can compare the bound~\eqref{eq:suboptimality_nocov} to~\eqref{eq:new_contained_sgd_suboptimality_bound}, which was derived for a simpler case of a locally convex loss. The new bound shows a very similar dependence on epsilons, except for the final term that depends on $d$ and comes from the need to estimate the first layer from scratch. The last term is the only one with explicit dependence on $d$. We note here that the requirements of the theorem enforce $2\epsilon_0^2d+\epsilon_\theta^2<1$, so $\epsilon_0^2d+\epsilon_\theta$ enjoy a dimension-independent lower bound.

\begin{proof}
  First, we replace the bound~\eqref{eq:main_local_conv_bound} by a slightly looser one, using $(a+b+c)^2\le 3(a^2+b^2+c^2)$. Under the requirements of Theorem~\ref{thm:main}, 
\begin{align}
           \risk(\overline{\w}, \overline{\theta})-\risk(\wstar,\theta^*)< \frac{1}{2\zeta T_2}&\left(3\epsilon_0^2d+3\epsilon_\delta d+3\epsilon_\eta^2+3B_\w^2\lambda+\epsilon_\theta^2\right) \notag \\
       & +\frac{D_\w^2\zeta}{2\lambda} +\frac{\zeta D_\theta^2}{2}+Q_\delta \label{eq:local_conv_bound_nocov}
\end{align}
where we also used $\norm{\wstar}\le B_\w$. In this bound, it is easy to express the optimal $\lambda$ using the other hyperparameters. The dependence on  $\lambda$ is contained in two terms,
   \begin{equation}
     \frac{3B_\w^2\lambda}{2T_2\zeta}+\frac{D_\w^2\zeta}{2\lambda}.
  \end{equation}
The minimal value of this expression with respect to $\lambda$ is achieved for
\begin{equation}
  \lambda=\frac{D_\w}{B_\w\sqrt{3}}\zeta \sqrt{T_2}. \label{eq:opt_lambda_nocov}
\end{equation}
As we shall see, we can pick $\lambda$ with the same asymptotic dependency of  $O(\zeta\sqrt{T_2})$ satisfying the constraints of Theorem~\ref{thm:main}. For now, we will substitute~\eqref{eq:opt_lambda_nocov} into~\eqref{eq:local_conv_bound_nocov} and separate the terms that depend on $\zeta$ from those that do not:
\begin{multline}
  \risk(\overline{\w}, \overline{\theta})-\risk(\wstar,\theta^*)< \frac{1}{2\zeta T_2}\left(3\epsilon_0^2d+3\epsilon_\delta d+3\epsilon_\eta^2+\epsilon_\theta^2\right)+\frac{\zeta D_\theta^2}{2} \\
       +\frac{B_\w D_\w\sqrt{3}}{\sqrt{T_2}} +Q_\delta, \label{eq:zeta_dep_bound_nocov}
\end{multline}
We need to minimize this expression under the constraint~\eqref{eq:require_small_zeta} on $\zeta$. While the exact optimization problem is analytically intractable, we can still get an optimal asymptotic rate for the solution in the regime $T_{\mathrm{exp}}\to \infty $ and $\delta\to 0$. We first deal with the part of~\eqref{eq:require_small_zeta} involving $T_1$. Theorem~\ref{thm:main} imposes that $T_1$ should be above $N_1$, with $N_1$ having the asymptotic form~\eqref{eq:asymp_n1}. We would like to remind the reader that $\epsilon_\w$ is a hyperparameter introduced by Theorem~\ref{thm:main}, which has to satisfy
\begin{equation}
  2\epsilon_0^2 d+\frac{B_\w^2}{2}\lambda<\epsilon_\w^2<\epsilon_c^2-\epsilon_\theta^2.
\end{equation}
Theorem~\ref{thm:main} also defines $\Delta\epsilon_\w=\epsilon_\w^2-2\epsilon_0^2d-\frac{B_\w^2}{2}\lambda$. Thus, through our choice~\eqref{eq:opt_lambda_nocov} of $\lambda$, $\Delta\epsilon_\w$ may also depend on $\zeta$ and $T_2$. Using 
\begin{equation}
  T_1=\Omega\left(\frac{\log (d /\delta)}{\lambda\Delta\epsilon_\w^2}  \right),
\end{equation}
we infer that
\begin{align}
  &\frac{16B_\w^2+4D_\w^2\zeta^2 / \lambda}{T_1}\left(B_\phi^2\log \frac{2d}{\delta}+\sqrt{B_\phi^4\log^2 \frac{2d}{\delta}+2B_\phi^4T_1\log\frac{2d}{\delta}}\right) \nonumber \\
  &\qquad\qquad\qquad=O\left( \lambda\Delta\epsilon_\w^2+\sqrt{\lambda}\Delta\epsilon_\w + \zeta^2\Delta\epsilon_\w^2+\frac{\zeta^2\Delta\epsilon_\w}{\sqrt{\lambda} }\right) \\
  &\qquad\qquad\qquad=O\left(\Delta\epsilon_\w^2\zeta\sqrt{T_2}+\Delta\epsilon_\w\sqrt{\zeta}T_2^{ 1/4}+\Delta\epsilon_\w^2\zeta^2+\frac{\Delta\epsilon_\w\zeta^{3 /2}}{T_2^{1 /4}}  \right).
\end{align}
Substituting this along with the expression~\eqref{eq:opt_lambda_nocov} for $\lambda$ into~\eqref{eq:require_small_zeta}, we get the following asymptotic requirement for $\zeta$:
\begin{multline}
  O\Big(\zeta\sqrt{T_2}+\Delta\epsilon_\w^2\zeta\sqrt{T_2}+\Delta\epsilon_\w\sqrt{\zeta}T_2^{ 1/4}+\Delta\epsilon_\w^2\zeta^2+\frac{\Delta\epsilon_\w\zeta^{3 /2}}{T_2^{1 /4}} \\
  +\zeta\sqrt{T_2(\log\log(\zeta^2 T_2)+\log(1 /\delta))} \Big)<\epsilon_c^2-\epsilon_\w^2-\epsilon_\theta^2. \label{eq:require_small_zeta_nocov}
\end{multline}
We set $\epsilon_\w$ such that
\begin{equation}
  \epsilon_\w^2=\frac{1}{2}\left( \epsilon_c^2-\epsilon_\theta^2+2\epsilon_0^2d \right).
\end{equation}
We claim that there exists a constant $c>0$ such that setting
\begin{equation}
  \zeta=c \frac{\Delta\epsilon}{\sqrt{T_2\log(1 /\delta)} } \label{eq:zeta_asymp_nocov}
\end{equation}
will satisfy the requirements of Theorem~\ref{thm:main} as $T_2\to \infty ,\ \delta\to 0$. Substituting expressions for $\zeta$ and $\epsilon_\w$ into~\eqref{eq:require_small_zeta_nocov} and rearranging,
\begin{multline}
  O\Bigg(\frac{c\Delta\epsilon}{\sqrt{\log(1 /\delta)} }\left( 1+\Delta\epsilon_\w^2+\sqrt{\log\log \left(\frac{1}{\sqrt{\log(1 /\delta)} }\right)+\log(1 /\delta)} \right) \\
  +\frac{\Delta\epsilon_\w\sqrt{c\Delta\epsilon} }{\left( \log(1 /\delta) \right)^{1 /4} }+\frac{\Delta\epsilon_\w^2c^2\Delta\epsilon^2}{T_2\log(1 /\delta)}+\frac{\Delta\epsilon_\w c^{3 /2}\Delta\epsilon^{3 /2}}{T_2(\log(1 /\delta))^{3 /4}}\Bigg) < \frac{\Delta\epsilon}{2}.
\end{multline}
As $\delta\to 0$ while the epsilons and $c$ stay fixed, we can drop the terms that approach $0$ in the limit, and the asymptotic expression above simplifies to
\begin{equation}
  O(c\Delta\epsilon) < \frac{\Delta\epsilon}{2}.
\end{equation}
Thus, selecting a small enough constant $c$ will satisfy~\eqref{eq:require_small_zeta}. Having selected $\zeta$, we substitute it back into the expression~\eqref{eq:opt_lambda_nocov}, and get an asymptotic expression for $\lambda$:
 \begin{equation}
  \lambda=O\left(\frac{\Delta\epsilon}{\sqrt{\log(1 /\delta)} }\right).
\end{equation}
For small enough $\delta$, this selection of $\lambda$ ensures that the requirement~\eqref{eq:require_small_epsilons} is also satisfied. Since $\lambda\to 0$ when $\delta\to 0$, the asymptotic expression for $\Delta\epsilon_\w$ is simply  $\Delta\epsilon_\w=O(\Delta\epsilon)$. Substituting this into~\eqref{eq:asymp_n1}, we learn that we can select $T_1$ such that
\begin{equation}
  T_1=\Theta\left( \left(1+\epsilon_0^2d  \right) \frac{d^\alpha T_{\mathrm{exp}}^\beta\log(d /\delta)\log(1 /\delta) }{\Delta\epsilon^3} \right) .\label{eq:select_t1_nocov}
\end{equation}
Note that we introduced a factor of $d^\alpha T_{\mathrm{exp}}^\beta$, which can only increase $T_1$ over the lower bound~\eqref{eq:asymp_n1}, so the requirement for $T_1$ is still satisfied. Finally, we can substitute the selected $\zeta$ into the suboptimality bound~\eqref{eq:zeta_dep_bound_nocov} and arrive at
\begin{align}
  &\risk(\overline{\w}, \overline{\theta})-\risk(\wstar,\theta^*) \nonumber\\
  &\quad< O\Bigg(\frac{\sqrt{\log (1/\delta)} }{\Delta\epsilon\sqrt{T_2}}\Bigg(\epsilon_0^2d+\epsilon_\theta^2+\frac{d}{T_1}+\frac{\log(1 /\delta)}{T_1}+d\sqrt{\frac{\log(1 / \delta)}{T_1}} \Bigg) \label{eq:prevriskbound} \\
  &\qquad\qquad\qquad\qquad\qquad\qquad\qquad\qquad+\frac{\Delta\epsilon}{\sqrt{T_2\log(1 /\delta)} }+\sqrt{\frac{\log(1/\delta)}{T_2}}  \Bigg) \nonumber \\
  &\quad= O\Bigg(\sqrt{\frac{\log(1 /\delta)}{T_2}} \Bigg(1+\frac{\epsilon_0^2d+\epsilon_\theta^2}{\Delta\epsilon}+\frac{\Delta\epsilon}{\log(1 / \delta)}\Bigg) +\frac{\Delta\epsilon^2 d^{1-\alpha /2}}{(1+\epsilon_0^2d)\sqrt{T_2} \log(d /\delta)}  \Bigg).
\end{align}
In the last equality, we substituted~\eqref{eq:select_t1_nocov} instead of $T_1$ and dropped the terms that decay quicker than the remaining ones. \[\]
\end{proof}

With this bound on the suboptimality gap at hand, it is easy to show a regret guarantee:
\begin{corollary}[Regret of E2TC]
  Let the assumptions of Theorem~\ref{thm:hyperparam_nocov} hold. In addition, assume that the parameters related to sizes of convex regions ($\epsilon_0$, $\epsilon_\theta$,  $\epsilon_c$, $\Delta\epsilon$) are asymptotically $O(1)$. Assume also that $\zeta,\lambda$ are selected according to~\eqref{eq:zetalabmda_sel}, and $T_1$ --- according to~\eqref{eq:opt_t1_nocov} (without yet fixing $\alpha$ and $\beta$). Then, the regret of E2TC admits a bound
\begin{equation}
  R_{1:T}=\widetilde{O}(T_1+T_2+KT T_2^{-1 /4} + KT\sqrt{d}T_1^{-1 /4} T_2^{-1 /4}). \label{eq:e2tc_regret}
\end{equation}
When  $K=O(T^{1 /5})$ and  $d=O(T^{2 /5})$, we can select $T_1=d^2, T_2=(KT)^{4 /5}$ and get a bound $R_{1:T}=\widetilde{O}((KT)^{4 /5})$. 
\end{corollary}
\begin{proof}
  As long as hyperparameters are selected according to Theorem~\ref{thm:hyperparam_nocov}, we know that the conditions for the regret bound will be satisfied. Hence, the suboptimality gap $\risk(\overline{\w}, \overline{\theta})-\risk(\wstar,\theta^*)$ will be with high probability bounded by~\eqref{eq:prevriskbound}. We can now use Theorem~\ref{thm:risktoregret} to translate this risk 
bound into a bound on regret, and arrive at~\eqref{eq:e2tc_regret}. \[\]
\end{proof}

Depending on how $d$ and $K$ scale with $T$, we can select different $T_1$ and $T_2$ and achieve different regrets. For example, we can assume that $d=\Theta(T^\gamma)$ for some  $\gamma>0$ and $K=O(1)$. Depending on  $\gamma$, the regret bounds can be:
\begin{enumerate}
  \item $\gamma\le \frac{2}{5}$. Select $T_1=O(d^2), T_2=O(T^{4 /5})$, achieve $R_{1:T}=O(T^{4/5})$.
  \item $\frac{2}{5}<\gamma<1$. Select $T_1=T_2=O(T^{(\gamma+2)/3})$, achieve $R_{1:T}=O(T^{(\gamma+2)/3})$.
  \item $\gamma\ge 1$. Sublinear regret cannot be reached.
\end{enumerate}
We note that if $\gamma> 1$, the hidden dimension of the last layer of the network is larger than the number of data samples. In this regime, we should not expect local convexity of the loss anyway. NTK- or Polyak-Łojasiewicz-based analysis would be more appropriate. Alternatively, we can fall back to the lazy training analysis from Chapter~\ref{ch:linear}. If the misspecification $\epsilon_0$ is small enough, the regret bound of $O(\epsilon_0\sqrt{d}KT+(KT)^{4 /5})$ might still be useful.

%% file: chapters/experiments.tex
\chapter{Experimental evaluation}
\label{ch:experiments}
\section{Online MNIST classification}
\label{sec:mnist}
Although predicting the reward from the input features is a regression task, our badit setup can be easily adapted to classification problems. Here, we consider the well-known task of handwritten digit recognition from the MNIST dataset~\cite{lecun1998gradient}. To evaluate the effects of pre-training, we split the data into two parts. Digits 0-4 are used to pre-train a classifier. The bandit is then given the pre-trained weights of the representation network and evaluated on digits 5-9.

The $5$-class classification problem is modeled by the bandit setup as follows. At time $t$, we are given a $28\times 28$ greyscale image of a digit as a vector $I_t\in\R^{784}$ drawn uniformly at random from the dataset. We represent the individual features of the action $a\in\{1,\ldots,5\}$ as $X_{t,a}=(\mathbf{0};\ldots; I_t;\ldots; \mathbf{0})\in\R^{3640}$, where $I_t$ is put into the  $a$-th block of size $784$. Five actions correspond to assigning a digit (for pre-training it is $0$ to  $4$, for the bandit ---  $5$ to $9$). If the correct digit is assigned, the reward is  $1$. Otherwise, it is $0$. This setup is employed as a testing ground in many bandit papers~\cite{kassraie2022neural,ban2021ee,zhang2020neural}. Assumptions~\ref{as:context} on the context distribution and~\ref{as:bound} on regularity are satisfied for this problem. Realizability of the reward function (Assumption~\ref{as:realizability}) is not satisfied exactly. Nevertheless, in this case we know that a trained neural network can achieve a success rate of over $99\%$ on the test set, so the rewards are very close to being realizable.

To make the pre-trained weights more transferable, we used several regularization techniques. First, we dealt with the problem of low-rank representations. Since the reward model only predicts a one-dimensional value, there is no need for the representations to span the entire space $\R^d$. In practice, all vectors $\pzero(X_{t,a})$ approximately lie in a low-dimensional subspace of $\R^d$. If the weights $\theta_0$ are only used on the data from the pre-training distribution, this scheme still works. When the weights are transferred, this brings issues. Low-rank representations do not convey valuable information about the input. Re-trained model will put the representations in a different subspace, and two random low-dimensional subspaces are with high probability nearly orthogonal in $\R^d$ for large $d$.

We use the following method to enforce orthogonality of the features. For each batch $\left\{(X_{i},r_i)\right\}_{i=1}^B$ received during pre-training, we request that the output features be approximately mutually orthogonal. We regularize the MSE loss with an additive term of
\begin{equation}
  \frac{c_1}{B^2}\sum_{i\ne j}\frac{\left|\phi_\theta(X_i)^\tp\phi_\theta(X_j)\right|}{\norm{\phi_\theta(X_i)}\norm{\phi_\theta(X_j)}}, \label{eq:cos_reg}
\end{equation}
where $c_1$ is a hyperparameter corresponding to the relative importance of this term. This enforces the cosines of angles between $\phi_\theta(X_i)$ and $\phi_\theta(X_j)$ to be close to zero.  

\begin{figure}[t]
  \centering
  \includegraphics[width=\textwidth]{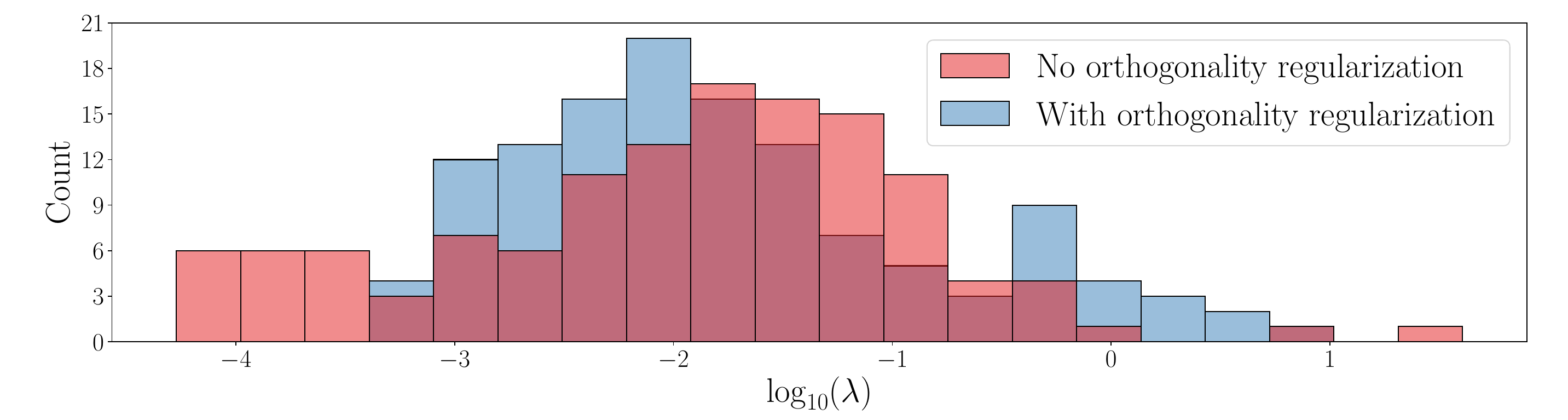}
  \caption{Distributions of spectra of $\widehat{\Sigma(\theta_0)}$ with and without regularization~\eqref{eq:cos_reg}.}
  \label{fig:spectrum_hist}
\end{figure}

To make representations more informative, we also added a decoder network $\psi_{\tilde{\theta}}$ on top of the representation network. We trained its parameters $\tilde{\theta}$ jointly with the rest of the predictor network, and added the auto-regressive squared error into the loss function. Since the features should capture the digit, but not its class, we did not regress into the zero-padded input to the network, but into the $784$-dimensional flattened original image. Let $\{I_k\}_{k=1}^B$ be the non-zero-padded $784$-dimensional images corresponding to the training batch $\{(X_k,r_k)\}_{k=1}^B$, where  $X_k$ is the zero-padded input and  $r_k\in\{0,1\}$ is the expected output of the reward model. Altogether, the loss on this batch is
\begin{equation}
  \begin{split}
  &\mathcal{L}(\w,\theta,\tilde{\theta})=\frac{1}{B}\sum_{k=1}^{B} (\w^\tp\phi_\theta(X_k)-r_k)^2
  + \frac{c_1}{B^2}\sum_{i\ne j}\frac{\left|\phi_\theta(X_i)^\tp\phi_\theta(X_j)\right|}{\norm{\phi_\theta(X_i)}\norm{\phi_\theta(X_j)}} \\
  &\qquad+ \frac{c_2}{784B}\sum_{k=1}^{B}\norm{\psi_{\tilde{\theta}}(\phi_\theta(X_k))-I_k}^2+c_3\left( \norm{\w}^2+\norm{\theta}^2+\norm{\tilde{\theta}}^2 \right).
  \end{split}\label{eq:pretrain_loss}
\end{equation}
The extra factor of $784$ is added to the denominator of the autoregressive regularization term to correct for the dimension of  $I_k$ and allow for a more interpretable value of  $c_2$.

\begin{figure}[t]
  \centering
  \includegraphics[width=\textwidth]{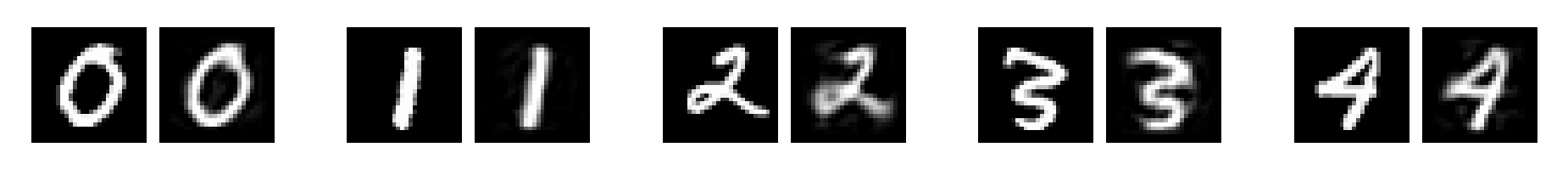}
  \includegraphics[width=\textwidth]{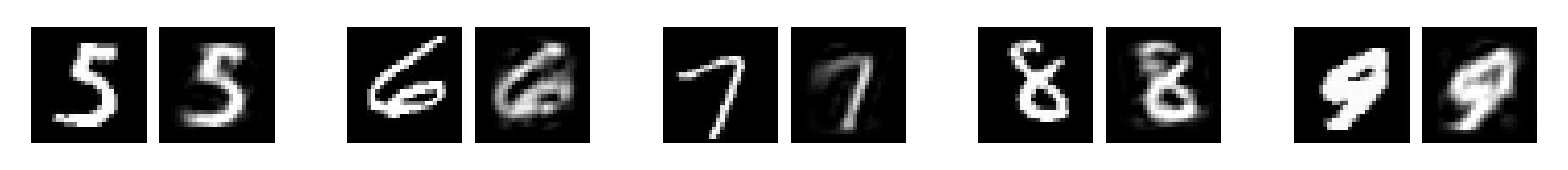}
  \caption{Sample images $I_k$ and encoded-decoded counterparts $\psi_{\tilde{\theta}}(\phi_\theta(X_k))$. For each example, we selected the first data point with the corresponding class from the dataset. In the top row, we present digits 0-4 from the validation set used to evaluate the prediction quality of the pre-trained model. In the bottom row, we show how the autoencoder generalizes to unseen digits 5-9.}
  \label{fig:decoded}
\end{figure}
The choices of hyperparameters, including the architecture of $\phi_\theta$ and  $\psi_{\tilde{\theta}}$ are provided in Appendix~\ref{ap:hyperparam_experiments}. We note that the regularization was chosen aggressively ($c_1=10$, $c_2=1$), but this did not affect the quality of predictions. On the validation set, the network assigns the highest reward to the correct class for $\approx 99.2\%$ images. In Figure~\ref{fig:spectrum_hist}, we present the distributions of spectra of the empirical covariance $\widehat{\Sigma(\theta_0)}$ with and without the orthogonality regularization, plotted on a $\log_{10}$ axis. We can see that without the regularization, there are two clear ``main directions'', with eigenvalues an order of magnitude larger than the rest, so features concentrate in a $2$-dimensional subspace. This mirrors results from~\cite{gur2018gradient}, where authors show that the Hessian of the loss of a classifier has $k$ main directions with the largest eigenvalues, with $k$ being the number of classes\footnote{It is easy to see that twice the covariance $2\Sigma(\theta_0)$ corresponds to the Hessian of the true risk~\eqref{eq:risk_decomposed} w.r.t. $\w$. Note also that the results in~\cite{gur2018gradient} are presented for the cross-entropy loss, but they seem to be transferable to our case of the MSE loss and outputs only from $\{0,1\}$.}. With regularization, the distribution is significantly less disjoint, and $18$ largest eigenvalues constitute $90\%$ of the trace of  $\widehat{\Sigma(\theta_0)}$. When regularization is disabled, this number goes down to $4$.

Although the decoder regularization bears secondary importance and does not influence prediction quality of the pre-trained model, we found that the auto-encoder successfully represents the input. After passing through the information bottleneck, the images of the digits are still well-recognizable, as Figure~\ref{fig:decoded} demonstrates. The model was only trained on digits 0-4, but the decoder recovers the other digits too, albeit at slightly worse quality.

\begin{figure}[t]
  \centering
  \includegraphics[width=\textwidth]{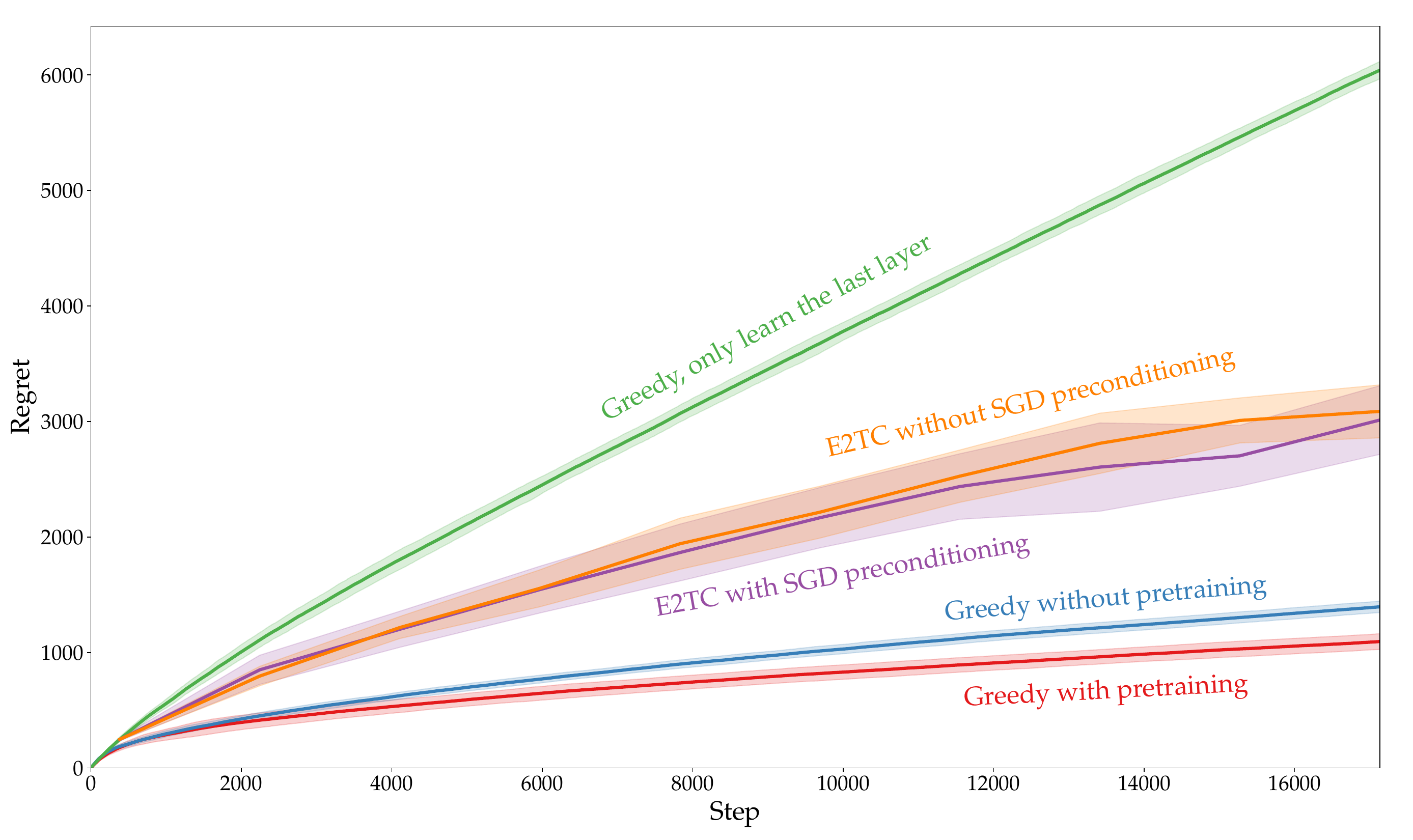}
  \caption{Regret of several bandit algorithms on the online MNIST classification task. For all algorithms we show the cumulative regret (number of misclassified digits so far) averaged over $20$ runs. For each curve we also show the empirical standard deviation. Since E2TC is not an anytime algorithm, we chose $10$ horizons for it and rerun the algorithm for each horizon with a different $T_2$. }
  \label{fig:mnist_regret}
\end{figure}

Pre-trained weights $\theta_0$ are used in a bandit setup to recognise digits 5-9. We compare E2TC with several baselines. First, we use three greedy algorithms, which select the class greedily according to the current estimate of the weights $(\w_t, \theta_t)$, and then update the weights as
\begin{equation}
  \w_{t+1}=\w_t-\zeta_\w \gv^\w_t,\quad\theta_{t+1}=\theta_t-\zeta_\theta\gv^\theta_t, \label{eq:basic_sgd_update}
\end{equation}
with gradient estimators  $\gv^\w_t$ and  $\gv^\theta_t$ given by~\eqref{eq:gvw_def} and~\eqref{eq:gvtheta_def}, respectively. The first greedy algorithm has $\theta$ frozen at the pre-trained weights $\theta_0$ and only updates $\w$, which is initialized with i.i.d. samples from $\mathcal{N}\left(0, \frac{1}{d}\right)$.  This evaluates whether only updating the last layer can learn anything at all, and how much worse this approach is compared to updating the whole network. The second greedy algorithm trains the entire network from scratch. Finally, the third one trains the network starting from $\theta_0$ and a random  $\w$. The learning rates $\zeta_\w$ and  $\zeta_\theta$ are searched on a grid (for the first version, only $\zeta_\theta$). We split the data for digits 5-9 into two equal parts, validation and test. We run the algorithm on the validation data and select the best hyperparameters by the cumulative regret at the end of the algorithm.

For E2TC, we compared the full Algorithm~\ref{alg:main_algo} with an ablated version that does not use pre-conditioning to update $\w$ during the second exploration stage, and instead updates the weights using~\eqref{eq:basic_sgd_update}. As Section~\ref{sec:hyperparams_no_assumptions} suggests, we kept $T_1$ fixed as the horizon was changing. We selected $T_1=d=128$. The other hyperparameters ($\zeta_\w$, $\zeta_\theta$, $\lambda$, and $T_2$) were tuned using the following procedure. For each selection $(\zeta_\w, \zeta_\theta,\lambda)$ from the grid, we chose $10$ horizons for E2TC runs (linearly spaced between $128$ and  $17132$), and ran the two exploration stages. We then evaluated the resulting weights $(\overline{\w},\overline{\theta})$ on a separate validation set. Thus, we get a set of points $\{(T_{2,k}, f_k)\}_{k=1}^{10}$, where  $f_k$ is the proportion of misclassified digits in the validation set for the selected  $T_2=T_{2,k}$. We then fitted a curve $f(T_2)=c(T_2+a)^\alpha+b$ with unknown parameters $a,b,c,\alpha$ to these $10$ points. This gave us a rough empirical formula for the dependence of the misclassification rate on $T_2$, which could be used to select the optimal $T_2$ as
\begin{equation}
  T_2^*=\argmax_{T_2} \left(0.9\left( T_1+T_2 \right)+(T-T_1-T_2)f(T_2)\right).
\end{equation}
With $T_2=T_2^*$, we ran the bandit on the validation set and recorded the final regret. If the learning rate were too high, the algorithm failed to learn anything and the dependence $f(T_2)$ was roughly constant. For lower learning rates, we would typically get a dependence $f=O(T_2^\alpha)$ with the exponent $\alpha\approx -0.35$, complying with our upper bound  $O(T_2^{-1 /4})$. Selecting $T_2$ according to the theoretically justified $T_2=O(T^{4 /5})$ would yield a better guarantee, but in practice it is an overly conservative choice.

Cumulative regrets for the three variants of the greedy algorithm and the two variants of E2TC on the test set are presented on Figure~\ref{fig:mnist_regret}. Greedy algorithms that learn the entire set of weights $(\w,\theta)$ clearly perform better than the rest. While we consider E2TC to be valuable as a theory-grounded approach to pre-training in bandits, it is not surprizing that an explore-then-commit algorithm does not outperform a greedy one. Greedy algorithms are known to outperform even those based on confidence bounds, especially in the random context regime. In~\cite{bastani2021mostly}, the authors prove that under some assumptions on the context distribution, a greedy algorithm without any heuristics for exploration can actually achieve regret $O(d^3\log^{3 /2} d \log T)$ when $K=2$, which is asymptotically optimal in $T$. The most important assumption there is that the context distribution is \emph{sufficiently diverse}, meaning that a greedy strategy would still try out all arms without explicitly exploring them. Among the two best greedy algorithms, we see that the one that uses pre-trained weights does consistently perform better. Pre-conditioning SGD on the weights $\w$, despite its theoretical importance, only confers a minor improvement in regret, especially if we consider the high variance in the regret curve. The greedy algorithm that only learns the last layer, as expected, performs worse than the others, but it still manages to learn something: by the end of the episode, its success rate is around $65\%$. This also speaks to the helpfullness of the pre-trained weights in this problem.

\section{Sample complexity of wine quality prediction}
As we discussed in the introduction, analyzing an ETC-type algorithm brings the extra benefit of a sample complexity bound in the offline i.i.d. learning regime. We can set $T_1+T_2=T$ in Theorem~\ref{thm:main}, and it still gives us a bound on $\risk(\overline{\w},\overline{\theta})-\risk(\wstar,\theta^*)$. In this section, we consider a regression problem in an offline i.i.d. setting, and study the validation curves of the algorithm and various ablations.

Evaluation is performed on the dataset from~\cite{cortez2009modeling} for the task of wine quality prediction from its physicochemical properties. The dataset consists of $4898$ white and  $1599$ red wine samples from Portugal. For each sample, $11$ numeric physicochemical properties were measured. Each wine was also rated by experts on a scale from  $1$ to $10$. The task of our models is to predict this rating. This task was chosen because it naturally allows for pre-training. We pre-train a model on the red wine dataset, and then fine-tune the representation network with the new last layer on the smaller set of white wines.

We consider the default training process from E2TC, starting from a pre-trained model. For $T_1$ steps, collect the data and run Ridge regression on the last layer, then, the remaining $T-T_1$ data points are used to run SGD on $(\w,\theta)$, pre-conditioning the steps for $\w$. We then remove some elements of this approach one by one, to see what effect this has on the validation curve. Our first baseline has pre-conditioning removed, updating the weights during SGD using~\eqref{eq:basic_sgd_update}. The second baseline also does not use pre-conditioning, but it additionally removes Ridge last layer estimation. Instead, the last layer weights are initialized randomly. Our final baseline does not use pre-trained weights. It runs simple SGD on the entire set of weights, all of which are randomly initialized.

\begin{figure}[t]
  \centering
  \includegraphics[width=\textwidth]{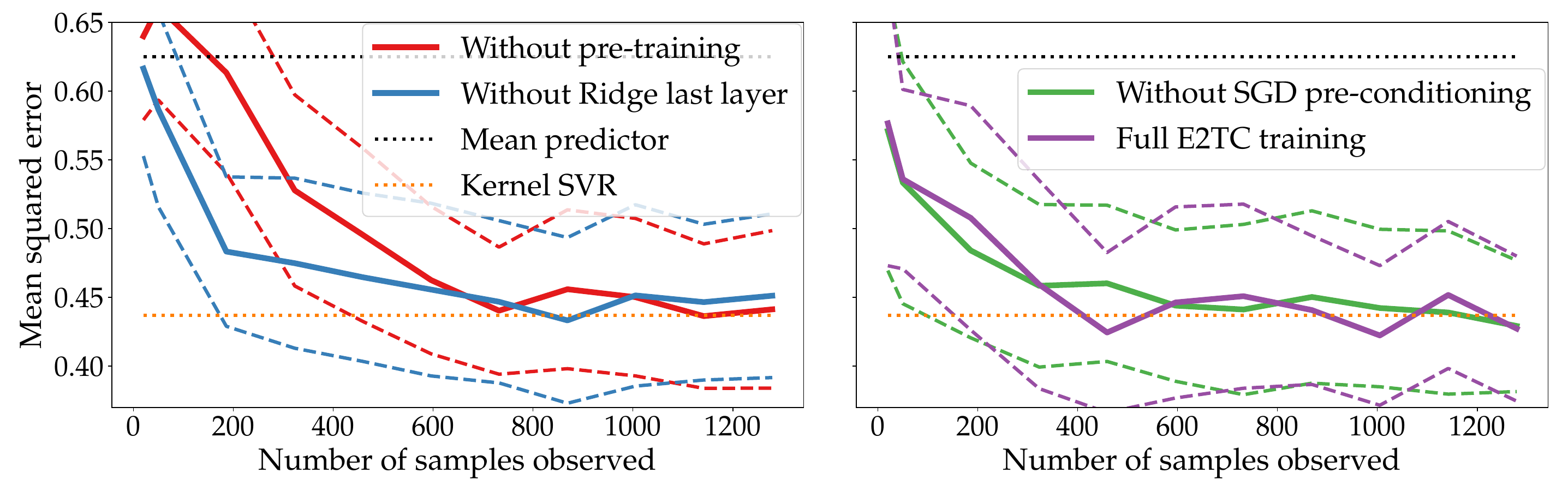}
  \caption{Mean squared error on the test data for E2TC training and various ablations, averaged over $100$ training runs with random permutations of the training data. To de-clutter the plot, we split it into two canvases. Dashed curves show standard deviations. Horizontal dotted lines show the performance of two non-neural-network baselines: predicting the mean of the training data and running kernelized Support Vector Regression with an RBF kernel on it. }
  \label{fig:wine_curves}
\end{figure}

As a preprocessing step, we normalize all input features and the output to have zero mean and unit variance. We use the same loss~\eqref{eq:pretrain_loss} for pre-training the model. We search the optimal hyperparameters, including the learning rate and hidden dimensions of the fully-connected network. We provide the details, including initialization distributions, in Appendix~\ref{ap:hyperparam_experiments}. The pre-trained network with the optimal choice of hyperparameters achieves a mean absolute error (MAE) of $0.51$ on the test data. Even despite heavy auto-regressive regularization, this is a higher test score than the result of the neural network reported in the original paper~\cite{cortez2009modeling} (MAE=$0.58$), probably due to an extra layer in our network or the use of GELU~\cite{hendrycks2016gaussian} activations instead of sigmoid.

We tune the hyperparameters that are not related to model architecture ($\zeta_\w$,  $\zeta_\theta$, $\lambda$) on the red wine data, instead of relying on the ones we selected for white wines. Hyperparameters with the lowest final mean squared error on a validation set are selected. We set $T_1=d=20$. After tuning, we train the model with these hyperparameters and record the mean squared error on a separate test set as the training progresses.

The resulting validation curves for all baselines are presented in Figure~\ref{fig:wine_curves}. We also show the performance of the mean predictor (to see that the methods do learn some nontrivial information about the wine) and of kernelized Support Vector Regression with an RBF kernel (SVR). The most surprising finding from this experiment is that experts' ratings are not completely random. A famous qualitative study of wine odours~\cite{morrot2001color} showed that a wine's scent described by experts is mostly derived from the wine's color. Hence, a further evaluation would be needed to determine whether the physico-chemical properties of the wine only influence its rating through the color, or other factors are at play.

There is significant variance in the training curves, so to distinguish genuine effects from noise we rerun training for $100$ times and average the results. One can see that after processing the entire dataset, all methods achieve roughly the same test performance, coinciding with the performance of SVR. Sample complexity of the methods, however, does differ. Pre-conditioning does not seem to change the performance in this experiment. Using only $20$ samples for Ridge estimation of the last layer (green and purple curves) brings the MSE down slightly at the initial stage compared to using these $20$ samples on SGD (blue curve). The most obvious is the effect of pre-training. Re-using the weights from the white wines is helpful to reach near-optimal performance with fewer red wine samples.


%% file: chapters/conclusion.tex
\chapter{Conclusion}
\label{ch:conclusion}
The main goal of this thesis was to provide an initial theoretical framework to study the effects of neural network pre-training in bandit learning scenarios. Utilizing methods from linear bandits and convex learning literature, we formalized the benefits of pre-trained weights for stochastic contextual bandit learning. We derived the precise requirements on the pre-trained weights under which we can relax global convexity of the loss function, an assumption common in the literature, to local convexity around the minimum. For realizable rewards and locally convex losses, we designed an algorithm, which we called E2TC, that achieves a regret of $\widetilde{O}((KT)^{4 /5})$ in the regime $d=O(T^{2 /5}),\ K= O(T^{1 /5})$. Here, $K$ is the number of actions, $T$ is the horizon, and  $d$ is the dimension of the representation provided by the pre-trained network.

We also provided guarantees for weak learning in stochastic contextual bandit setting, when only the linear last layer of the network is updated and the pre-trained representation weights are kept fixed. We showed how Ridge regression can be used to achieve a regret bound of $O(\epsilon_0\sqrt{d}KT+(KT)^{4 /5})$  or $O(\epsilon_0\sqrt{d}KT)+\widetilde{O}(d^{1 /3} (KT)^{2 /3})$, depending on the regularization strength. Here, $\epsilon_0$ is a measure of misspecification of the pre-trained weights that is less restrictive than infinity-norm distance between the representations, traditionally used for adversarial bandits with misspecification. When no misspecification is present ($\epsilon_{0}=0$), the first bound becomes dimension-independent! This result is only possible due to the stochasticity of contexts and the $K^{4 /5}$ dependence of the bound on the number of actions. For both bounds, no assumption on the shape of the loss function was made, and they can be used both in under- and over-parameterized learning scenarios.

These results advertise a certain way of thinking about pre-training. Instead of relying on polynomially overparameterized networks required by NTK, we look at the problem through the lens of classical optimization, relaxing its overly restrictive assumptions with the help of pre-trained weights. We believe that much further progress can be made along this direction. Our results were based on an explore-than-commit technique, hence the regret bounds that scale worse than $\widetilde{O}(\sqrt{T})$. With an algorithm based on confidence bounds, we could probably reach this target. This approach should be applicable in the adversarial context setup too, albeit under more restrictive misspecification assumptions. When the contexts are stochastic, it was shown that regret logarithmic in $T$ is achievable for linear bandits~\cite{goldenshluger2013linear,bastani2021mostly}. This dependence might be extendable to the more general setting of regrets modelled by neural networks.

We base our analysis of E2TC on local convexity of the risk function around its minimum. When the problem is over-parameterized, this framework is inadequate~\cite{liu2022loss}. Instead we could rely on the Polyak-Łojasiewicz property of the loss. In this setting, the shape of the region around the minima of the loss could be simplified, as a small loss at a point implies its closeness to some minimizer~\cite{garrigos2023square}. We believe that our result that SGD stays in the convex basin w.h.p. from Section~\ref{sec:locally_convex_loss} can be extended to cover PŁ functions. This would immediately translate into an analogue of Theorem~\ref{thm:main} in the over-parameterized regime, assuming the reward is locally PŁ.

By now, machine learning community has largely embraced the bitter lesson. As Richard Sutton concisely puts it in~\cite{sutton2019bitter}, ``The biggest lesson that can be read from 70 years of AI research is that general methods that leverage computation are ultimately the most effective, and by a large margin.'' This effect has recently manifested itself even more explicitly through the scaling laws in large language models. We now know that scaling the size of the model translates directly into its performance on next word prediction, as long as available data is scaled accordingly~\cite{hoffmann2022training}. This dynamic gave rise to the phenomenon of foundation models, giant neural networks trained on large amounts of unsupervised data. Fine-tuning these models is an increasingly ubiquitous trend in machine learning. We hope that the new theory will reflect this shift, and that pre-training will make its way into the guarantees for our algorithms.

%% file: appendix.tex
\chapter{Appendix}
\section{Proof of Lemma~\ref{lem:wellspec_new}} \label{ap:proof_wellspec_new}
  The proof will rely on notation introduced in the beginning of Chapter~\ref{ch:linear}. We will use a standard covering argument. The proof closely follows the exposition in the beginning of Chapter~20 in~\cite{lattimore2020bandit}, but we have to adapt it to allow for an arbitrary scale of the noise and for nonzero regularization $\lambda$.

  Let 
\begin{align}
  \hwl &=\argmin_{\w}\frac{1}{T_1}\sum_{t=1}^{T_1}\left(\w^\tp\pzero(X_t)-\wstartp\pstar(X_t)\right)^2+\lambda \norm{\w}^2 \\
       &=\hsigl^{-1} \frac{1}{T_1}\sum_{t=1}^{T_1} \pzero(X_t)\pstar(X_t)^\tp\wstar. \label{eq:hwl_expr}
\end{align}
We observe that, due to the expression~\eqref{eq:wzero_newdef} for $\wzero$,
  \begin{equation}
    \left\Vert \wzero-\hwl\right\Vert_{\hsigl}=\left\Vert{\frac{1}{T_1}\sum_{t=1}^{T_1} \phi_{\theta_0}(X_t)\eta_t}\right\Vert_{\hsigl^{-1}}. \label{eq:normdiff_rewrite}
  \end{equation}
  Now we will analyze a simpler expression $\mathbf{x}^\tp(\wzero-\hwl)$ for some $\mathbf{x}\in S^{d-1}$ and use covering to make a bound on it uniform over $\mathbf{x}$. Substituting~\eqref{eq:hwl_expr} and~\eqref{eq:wzero_newdef},
\begin{equation}
  \mathbf{x}^\tp(\wzero-\hwl)= \frac{1}{T_1}\sum_{t=1}^{T_1} \mathbf{x}^\tp\hsigl^{-1}\pzero(X_t)\eta_t.
\end{equation}
Since $\eta_t$ are independent (from each other and from  $X_t$) and $B_\eta$-subgaussian, the sum is $\sigma$-subgaussian conditioned on $X_1,\ldots, X_{T_1}$. The subgaussianity parameter satisfies
\begin{align}
  \sigma^2 &=\frac{B_\eta^2}{T_1^2}\sum_{t=1}^{T_1} \left(\mathbf{x}^\tp\hsigl^{-1}\pzero(X_t)  \right) ^2 \\
           &=\frac{B_\eta^2}{T_1^2}\mathbf{x}^\tp \hsigl^{-1}\sum_{t=1}^{T_1}\pzero(X_t)\pzero(X_t)^\tp \hsigl^{-1}\mathbf{x} \\
           & =\frac{B_\eta^2}{T_1}\mathbf{x}^\tp   \hsigl^{-1}\hsigz\hsigl^{-1}\mathbf{x} \\
           &\le \frac{B_\eta^2}{T_1}\left\Vert\mathbf{x}\right\Vert_{\hsigl^{-1}}^2.
\end{align}
In the last line, we used $\hsigz\le\hsigl$. Therefore, for any fixed $X_1,\ldots,X_{T_1}$ and $\mathbf{x}$, we have
\begin{equation} \PP\left[ \mathbf{x}^\tp (\wzero-\hwl) < B_\eta \left\Vert\mathbf{x}\right\Vert_{\hsigl^{-1}} \sqrt{\frac{2\log(1 /\delta)}{T_1}} \right] \ge 1-\delta. \label{eq:highp_onex}
\end{equation}
Since the above holds for any fixed $X_1,\ldots, X_{T_1}$, we can argue from the law of total probability that this bound holds when $X_1,\ldots, X_{T_1}$ are random.

Let now $C_\epsilon$ be the covering net from Lemma~\ref{lem:sphere_cover}. We can use the bound~\eqref{eq:highp_onex}, substituting $\delta /|C_\epsilon|$ instead of $\delta$ and $\hsigl^{1 /2} \mathbf{x}$ instead of $\mathbf{x}$. A union bound then yields that
\begin{equation} 
  \PP\left[\forall \mathbf{x} \in C_\epsilon\quad \mathbf{x}^\tp \hsigl^{1 /2} (\wzero-\hwl) < B_\eta \sqrt{\frac{2\log(|C_\epsilon| /\delta)}{T_1}} \right] \ge 1-\delta. \label{eq:highp_net}
\end{equation}
We can now use the defining property of $C_\epsilon$, and convert this bound to one on $\norm{\wzero-\hwl}_{\hsigl}$:
\begin{align}
  &\norm{\wzero-\hwl}_{\hsigl} \nonumber \\
  &\qquad= \max_{\mathbf{y}\in S^{d-1}} \mathbf{y}^\tp \hsigl^{1 /2}(\wzero-\hwl)\\
  &\qquad= \max_{\mathbf{y}\in S^{d-1}}\min_{\mathbf{x} \in C_\epsilon} \Big\{(\mathbf{y}-\mathbf{x})^\tp \hsigl^{1 /2}(\wzero-\hwl) +\mathbf{x}^\tp \hsigl^{1 /2}(\wzero-\hwl)\Big\} \\
  &\qquad< \epsilon\norm{\wzero-\hwl}_{\hsigl}+\sqrt{\frac{2\log(|C_\epsilon|/\delta)}{T_1}}.
\end{align}
Rearranging,
\begin{equation}
\norm{\wzero-\hwl}_{\hsigl} < \frac{1}{1-\epsilon}\sqrt{\frac{2\log(|C_\epsilon|/\delta)}{T_1}}.
\end{equation}
Setting $\epsilon=1 /2$, we finally arrive at
\begin{equation}
  \norm{\wzero-\hwl}_{\hsigl}<2\sqrt{\frac{d\log 6 + \log(1 /\delta)}{T_1}}. 
\end{equation}
Together with~\eqref{eq:normdiff_rewrite}, this gives the statement of the lemma.
\begin{flushright}
\ensuremath{\square}
\end{flushright}

\section{Proofs of Theorems~\ref{thm:nodimregret} and~\ref{thm:dimregret}.}
\label{ap:boringproofs}
Here we provide proofs for the two theorems about the perfomance of E2TC with the second exploration stage disabled ($T_2=0$). These proofs are very similar, and mostly consist of checking the conditions for Theorem~\ref{thm:ridge_base} and extracting the asymptotic expressions from its unwieldy explicit bounds.

\begin{proof}[Theorem~\ref{thm:nodimregret}]
  In the first exploration stage of E2TC, we perform Ridge regression on $T_1$ i.i.d. samples $(X_t,r_t)$ from the data distribution. Theorem~\ref{thm:ridge_base} states that under certain conditions several bounds hold simultaneously with high probability. We will call the event that these bounds hold $S$ (``success''). Its complement will be called $F$ (``failure''). When $S$ holds, these bounds combine into an overall bound on $\norm{\wzero-\widetilde{\w}}_{\sigz}$, which we then substitute into~\eqref{eq:ridgeboundrisk} to move to $\risk(\omega_0)-\risk(\omega^*)$, and then use Theorem~\ref{thm:risktoregret} to get a regret bound.

Theorem~\ref{thm:ridge_base} states that $\PP[S] \ge  1-4\delta$ under certain conditions. We will use notation $\nu=\log(1 /\delta)$ for conciseness. We select
\begin{equation}
  \nu:=\log\frac{\max\left\{ T_1,e^3\ceff \right\} }{\ceff}=\max\left\{3,\log \frac{T_1}{\ceff}\right\}.
\end{equation}
We also select $\lambda=T_1^{-1 /2}$. The probability of failure is then bounded as
\begin{equation}
  \PP[F]<4e^{-\nu}=\frac{4\ceff}{\max\left\{ T_1,e^3\ceff \right\}}\le \frac{4\ceff}{T_1}\le \frac{4\max \left\{ 1,B_\phi^2 /\lambda \right\}}{T_1}=O(T_1^{-1 /2}) . \label{eq:pfailbound}
\end{equation}

We will now verify that the conditions of Theorem~\ref{thm:ridge_base} hold. First, we need that
\begin{equation}
  \nu > \max\left\{ 0, 2.6-\log \ceff \right\}.
\end{equation}
This holds whenever $T_1\ge 14$. We also need
\begin{align}
  &T_1\ge \frac{6 B_\phi^2}{\lambda}( \log \ceff+\nu)=6 B_\phi^2\sqrt{T_1} \log \max\left\{ T_1,e^3\ceff \right\} \\
  \iff & \left(\sqrt{T_1} \ge 6B_\phi^2\log T_1 \right)\ \land\ \left(\sqrt{T_1} \ge 6B_\phi^2( \log \ceff +3)   \right).
\end{align}
The first inequality here is satisfied for all sufficiently large $T_1$. We note that this introduces a dependency of $N_0$ on $B_\phi$. To see that the second inequality holds as well, we remember that $\ceff \le \max\left\{ 1, B_\phi^2 /\lambda \right\}$, so it suffices to have
\begin{equation}
  \sqrt{T_1} \ge 6B_\phi^2\left( \max\left\{0, \log(B_\phi^2 \sqrt{T_1})\right\} +3 \right).
\end{equation}
This also holds for sufficiently large $T_1$. Requirement~\eqref{eq:ridge_req_rho} is satisfied for $\rho_\lambda$ such that $\rho^2\deffo=B_\phi / \lambda$, according to~\eqref{eq:good_rho}. Requirement~\eqref{eq:ridge_req_b} is given to us by Lemma~\ref{lem:ridgeboundmisspec}. 

Now, we will provide the asymptotic bounds for individual terms used in Theorem~\ref{thm:ridge_base}. First, the reader is reminded of the notation
\begin{equation}
  \text{approx}_\lambda(X)= \wstartp\pstar(X)-\w^{\lambda \tp}\pzero(X).
\end{equation}
Its squared expected value is bounded as 
\begin{align}
  &\E_X\left[\text{approx}_\lambda(X)^2\right]\nonumber \\
  &\quad=\E_x\left[\left(\w^{\lambda \tp}\pzero(X)-\wstartp\pstar(X)\right) ^2 \right] +\lambda \Vert \w^\lambda\Vert^2-\lambda \Vert \w^\lambda\Vert^2 \\
  &\quad\le \E_x\left[\left(\wstar\pzero(X)-\wstartp\pstar(X)\right) ^2 \right] +\lambda \Vert \wstar\Vert^2 -\lambda \Vert \w^\lambda\Vert^2 \label{eq:approxlambound} \\
  &\quad\le B_\w ^2/\sqrt{T_1}+\epsilon_0^2.
\end{align}
Here we used the definitions of $\w^\lambda$ and of $\epsilon_0$.
Next,
\begin{equation}
  \log\ceff+\nu=\log\max\left\{ T_1,e^3\ceff \right\} \le \log\max\left\{e^3,T_1,e^3B_\phi\sqrt{T_1}  \right\} =O(\log T_1).
\end{equation}
Substituting this into~\eqref{eq:deltas},
\begin{equation}
  \delta_s=O(T_1^{-1 /4}\sqrt{\log T_1}).
\end{equation}
Analogously, since $\ceff\ge 1$, we have $\nu=O(\log T_1)$. Lemma~\ref{lem:ridgeregerr} states that
\begin{equation}
  \epsilon_{rg}=O(T_1^{-1 /2}+\epsilon_0^2 d). \label{eq:epsilonrgasymp}
\end{equation}
We are ready to write the asymptotics of  $\epsilon_{bs}$ using~\eqref{eq:epsilonbs}. Since $(1-\delta_s)^{-2}=1+O(\delta_s)$, $\delta_s$ only contributes lower-order terms.
\begin{align}
  \begin{split}
  \epsilon_{bs}&=O\left(\left(\frac{1}{T_1}+\frac{\epsilon_0^2}{\sqrt{T_1}}+\frac{1}{T_1^{3 /2}}+\frac{\epsilon_0^2d}{T_1}\right)\log T_1\right. \\
               &\qquad\qquad\qquad\left.+\left( \frac{1}{T_1^{3 /2}}+\frac{\epsilon_0^2}{\sqrt{T_1}}+\frac{1}{T_1^{5 /2}}+\frac{\epsilon_0^2d}{T_1^2} \right)\log^2 T_1 \right)
  \end{split} \\
               &=O\left(\frac{\log T_1}{T_1}+\frac{\epsilon_0^2\log^2 T_1}{\sqrt{T_1}}+\frac{\epsilon_0^2d\log T_1}{T_1}\right). \label{eq:epsilonbsasymp}
\end{align}
We used~\eqref{eq:asympblambda} above. Note that due to Lemma~\ref{lem:ridgeboundmisspec}, $b_\lambda$ only contributes a $\log T_1$ factor to the second term here.

Before approaching $\epsilon_{vr}$, we need to deal with $\delta_f$. We will use inequalities~\eqref{eq:boundeffdim}:
\begin{align}
  \delta_f\sqrt{\deffo\defft}&=\sqrt{\frac{4B_\phi^2\deffo\defft/\lambda -\defft^2}{T_1}}\left( 1+\sqrt{8\nu}  \right) +\frac{4\nu\sqrt{B_\phi^4\defft/\lambda+\defft^2}}{3T_1} \label{eq:deltafsq} \\
                             & \le \sqrt{\frac{4B_\phi^6}{\lambda^3T_1}}( 1+\sqrt{8\nu}) +\frac{4\nu\sqrt{B_\phi^6/\lambda^2+B_\phi^4/\lambda^2}}{3T_1}  \\
                             &= O(T_1^{1 /4}\sqrt{\log T_1}).
\end{align}
This brings us to~\eqref{eq:epsilonvar}. Bounding $\defft$ using~\eqref{eq:boundeffdim} again, and discarding the  $\delta_s$ factors,
\begin{align}
  \epsilon_{vr}=O\left(\frac{1}{\sqrt{T_1}}+\frac{\sqrt{\log T_1}}{T_1^{3 /4}}+\frac{\sqrt{ \log T_1}}{T_1^{3 /4}}+\frac{\log^{3 /4} T_1}{T_1^{7 /8} }+\frac{\log T_1}{T_1}\right) =O\left(\frac{1}{\sqrt{T_1} }\right).  \label{eq:epsilonvarasymp}
\end{align}
Thus, variance brings the highest-order errors into our analysis, if we do not consider misspecification. Note how we could have bounded  $d_{p,\lambda} \le d$ here instead, and would get better dependence on $T_1$ (with appropriate  $\lambda$) at the cost of introducing the dimension into asymptotics.

Bringing ~\eqref{eq:epsilonrgasymp},~\eqref{eq:epsilonbsasymp}, and~\eqref{eq:epsilonvarasymp} together,
\begin{equation}
  \norm{\wzero-\widetilde{\w}}_{\sigz}^2\le 3\left( \epsilon_{rg}+\epsilon_{bs}+\epsilon_{vr} \right)=O\left(\epsilon_0^2 \dht+\frac{1}{\sqrt{T_1} } \right).
\end{equation}
This proves the first statement of the theorem. Moving to the risk relevant to us using~\eqref{eq:ridgeboundrisk},
\begin{equation}
  \sqrt{\risk(\omega_0)-\risk(\omega^*)}=O\left(\epsilon_0 \sqrt{\dht}+\frac{1}{T_1^{1 /4}} \right).
\end{equation}
Alternatively, using~\eqref{eq:altriskbound}, we get
\begin{equation}
  \sqrt{\risk(\omega_0)-\risk(\omega^*)}=O\left(\epsilon_0 \sqrt{\defft}+\frac{1}{T_1^{1 /4}} \right).
\end{equation}

Now we use Theorem~\ref{thm:risktoregret} and get
\begin{align}
  \E[R_{1:T}] = O\left(TT_1^{-1 /2}+T_1 +\epsilon_0 \sqrt{\min\{\dht,\defft\}} KT+T_1^{-1 /4}KT\right).
\end{align}
\end{proof}

Now we turn to the second theorem.
\begin{proof}[Theorem~\ref{thm:dimregret}]
  The proof proceeds similarly to that of Theorem~\ref{thm:nodimregret}. We reuse notation $\nu=\log(1 /\delta)$. We select
\begin{equation}
  \nu:=\frac{1}{2}\log T_1.
\end{equation}
Equivalently, this means that $\delta=1 /\sqrt{T_1} $. The choice of $\lambda$ is dictated by~\eqref{eq:lambdatwo}. The condition that
\begin{equation}
  \nu > \max\left\{ 0, 2.6-\log \ceff \right\}
\end{equation}
holds for $T_1\ge 183$ (in fact, for such $T_1$ it holds that $\nu> 2.6$). The other requirement 
\begin{align}
  T_1\ge \frac{6 B_\phi^2}{\lambda}( \log \ceff+\nu)=\frac{6T_1}{7\log T_1} \log (\ceff \sqrt{T_1} )
\end{align}
is satisfied for
\begin{equation}
  T_1>\frac{6}{7}\left( a+\frac{1}{2} \right)
\end{equation}
since $\ceff \le d<T_1^a$. Now we proceed to bounding the same quantities as we did in the previous theorem, substituting the new values of $\lambda$ and $\nu$. First, the probability of failure:
\begin{equation}
  \PP[F]<4e^{-\nu}=\frac{4}{\sqrt{T_1}}.
\end{equation}
Next,~\eqref{eq:deltas} under our conditions implies $\delta_s=O(1)$. Lemma~\ref{lem:ridgeregerr} this time gives
\begin{equation}
  \epsilon_{rg}=O\left(\epsilon_0^2d+\frac{\log T_1}{T_1}\right).
\end{equation}
From~\eqref{eq:approxlambound},
\begin{equation}
  \E_X\left[ \text{approx}_\lambda(X)^2 \right] = O\left( \frac{\log T_1}{T_1}+\epsilon_0^2 \right).
\end{equation}
Substituting the above into~\eqref{eq:epsilonbs},
\begin{align}
  \begin{split}
  \epsilon_{bs}&=O\left(\left( \frac{1}{T_1}+\frac{\epsilon_0^2}{\log T_1}+\frac{\log T_1}{T_1^2}+\frac{\epsilon_0^2 d}{T_1} \right)\log T_1 \right. \\
               &\qquad\qquad+\left.\left(\frac{1}{T_1\log T_1}+\frac{\epsilon_0^2d}{\log^2 T_1}+\frac{\log T_1}{T_1^3}+\frac{\epsilon_0^2 d}{T_1^2}  \right)\log^2T_1 \right)
  \end{split} \\
               &=O\left(\frac{\log T_1}{T_1}+\epsilon_0^2d \right).
\end{align}
Note that Lemma~\ref{lem:ridgeboundmisspec} was crucial here to avoid a term that grows with $T_1$.

Bounding $\deffo\le d$ and $\defft\le d$ in~\eqref{eq:deltafsq},
\begin{equation}
  \delta_f\sqrt{\deffo\defft} =O(d).
\end{equation}
When put into~\eqref{eq:epsilonvar} together with $\deffo\le d$ and $\defft\le d$, this gives
\begin{equation}
  \epsilon_{vr}=O\left(\frac{d}{T_1}+\frac{\sqrt{d\log T_1}}{T_1}+\frac{\log T_1}{T_1}\right) .
\end{equation}
We see that  $\epsilon_{vr}$ once again brings the largest terms (asymptotically) into the excess risk bound. This time, the dependence of the sublinear in $T_1$ term on the dimension comes from $\epsilon_{vr}$.

Bringing the above asymptotic results together,
\begin{equation}
  \norm{\wzero-\widetilde{\w}}_{\sigz}^2\le 3\left( \epsilon_{rg}+\epsilon_{bs}+\epsilon_{vr} \right)=O\left(\epsilon_0^2 d+\frac{d}{T_1}+\frac{\sqrt{d \log T_1}}{T_1}+\frac{\log T_1}{T_1} \right).
\end{equation}
This gives the first statement of the theorem. In terms of the risk, we have
\begin{equation}
\sqrt{\risk(\omega_0)-\risk(\omega^*)} \le O\left( \epsilon_0 \sqrt{d} +\sqrt{\frac{d}{T_1}}+\frac{d^{1/4}\log^{1 /4} T_1}{\sqrt{T_1}}+\sqrt{\frac{\log T_1}{T_1}}  \right).
\end{equation}
Substituting into Theorem~\ref{thm:risktoregret},
\begin{equation}
  \E\left[R_{1:T}  \right] =O\left(\frac{T}{\sqrt{T_1}} +T_1+KT\left(\epsilon_0 \sqrt{d} \sqrt{\frac{d}{T_1}} +\frac{d^{1 /4}\log ^{1 /4}T_1}{\sqrt{T_1} }+\sqrt{\frac{\log T_1}{T_1}}  \right) \right).
\end{equation}
\end{proof}

\section{Proof of Theorem~\ref{thm:main}} \label{ap:proof_main}

  We begin by showing that with $\PP>1-4\delta$, we will have $\norm{\wzero-\wstar}_{\Sigma(\theta_0)}<\epsilon_\w$. For this, we once again turn to the Ridge regression analysis from Section~\ref{sec:ridge_lls}. We would like to apply Theorem~\ref{thm:ridge_base} like we did to prove Theorem~\ref{thm:nodimregret}. Theorem~\ref{thm:ridge_base} provides its results in terms of constants $\rho_\lambda$ and  $b_\lambda$. As we have seen in Section~\ref{sec:ridge_lls}, they can be selected as follows. First, $\rho_\lambda$ is such that $\rho_\lambda^2 d_{1,\lambda}=B_\phi^2/\lambda$. Next, by Lemma~\ref{lem:ridgeboundmisspec}, $b_\lambda$ satisfies 
  \begin{equation}
    \frac{b_\lambda^2 d_{1,\lambda}}{T_1^2}=O\left(\frac{\epsilon_0^2 d}{\lambda^2 T_1^2}+ \frac{1}{\lambda T_1^2}\right).
  \end{equation}
One condition for the theorem is that 
\begin{equation}
 T_1\ge \frac{6B_\phi^2}{\lambda}\log \frac{\ceff}{\delta},
\end{equation}
where $\ceff$ is defined in~\eqref{eq:extra_effdim_def}. Since $\ceff<\min\{d, \max\{B_\phi^2 /\lambda, 1\}\}$, this condition holds for $T_1>N_1$ if $N_1$ satisfies~\eqref{eq:asymp_n1}. The theorem also asks for $\nu>\max\{0, 2.6-\log\ceff\}$. This holds whenever  $\nu>2.6$ or, equivalently, when $\delta<\exp(-2.6)$. Thus, Theorem~\ref{thm:ridge_base} is applicable. With $\PP\ge 1-4\delta$, the bound~\eqref{eq:epsilonbs} on $\epsilon_{bs}$ and the bound~\eqref{eq:epsilonvar} on $\epsilon_{vr}$ hold. Here, $\epsilon_{bs}$ and  $\epsilon_{vr}$ are respectively the bias and variance errors of Ridge regression, introduced in~\eqref{eq:introerrors}. They are used to bound the distance $\norm{\wstar-\wzero}_{\Sigma(\theta_0)}$:
\begin{equation}
  \norm{\wstar-\wzero}_{\Sigma(\theta_0)}<\norm{\wstar-\w^\lambda}_{\Sigma(\theta_0)}+\sqrt{\epsilon_{bs}}+\sqrt{\epsilon_{vr}}.  \label{eq:decomp_bound_for_epsilonw}
\end{equation}
The first term is bounded by Lemma~\ref{lem:ridgeregerr_wstar} as
\begin{equation}
  \norm{\wstar-\w^\lambda}_{\Sigma(\theta_0)}\le \frac{\lambda}{2} B_\w ^2+2\epsilon_0^2 d_{2,\lambda},
\end{equation}
where we used that, by Assumption~\ref{as:bound}, $\norm{\wstar}< B_\w $. Bounds~\eqref{eq:epsilonbs} and~\eqref{eq:epsilonvar} give us the following assymptotics for $T_1\to \infty$ and $\delta,\lambda\to 0$:
\begin{gather}
 \epsilon_{bs}=O\left( \frac{l}{T_1}\left( 1+\frac{\epsilon_0^2}{\lambda}+\lambda+\epsilon_0^2d \right) +\frac{l^2}{T_1^2}\left( \frac{1}{\lambda}+\frac{\epsilon_0^2d}{\lambda^2} \right) \right), \\
 \begin{split}
 \epsilon_{vr}=O\Bigg( \frac{\deffp}{T_1}+\frac{\deffp l^{1 /2}}{\lambda^{1 / 2} T_1^{3 /2}}+\frac{\deffp ^{1 /2} l}{T_1^2\lambda^{1 /2}}+\frac{\deffp l}{T_1^2}+\frac{\deffp^{1 /2} l^{1 /2}}{T_1} \\
 +\frac{\deffp^{1 /2} l^{3 /4}}{\lambda^{1 /4}T_1^{5 /4}} +\frac{\deffp^{1/4}l}{\lambda^{1 /4}T_1^{3 /2}}+\frac{\deffp^{1 /2}l}{T_1^{3 /2}}\Bigg),
 \end{split}
\end{gather}
where $\deffp$ is the effective dimension with asymptotics $\deffp=O(\min\{d,1 /\lambda\})$ and  $l=\log(1 /\delta)$. To get the bound in~\eqref{eq:decomp_bound_for_epsilonw} below $\epsilon_\w$, we need that $\sqrt{\epsilon_{bs}} +\sqrt{\epsilon_{vr}}<\Delta\epsilon_\w$. Asymptotically, this requires that the square root of each term is below $\Delta\epsilon_\w$, which happens for large enough $T_1$. More specifically, there exists $N_1(\delta,\lambda,\epsilon_0,d,\Delta\epsilon_\w)$ with asymptotics
\begin{equation}
  \begin{split}
  N_1=\Omega\Bigg( \max\Bigg\{\frac{l}{\deps^2},\ \frac{\epsilon_0^2 l}{\lambda \deps^2},\ \frac{\epsilon_0^2 dl}{\deps^2},\ \frac{l}{\lambda^{1 /2}\deps},\ \frac{\epsilon_0d^{1 /2}l}{\lambda\deps}, \frac{\deffp}{\deps^2}, \\
  \frac{\deffp^{2 /3}l^{1 /3}}{\lambda^{1 /3}\deps^{4/3}},\ \frac{\deffp^{1 /2}l^{1 /2}}{\deps^2},\ \frac{\deffp^{2 /5}l^{3 /5}}{\lambda^{1 /5}\deps^{5 /8}},\ \frac{\deffp^{1 /6}l^{2 /3}}{\lambda^{1 /6}\deps^{4 /3}},\ \frac{\deffp^{1 /3}l^{2 /3}}{\deps^{4 /3}}\Bigg\} \Bigg)
\end{split} \label{eq:full_asymp_n1}
\end{equation}
such that, whenever $T_1>N_1$, we have $\norm{\wstar-\w_0}_{\Sigma(\theta_0)}\le \epsilon_\w$. The expression~\eqref{eq:full_asymp_n1} is clearly inconvenient to handle. By using the fact that $\deffp=O(1 /\lambda)$ and that $\Delta\epsilon_\w<\epsilon_c<1$, we can upper bound it with the asymptotic expression~\eqref{eq:asymp_n1}.  Thus, with $\PP\ge 1-4\delta$, for large enough $T_1$, we have $\norm{\wstar-\w_0}_{\Sigma(\theta_0)}<\epsilon_\w$.

Now we turn to the SGD stage of the algorithm. The remaining part of the proof follows the structure of the proof of Theorem~\ref{thm:sgd_containment}. There are two differences. First, we have to separately consider that $\w$ and $\theta$ remain close to $\wstar$  and  $\theta^*$, respectively. Next, we have to take preconditioning in $\w$ into account. Dealing with preconditioning is complicated by the fact that the algorithm uses the covariance estimate $\widehat{\Sigma(\theta_0)}$ instead of the real covariance, while the latter defines the convex basin $\conv$.

We introduce a filtration $\{\mathcal{F}_t\}_{t=0}^{T_2}$, defined as
\begin{equation}
  \quad\mathcal{F}_{t}=\sigma(X_1, \ldots,X_{T_1},\widetilde{X}_{1},\ldots,\widetilde{X}_t).
\end{equation}
We remind the reader here that $(X_t,r_t)_{t=1}^{T_1}$ is the data used for Ridge estimation of $\w_0$, and $(\widetilde{X}_t,\tilde{r}_t)_{t=1}^{T_2}$ is the data used for SGD. With this notation, $\w_0=\w_1$ is $\mathcal{F}_0$-measurable. In general, $\w_t$ and $\theta_t$ are $\mathcal{F}_{t-1}$-measurable, while $\gv^\w_t$ and $\gv^\theta_t$ are  $\mathcal{F}_t$-measurable. For brevity, we also introduce
\begin{equation}
  \sigz:=\Sigma(\theta_0),\quad\hsigz:=\widehat{\Sigma(\theta_0)},\quad\hsigl:=\widehat{\Sigma(\theta_0)}+\lambda I.
\end{equation}
Let us assume that until step $t$, SGD both in  $\w$ and  $\theta$ was contained in the convex basin and that the projection was not involved. Mathematically speaking, we assume that
\begin{gather}
  \forall k\le t\ (\w_k,\theta_k)\in\conv\quad\text{and}\quad \forall k< t\ (\w_k-\zeta\hsigl^{-1}\gv^\w_k,\ \theta_k-\zeta\gv^\theta_k)\in\conv. \label{eq:assumption_conv}
\end{gather}
Let $\widehat{\w}_{t+1}=\w_t-\zeta\hsigl^{-1}\gv^\w_t$ and $\hat{\theta}_{t+1}=\theta_t-\zeta\gv^\theta_t$. We have that $\w_{t+1}=\Pi_{\dom_\w}(\widehat{\w}_{t+1})$ and $\theta_{t+1}=\Pi_{\dom_\theta}(\hat{\theta}_{t+1})$. As we shall see, under certain conditions that hold with high probability, $(\widehat{\w}_{t+1},\hat{\theta}_{t+1})\in\conv$, which implies $\widehat{\w}_{t+1}=\w_{t+1}$ and $\hat{\theta}_{t+1}=\theta_{t+1}$. Containing $\widehat{\w}_{t+1}$ in $\conv$ requires bounding the following quantity:
\begin{align}
  \norm{\widehat{\w}_{t+1}-\wstar}_{\sigz}^2 \le \norm{\hsigz-\sigz}_2\norm{\widehat{\w}_{t+1}-\wstar}^2+\norm{\widehat{\w}_{t+1}-\wstar}_{\hsigz}^2.
\end{align}
We will now use Lemma~\ref{lem:matmomentum} to bound $\norm{\hsigz-\sigz}_2$. Also, 
\begin{equation}
\norm{\widehat{\w}_{t+1}-\wstar}^2\le 2\norm{\w_t-\wstar}^2+2\zeta^2\norm{\hsigz^{-1}\gv^\w_t}^2\le 8B_\w^2+2\frac{\zeta^2}{\lambda}D_\w^2.
\end{equation}
Introducing
 \begin{equation}
   S_\delta:=\frac{8B_\w^2+2D_\w^2\zeta^2 / \lambda}{T_1}\left(B_\phi^2\log \frac{2d}{\delta}+\sqrt{B_\phi^4\log^2 \frac{2d}{\delta}+2B_\phi^4T_1\log\frac{2d}{\delta}}\right), 
 \end{equation}
 we get, with $\PP\ge 1-\delta$,
\begin{align}
  \begin{split}
  &\norm{\widehat{\w}_{t+1}-\wstar}_{\sigz}^2 \\
  &\quad< S_\delta+\norm{\widehat{\w}_{t+1}-\wstar}_{\hsigz}^2
  \end{split} \\
  &\quad=S_\delta+\norm{\w_{t}-\wstar}_{\hsigz}^2+\zeta^2\norm{\hsigl^{-1}\gv^\w_t}^2_{\hsigz}+2\zeta(\wstar-\w_t)^\tp\hsigz\hsigl^{-1}\gv^\w_t.
\end{align}
Using $\hsigz\hsigl^{-1}=I-\lambda\hsigl^{-1}$, we arrive at
\begin{align}
  \begin{split}
  &\norm{\widehat{\w}_{t+1}-\wstar}_{\sigz}^2< S_\delta+\norm{\w_{t}-\wstar}_{\hsigz}^2 \\
  &\qquad+\zeta^2\norm{\hsigl^{-1}\gv^\w_t}^2_{\hsigz}+2\zeta(\wstar-\w_k)^\tp\gv^\w_k-2\lambda\zeta(\wstar-\w_k)^\tp\hsigl^{-1}\gv^\w_k.
  \end{split}\label{eq:single_step_decomp}
\end{align}
By the assumption~\eqref{eq:assumption_conv}, projection was not involved until step $t$, so we can apply~\eqref{eq:single_step_decomp} recursively to $\norm{\w_{t}-\wstar}_{\hsigz}^2$ and get
\begin{align}
  &\norm{\widehat{\w}_{t+1}-\wstar}_{\sigz}^2< S_\delta+\norm{\w_1-\wstar}_{\hsigz}^2 \\
  &\quad+\zeta^2\sum_{k=1}^{t}\norm{\hsigl^{-1}\gv^\w_k}^2_{\hsigz}+2\zeta \sum_{k=1}^{t}(\wstar-\w_k)^\tp\gv^\w_k-2\lambda\zeta \sum_{k=1}^{t}(\wstar-\w_k)^\tp\hsigl^{-1}\gv^\w_k.\nonumber
\end{align}
Using Lemma~\ref{lem:matmomentum} again to switch from $\norm{\w_1-\wstar}_{\hsigz}^2$ to  $\norm{\w_1-\wstar}_{\sigz}^2$, we get with $\PP\ge 1-2\delta$
\begin{align}
  &\norm{\widehat{\w}_{t+1}-\wstar}_{\sigz}^2 <2S_\delta+\norm{\w_1-\wstar}_{\sigz}^2 \label{eq:prelim_bound_sigma0} \\
  &\quad+\zeta^2\sum_{k=1}^{t}\norm{\gv^\w_k}^2_{\hsigl^{-1}}+2\zeta \sum_{k=1}^{t}(\wstar-\w_k)^\tp\gv^\w_k-2\lambda\zeta \sum_{k=1}^{t}(\wstar-\w_k)^\tp\hsigl^{-1}\gv^\w_k.\nonumber
\end{align}
A similar recursive derivation for $\norm{\widehat{\w}_{t+1}-\wstar}_2^2$ gives us that
\begin{equation}
  \norm{\widehat{\w}_{t+1}-\wstar}_2^2=\norm{\w_1-\wstar}_2^2+\zeta^2 \sum_{k=1}^{t} \norm{\hsigl^{-1}\gv^\w_k}_2^2+2\zeta \sum_{k=1}^{t}(\wstar-\w_k)^\tp\hsigl^{-1}\gv^\w_k.
\end{equation}
We can express the last sum from this equation and substitute it into~\eqref{eq:prelim_bound_sigma0}:
\begin{align}
  &\norm{\widehat{\w}_{t+1}-\wstar}_{\sigz}^2< 2S_\delta+\norm{\w_1-\wstar}_{\sigz}^2+\lambda( \norm{\w_1-\wstar}_2^2-\norm{\widehat{\w}_{t+1}-\wstar}_2^2) \nonumber \\
  &\quad+\zeta^2\sum_{k=1}^{t}\norm{\hsigl^{-1}\gv^\w_k}^2_{\hsigz}+2\zeta \sum_{k=1}^{t}(\wstar-\w_k)^\tp\gv^\w_k+\lambda\zeta^2 \sum_{k=1}^{t} \norm{\hsigl^{-1}\gv^\w_k}_2^2. \label{eq:prelim2_bound_sigma0}
\end{align}
Now the first and the last sum can be combined using
\begin{equation}
  \norm{\hsigl^{-1}\gv^\w_k}^2_{\hsigz}+\lambda\norm{\hsigl^{-1}\gv^\w_k}_2^2=\norm{\hsigl^{-1}\gv^\w_k}^2_{\hsigl}=\norm{\gv^\w_k}^2_{\hsigl^{-1}}.
\end{equation}
Substituting this into~\eqref{eq:prelim_bound_sigma0},
\begin{align}
  \begin{split}
  &\norm{\widehat{\w}_{t+1}-\wstar}_{\sigz}^2< 2S_\delta+\norm{\w_1-\wstar}_{\sigz}^2+4B_\w^2\lambda \\
  &\quad\qquad+\zeta^2\sum_{k=1}^{t}\norm{\gv^\w_k}^2_{\hsigl^{-1}}+2\zeta \sum_{k=1}^{t}(\wstar-\w_k)^\tp\gv^\w_k. \label{eq:randomeqwhenwillitend}
  \end{split}
\end{align}
Terms in the first sum in~\eqref{eq:randomeqwhenwillitend} can be bounded as $\norm{\gv^\w_k}_{\hsigl^{-1}}^2<D_\w^2 /\lambda$. The last sum in this expression will be bounded using time-uniform Hoeffding-Azuma inequality. Let
\begin{align}
 &a_k:=2\zeta(\wstar-\w_k)^\tp\gv^\w_k - \E\left[2\zeta(\wstar-\w_k)^\tp\gv^\w_k \mid \mathcal{F}_{k-1}  \right] .
\end{align}
It is an $\mathcal{F}$-adapted martingale difference sequence, which can be bounded as
\begin{equation}
  |a_k|\le 8B_\w D_\w\zeta \quad\text{a.s.}
\end{equation}
Let $E_\w$ be an even defined as
\begin{equation}
  E_\w:=\left\{ \forall t\quad \sum_{k=1}^{t} a_k<20 B_\w D_\w\zeta\sqrt{t\left((\log \log (64eB_\w D_\w t\zeta^2))_++\log(2 /\delta)\right)}  \right\}. \label{eq:def_ew}
\end{equation}
By Lemma~\ref{lem:uniform_azuma}, $\PP[E_\w]\ge 1-\delta$. The bound in $E_\w$ is itself upper-bounded by
\begin{equation}
  P_\delta^\w:= 20 B_\w D_\w\zeta\sqrt{T_2\left((\log \log (64eB_\w D_\w T_2\zeta^2))_++\log(2 /\delta)\right)} .
\end{equation}
With this notation, whenever $E_\w$ holds, we have for all $t\le T_2$,
\begin{align}
  \begin{split}
    \norm{\widehat{\w}_{t+1}-\wstar}_{\sigz}^2\le \norm{\w_1-\wstar}_{\sigz}^2+4B_\w^2\lambda +2S_\delta+ \zeta^2T_2 \frac{D_\w^2}{\lambda}+P^\w_\delta \\
  +2\zeta \sum_{k=1}^{t} (\wstar-\w_k)^\tp\nabla_\w \risk(\w_k,\theta_k), \label{eq:final_bound_hatw}
  \end{split}
\end{align}
where we used that, by definition of $\gv^\w_k$, we have  $\E\left[ \gv^\w_k \mid \mathcal{F}_{k-1} \right] =\nabla_\w \risk(\w_k,\theta_k)$. 

Now we will turn to $\norm{\hat{\theta}_{t+1}-\theta^*}$. We would like to remind the reader that $\hat{\theta}_{t+1}=\theta_t-\zeta \gv^\theta_t$, and that we are still operating under the assumption~\eqref{eq:assumption_conv}. Since no preconditioning is involved with respect to $\theta$, analyzing its trajectory is significantly easier. We have
\begin{equation}
  \norm{\hat{\theta}_{t+1}-\theta^*}^2=\norm{\theta_t-\theta^*}^2+\zeta^2\norm{\gv^\theta_t}^2+2\zeta(\theta^*-\theta_t)^\tp\gv^\theta_t.
\end{equation}
Applying this recursively to $\norm{\theta_t-\theta^*}^2$,
\begin{equation}
  \norm{\hat{\theta}_{t+1}-\theta^*}^2=\norm{\theta_1-\theta^*}^2+\zeta^2 \sum_{k=1}^{t} \norm{\gv^\theta_t}^2+2\zeta \sum_{k=1}^{t} (\theta^*-\theta_k)^\tp\gv^\theta_k. \label{eq:thetahat_sum_expr}
\end{equation}
The last sum enjoys a uniform high-probability bound analogous to~\eqref{eq:def_ew}. Let
\begin{equation}
  P_\delta^\theta:= 20 B_\theta D_\theta\zeta\sqrt{T_2\left((\log \log (64eB_\theta D_\theta T_2\zeta^2))_++\log(2 /\delta)\right)} .
\end{equation}
By Lemma~\ref{lem:uniform_azuma},
\begin{equation}
  \PP\left[\forall t\quad\sum_{k=1}^{t} (\theta^*-\theta_k)^\tp \gv^\theta_k< \sum_{k=1}^{t} (\theta^*-\theta_k)^\tp \nabla_\theta \risk(\w_k,\theta_k) + P^\theta_\delta\right]\ge 1-\delta.
\end{equation}
Substituting this into~\eqref{eq:thetahat_sum_expr} and using $\norm{\gv^\theta_t}\le D_\theta$,
\begin{equation}
  \norm{\hat{\theta}_{t+1}-\theta^*}^2< \norm{\theta_1-\theta^*}^2+\zeta^2 T_2 D_\theta+2\zeta \sum_{k=1}^{t} (\theta^*-\theta_k)^\tp\nabla_\theta \risk(\w_k,\theta_k) + P^\theta_\delta. \label{eq:final_bound_theta}
\end{equation}
Let, as elsewhere in the thesis, $\omega=(\w,\theta)$ be the combined vector of parameters of the network. We will also define the combined gradient estimator $\gv_t=(\gv^\w_t;\ \gv^\theta_t)\in\R^{d+d_0}$. With this notation, $\omega^*=(\w^*,\theta^*)$ and $\omega_k=(\w_k,\theta_k)$. The two sums in~\eqref{eq:final_bound_hatw} and~\eqref{eq:final_bound_theta} add to
\begin{equation}
  \sum_{k=1}^{t} (\theta^*-\theta_k)^\tp\nabla_\theta \risk(\w_k,\theta_k)+\sum_{k=1}^{t} (\wstar-\w_k)^\tp\nabla_\w \risk(\w_k,\theta_k)=\sum_{k=1}^{t} (\omega^*-\omega_k)^\tp\nabla \risk(\omega_k).
\end{equation}
By assumption~\eqref{eq:assumption_conv}, $\omega_k\in\conv$ for $k \le t$. Therefore,
\begin{equation}
  (\omega^*-\omega_k)^\tp\nabla \risk(\omega_k)\le \risk(\omega^*)-\risk(\omega_k) \le 0,
\end{equation}
where the last inequality holds since $\omega^*\in\argmin_\omega \risk(\omega)$. Thereforee, when adding the inequalities~\eqref{eq:final_bound_hatw} and~\eqref{eq:final_bound_theta}, we get
\begin{align}
  &\norm{\widehat{\w}_{t+1}-\wstar}_{\sigz}^2+\norm{\hat{\theta}_{t+1}-\theta^*}^2\le \norm{\w_1-\wstar}_{\sigz}^2+\norm{\theta_1-\theta^*}^2  \label{eq:final_bound_omega} \\
  &\quad+4B_\w^2\lambda+2S_\delta+ \zeta^2T_2 \left(\frac{D_\w}{\lambda}+D_\theta\right) +P^\w_\delta+P^\theta_\delta. \nonumber
\end{align}
Whenever condition~\eqref{eq:require_small_zeta} holds, the terms on the second line of this expression are bounded by $\epsilon_c^2-\epsilon_\w^2-\epsilon_\theta^2$. At the same time, 
\begin{equation}
  \norm{\w_1-\wstar}_{\sigz}^2+\norm{\theta_1-\theta^*}_2^2<\epsilon_\w^2+\epsilon_\theta^2.
\end{equation}
Together, this gives
\begin{equation}
  \norm{\widehat{\w}_{t+1}-\wstar}_{\sigz}^2+\norm{\hat{\theta}_{t+1}-\theta^*}^2\le \epsilon_\w^2+\epsilon_\theta^2+\epsilon_c^2-\epsilon_\w^2-\epsilon_\theta^2= \epsilon_c^2.
\end{equation}
We have, as promised, shown that $(\widehat{\w}_{t+1},\hat{\theta}_{t+1})\in\conv$. Since $\conv\subset\dom_\w\times\dom_\theta$, this means that projection will not be used, and that at step $t$ SGD stays within the convex basin. Therefore, whenever the used high-probability bounds hold, the entire SGD trajectory will stay in $\conv$. 

Since the function is convex in the domain where SGD operates, we can use the technique from the proof of Theorem~\ref{thm:highp_sgd} to show a bound on the suboptimality gap. 
Let $c_t$ be defined as
\begin{equation}
  c_t=\risk(\omega_t)-\risk(\omega^*)-(\omega_t-\omega^*)^\tp \gv_t. 
\end{equation}
When $\risk$ is convex in the region containing $\omega^*$ and all $\omega_t$, it satisfies
\begin{equation}
  \E[c_t|\mathcal{F}_{t-1}]=\risk(\omega_t)-\risk(\omega^*)-(\omega_t-\omega^*)^\tp\nabla \risk(\omega_t) \le 0
\end{equation}
The martingale difference $\tilde{c}_t=c_t-\E[c_t|\mathcal{F}_{t-1}]$ is almost surely bounded as
\begin{equation}
|\tilde{c}_t|<4B_\w B_\phi+2B_\w D_\w + 2B_\theta D_\theta.
\end{equation}
We can apply Lemma~\ref{lem:azuma} to it and get, with $\PP\ge 1-\delta$
\begin{equation}
  \frac{1}{T_2}\sum_{t=1}^{T_2}\left(\risk(\omega_t)-\risk(\omega^*)\right) <\frac{1}{T_2}\sum_{t=1}^{T_2}\left((\omega_t-\omega^*)^\tp \gv_t +\E[c_t|\mathcal{F}_{t-1}]\right)+ Q_\delta,
\end{equation}
where 
\begin{equation}
  Q_\delta:=(4B_\w B_\phi+2B_\w D_\w + 2B_\theta D_\theta)\sqrt{\frac{2}{T_2}\log \frac{1}{\delta}}. 
\end{equation}
In total, we used $5$ high-probability bounds in the proof and $\norm{\w_1-\wstar}_{\sigz}^2<\epsilon_\w^2$ holds with $\PP\ge 1-4\delta$. Thus, with $\PP\ge 1-9\delta$, we have $\E[c_t|\mathcal{F}_{t-1}]\le 0$. Combining this with the above bound, we get
\begin{equation}
  \PP\left[\frac{1}{T_2}\sum_{t=1}^{T_2}\left(\risk(\omega_t)-\risk(\omega^*)\right) <\frac{1}{T_2}\sum_{t=1}^{T_2}(\omega_t-\omega^*)^\tp \gv_t +Q_\delta\right] \ge 1-10\delta. \label{eq:highp_sumf}
\end{equation}
Now we only need to decompose the sum on the right. First, since we now know that with high probability, projection was not involved in the algorithm, so
\begin{equation}
  \norm{\w_{t+1}-\wstar}_{\hsigl}^2=\norm{\w_t-\wstar}_{\hsigl}^2-2\zeta(\w_t-\wstar)^\tp\gv^\w_t+\zeta^2\norm{\gv^\w_t}_{\hsigl^{-1}}^2.
\end{equation}
Applying this recursively,
\begin{equation}
  \norm{\w_{T_2+1}-\wstar}_{\hsigl}^2=\norm{\w_1-\wstar}_{\hsigl}^2-2\zeta \sum_{t=1}^{T_2} (\w_t-\wstar)^\tp\gv^\w_t+\zeta^2 \sum_{t=1}^{T_2}\norm{\gv^\w_t}_{\hsigl^{-1}}^2. \label{eq:recursive_nosigmadot}
\end{equation}
We can express the first sum from this expression and substitute it into the second sum in~\eqref{eq:highp_sumf}. With $\PP>1-10\delta$, we have that 
\begin{align}
  \risk(\overline{\omega})-\risk(\omega^*) &\le\frac{1}{T_2}\sum_{t=1}^{T_2} (\risk(\omega_t)-\risk(\omega^*)) \nonumber\\
  &< \sum_{t=1}^{T_2}(\omega_t-\omega^*)^\tp \gv_t+Q_\delta \\
  &=\frac{1}{T_2}\sum_{t=1}^{T_2}(\w_t-\wstar)^\tp\gv^\w_t+\frac{1}{T_2}\sum_{t=1}^{T_2} (\theta_t-\theta^*)^\tp\gv^\theta_t+Q_\delta.
\end{align}
We express the first sum here from~\eqref{eq:recursive_nosigmadot} and the sum --- from~\eqref{eq:thetahat_sum_expr}:
\begin{align}
                                   \begin{split}
                                   &= \frac{1}{2\zeta T_2}\left(\norm{\w_{1}-\wstar}_{\hsigl}^2- \norm{\w_{T_2+1}-\wstar}_{\hsigl}^2\right)+\frac{\zeta}{2T_2}\sum_{t=1}^{T_2}\norm{\gv^\w_t}_{\hsigl^{-1}}^2 \\
                                   &\qquad+\frac{1}{2\zeta T_2}\left(\norm{\theta_1-\theta^*}^2-\norm{\theta_{T_2+1}-\theta^*}^2\right)+\frac{\zeta}{2T_2}\sum_{t=1}^{T_2}\norm{\gv^\theta_t}^2+Q_\delta
                                   \end{split} \\
                                   &\le \frac{1}{2\zeta T_2}\left(\norm{\w_{1}-\wstar}_{\hsigl}^2+\epsilon_\theta^2\right)+\frac{\zeta P^\gv_\delta}{2T_2} +\frac{\zeta D_\theta^2}{2}+Q_\delta.
\end{align} 
Finally, we substitute the bound~\eqref{eq:highp_bound_hsigl} from Section~\ref{sec:bandit_lls} for $\norm{\w_{1}-\wstar}_{\hsigl}$ (the reader is reminded here that we use notation $\w_0=\w_1$ and $\theta_0=\theta_1$ for convenience) and get the last statement of the theorem.
\begin{flushright}
\ensuremath{\square}
\end{flushright}

\section{Hyperparameter choices in experiments}
\label{ap:hyperparam_experiments}
\paragraph{Network architecture} In both MNIST and wine quality experiments, the representation network $\phi_\theta$ is a fully-connected network with one hidden layer and GELU activations~\cite{hendrycks2016gaussian} after both the hidden and the output layer. Thus, the overall reward model  $x\mapsto \w^\tp\phi_\theta(x)$ is a network with two hidden layers, whose dimensions we will denote as $[d_h, d]$.

\paragraph{Initialization} Whenever the weights of a fully-connected layer had to be initalized, we used i.i.d. samples from $\mathcal{N}\left( 0, \frac{1}{d_{in}} \right)$ for the entries of the kernel matrix, where $d_{in}$ is the dimension of the input of this layer. This prevents the activations from exploding during the forward pass. If the layer has a bias (i.e. it is not the last layer), bias weights were zero-initialized.

\paragraph{MNIST experiment} Regularization factors in the loss~\eqref{eq:pretrain_loss} were selected as $c_1=10,\ c_2=1, c_3=0$. Such aggressive regularization was possible because it did not influence the accuracy of the network, which stayed at $99.2\%$ correct classifications. The hidden dimensions of the network were $d_h=d=128$. For the bandit runs with or without the fine-tuned network, the data for digits 5-9 was divided into two equal parts. One was used for hyperparameter search, the other --- to evaluate the bandit algorithm. Hyperparameters for all algorithms were searched on a grid $\zeta_\w,\zeta_\theta\in\left\{10^{-4}, 10^{-3},10^{-2},10^{-1}\right\} $ and $\lambda\in\left\{ 10^{-3}, 10^{-2},5*10^{-1},10^{-1} \right\} $. For each tuple $(\zeta_\w,\zeta_\theta,\lambda)$, the optimal $T_2$ was selected with the procedure described in Section~\ref{sec:mnist}. We used $T_1=128$.

\paragraph{Wine quality experiment} Regularization factors were $c_1=0.1,\ c_2=0.1,\ c_3=0.01$. For pre-training on white wine, we set $\zeta_\w=\zeta_\theta=:\zeta$ and searched through tuples of $(d_h,d,\zeta)$ on a grid $d_h,d\in\left\{ 2,5,10,20,50,100 \right\},\ \zeta\in\left\{ 10^{-1},10^{-2},10^{-3},10^{-4} \right\} $. For the main experiment on the red wines, we set $\zeta_\theta=0.01$ and searched through the grid  $\zeta_\w\in\left\{ 10^{-4},10^{-3},10^{-2} \right\},\ \lambda\in \left\{ 10^{-1}, 10^{0}, 10^1,10^2 \right\} $.

\section{External results}
In this section we provide the external results that this thesis relies on. Most of the original notation from the sources was adapted to our needs. 

\begin{lemma}[Sphere covering number, Lemma~20.1 in~\cite{lattimore2020bandit}] \label{lem:sphere_cover} $\ $ \\
  There exists a set $C_\epsilon \subset \R^d$ with $|C_\epsilon|<(3 /\epsilon)^d$ such that for all $\mathbf{x}\in S^{d-1}$ there exists a  $\mathbf{y}\in C_\epsilon$ with $\norm{\mathbf{x}-\mathbf{y}}\le \epsilon$.
\end{lemma}

\begin{lemma}[Estimation of second momentum, adapted from~\cite{wainwright2019high}] \label{lem:matmomentum} $\ $ \\
  Let $\mathbf{x}_1, \ldots,\mathbf{x}_n$ be i.i.d. random vectors such that $\E\left[ \mathbf{x}_i\mathbf{x}_i^\tp \right]=\Sigma$ and $\norm{\mathbf{x}_i}<B$ a.s. Let the empirical estimator of $\Sigma$ be 
\begin{equation}
  \widehat{\Sigma}=\frac{1}{n}\sum_{i=1}^{n} \mathbf{x}_i\mathbf{x}_i^\tp.
\end{equation}
  Then, for all $\delta>0$,
  \begin{equation}
    \PP \left[ \Vert \widehat{\Sigma}-\Sigma\Vert_2<\frac{1}{n}\left(B^2\log \frac{2d}{\delta}+\sqrt{B^4\log^2 \frac{2d}{\delta}+2B^2\norm{\Sigma}_2n\log \frac{2d}{\delta}} \right) \right] \ge  1-\delta.
  \end{equation}
\end{lemma}
\begin{proof}
  This lemma follows from~\cite[Corollary 6.20]{wainwright2019high}. This corollary states that
  \begin{equation}
    \PP\left[ \Vert \widehat{\Sigma}-\Sigma\Vert_2\ge \epsilon \right]  \le 2d\exp\left[ -\frac{n\epsilon^2}{2B^2(\norm{\Sigma}_2+\epsilon)} \right].
  \end{equation}
  In the book, it is formulated for zero-mean vectors $\mathbf{x}_i$, but this is never used in the proof, and only allows us to state that $\widehat{\Sigma}$ is a \emph{covariance} estimator of $\mathbf{x}_i$. We don't need that, and instead work with the second momentum of the noncentered $\mathbf{x}_i$. To get the bound in terms of $\delta$, we simply solve the quadratic equation w.r.t. $\epsilon$.
\end{proof}

\begin{lemma}[Hoeffding-Azuma inequality for supermartingales] \label{lem:azuma} $\ $ \\
  Let $X_1,\dotsc,X_T$ be a supermartingale difference sequence such that $|X_k|\le B$ a.s. for all $t$. Then, for all $\delta > 0$
\begin{equation}
  \PP\left[ \sum_{t=1}^T X_t < B\sqrt{2T \log \frac{1}{\delta}}  \right] \ge  1-\delta .
\end{equation}
\end{lemma}

Another variant of the Hoeffding-Azuma inequality, presented in~\cite{kassraie2023anytime}, is useful for obtaining so-called \emph{time-uniform} high-probability guarantees, which hold for all partial sums of $X_t$:
\begin{lemma}[Time-uniform Hoeffding-Azuma, Lemma 26 in~\cite{kassraie2023anytime}] \label{lem:uniform_azuma} $\ $ \\
  Let $X_1,\ldots,X_T$ be a martingale difference sequence such that $|X_t|\le B$ for all $t$ a.s. Then for all $\delta>0$,
  \begin{equation}
    \PP\left[ \forall t\quad \sum_{k=1}^{t} X_k < \frac{5B}{2}\sqrt{t\left( (\log \log etB^2)_++\log (2/\delta) \right) }  \right] \ge  1-\delta,
  \end{equation}
  where $(x)_+=\max\{x,0\}$.
\end{lemma}

In the following theorem, compared to~\cite{hsu2012random}, we purposefully lost a bit of generality and adapted the notation to our needs.
\begin{theorem}[Random design Ridge regression, Theorem~16 in~\cite{hsu2012random}] \label{thm:ridge_base}
  $\ $ \\
  Let $X_1,\ldots,X_{T_1}$ be i.i.d. samples from a distribution $\mathcal{D}_X$ over $\R^d$. Let  $r_t=\wstartp\pstar(X_t)+\eta_t$, where $\eta_t$ is i.i.d. noise. Let Assumption~\ref{as:bound} from Section~\ref{sec:formulation} hold. Assume that for some $\rho_\lambda$,
  \begin{equation}
    \frac{\norm{X}_{\sigl^{-1}}}{\sqrt{d_{1,\lambda}} }\le \rho_\lambda\qquad\text{a.s. for $X\sim\mathcal{D}_X$}, \label{eq:ridge_req_rho}
  \end{equation}
  where $\sigl$ is defined by~\eqref{eq:succ_sigdef} and $\deffo$ --- by~\eqref{eq:deffot_def}. Let 
\begin{equation}
  \text{approx}_\lambda(X)=\wstartp\pstar(X)-\w^{\lambda\tp}\pzero(X),
\end{equation}
with $\w^\lambda$ given by~\eqref{eq:wlam_def}. Assume also that for some $b_\lambda$,
\begin{equation}
  \frac{\norm{\text{approx}_\lambda(X)\pzero(X)}_{\sigl^{-1}}}{\sqrt{d_{1,\lambda}}}\le b_\lambda. \label{eq:ridge_req_b}
\end{equation}
Finally, assume that $\delta\in(0, 1)$ is chosen so that
\begin{equation}
  \log \frac{1}{\delta} >2.6-\log\ceff\quad\text{and}\quad T_1\ge 6\rho_\lambda^2\deffo\log\frac{\ceff}{\delta}.
\end{equation}
Let $\epsilon_{rg}$, $\epsilon_{bs}$, and $\epsilon_{vr}$ be defined by~\eqref{eq:introerrors}. Then, with $\PP\ge 1-4\delta$, the following statements hold simultaneously:
\begin{enumerate}
  \item $\hsigl$ is invertible and
\begin{equation}
  \norm{\sigl^{1 /2}\hsigl^{-1}\sigl^{1 /2}}_2\le (1-\delta_s)^{-1},
\end{equation}
where
\begin{equation}
  \delta_s=\sqrt{\frac{4\rho_\lambda^2\deffo\log(\ceff /\delta)}{T_1}}+ \frac{2\rho_\lambda^2\deffo\log(\ceff /\delta)}{3T_1}. \label{eq:deltas}
\end{equation}
\item Frobenius norm error in $\hsigl$ is bounded:
   \begin{equation}
     \norm{\sigl^{-1 /2}(\sigz-\hsigz)\sigl^{-1 /2}}_F\le \delta_f\sqrt{\deffo}, 
  \end{equation}
  where
  \begin{equation}
    \delta_f=\sqrt{\frac{\rho_\lambda^2\deffo-\defft /\deffo}{T_1}}\left(1+\sqrt{8\log \frac{1}{\delta}}\right)+\frac{4\sqrt{\rho_\lambda^4\deffo+\defft /\deffo} }{3T_1}\log \frac{1}{\delta}.
  \end{equation}
\item $\epsilon_{bs}$ is bounded:
\begin{multline}
  \epsilon_{bs}\le \frac{2}{(1-\delta_s)^2}\Bigg(\frac{\rho_\lambda^2 \deffo\E_X[\text{approx}_\lambda(X)^2]+\epsilon_{rg}}{T_1}\Bigg( 1+\sqrt{8\log \frac{1}{\delta}}  \Bigg) ^2 \\
  +\frac{16(b_\lambda \sqrt{\deffo}+\sqrt{\epsilon_{rg}})^2}{T_1^2}\log^2 \frac{1}{\delta}\Bigg). \label{eq:epsilonbs}
\end{multline}
  \item $\epsilon_{vr}$ is bounded:
  \begin{multline}
    \epsilon_{vr}\le \frac{B_\eta^2\left( \defft+\delta_f\sqrt{\deffo\defft}  \right) }{T_1(1-\delta_s)^2}+\frac{2B_\eta^2\sqrt{\left( \defft+\delta_f\sqrt{\deffo\defft}  \right)\log(1 /\delta) } }{T_1(1-\delta_s)^{3 /2}} \\
    +\frac{2B_\eta^2}{T_1(1-\delta_s)}\log \frac{1}{\delta}. \label{eq:epsilonvar}
  \end{multline}
\end{enumerate}
\end{theorem}